\documentclass[letterpaper]{article}
\usepackage{fullpage}
\usepackage{caption}
\usepackage{graphicx}
\usepackage[utf8]{inputenc}
\usepackage[T1]{fontenc}
\usepackage[english]{babel}
\usepackage[style=numeric, url=false, backend=biber,maxbibnames=5,maxcitenames=2,backref=false,uniquelist=false,uniquename=false,sorting=none]{biblatex}
\usepackage{url}
\usepackage[colorlinks,linkcolor=black,citecolor=black,bookmarks=true]{hyperref}
\usepackage{times}
\usepackage{csquotes}
\usepackage[dvipsnames]{xcolor}
\usepackage{breakurl}
\usepackage{ragged2e}
\usepackage[export]{adjustbox}
\newrobustcmd{\B}{\boldmath}
\usepackage{multirow}
\usepackage{dsfont}
\usepackage{amsmath,amssymb,amsthm}
\usepackage{amsfonts}
\usepackage{macro}
\usepackage{cleveref}
\usepackage{enumitem}
\usepackage{upgreek}
\usepackage{thmtools}

\usepackage{booktabs}  
\usepackage{amsfonts}  
\usepackage{nicefrac}  
\usepackage{microtype} 
\usepackage{xcolor}    
\usepackage[linesnumbered,ruled,vlined,noend]{algorithm2e}
\usepackage{color}
\usepackage{float}
\usepackage{bm}
\usepackage{dsfont}
\usepackage{xspace,lipsum}
\usepackage{thm-restate}
\usepackage[disable]{todonotes}
\usepackage[outdir=./]{epstopdf}

\usepackage{tikz}
\usepackage{pgfplots}
\pgfplotsset{compat=1.11}
\usepgfplotslibrary{fillbetween}
\usetikzlibrary{intersections}
\usetikzlibrary{shapes.geometric}

\newcommand{\citep}{\parencite}

\title{An Operator Splitting View of Federated Learning}
\author{Saber Malekmohammadi\textsuperscript{1,2}\thanks{These authors contributed equally to this work}
, Kiarash Shaloudegi\textsuperscript{2}\footnotemark[1], Zeou Hu \textsuperscript{1,2}, Yaoliang Yu\textsuperscript{1}
\\[2em]
University of Waterloo, Canada\textsuperscript{1} \\
\texttt{\{saber.malekmohammadi, zeou.hu, yaoliang.yu\}@uwaterloo.ca} \\[1em]
Huawei Noah's Ark Lab, Montreal, Canada\textsuperscript{2} \\
\texttt{kiarash.shaloudegi@huawei.com}
\date{}
}

\newcommand{\gray}[1]{\textcolor{gray}{#1}}

\newcommand{\FL}{\texttt{FL}\xspace}
\newcommand{\FA}{\texttt{FedAvg}\xspace}
\newcommand{\FP}{\texttt{FedProx}\xspace}
\newcommand{\FR}{\texttt{FedRP}\xspace}
\newcommand{\FS}{\texttt{FedSplit}\xspace}

\newcommand{\FPI}{\texttt{FedPi}\xspace}

\newcommand{\numuser}{m}
\newcommand{\numdim}{d}
\newcommand{\userind}{i}
\newcommand{\localstep}{j}
\newcommand{\numstep}{k}
\newcommand{\Fm}{\mathsf{F}}
\newcommand{\Tm}{\mathsf{T}}

\newcommand{\id}{\mathrm{id}}

\newcommand{\wbs}{\mathsf{w}}
\newcommand{\zbs}{\mathsf{z}}
\newcommand{\ubs}{\mathsf{u}}
\newcommand{\vbs}{\mathsf{v}}
\newcommand{\Ksf}{\mathsf{K}}

\newcommand{\cl}{\mathop{\mathrm{cl}}}
\newcommand{\dist}{\mathsf{dist}}
\newcommand{\dom}{\mathop{\mathrm{dom}}}
\newcommand{\ran}{\mathop{\mathrm{rge}}}
\newcommand{\dual}[2]{\langle #1; #2\rangle}

\newcommand{\ra}[2][]{\mathsf{R}_{#2}^{#1}}

\newcommand{\grad}[2][]{\mathsf{G}_{#2}^{#1}}
\newcommand{\hprox}[2][]{\mathsf{\hat{P}}_{#2}^{#1}}
\newcommand{\hgrad}[2][]{\mathsf{\hat{G}}_{#2}^{#1}}

\renewbibmacro{in:}{%
	\ifentrytype{article}{}{%
		\printtext{\bibstring{in}\intitlepunct}}}
\renewbibmacro*{volume+number+eid}{%
	\printfield{volume}%
	\setunit*{\addnbspace}
	\printfield{number}%
	\setunit{\addcomma\space}%
	\printfield{eid}}
\renewbibmacro*{volume+number+eid}{%
	\setunit*{\addcomma\space}
	\printfield{volume}%
	\setunit*{\addcomma\space}
	\printfield{number}%
	\setunit{\addcomma\space}%
	\printfield{eid}
}
\DeclareFieldFormat[article]{volume}{\bibstring{volume}~#1}
\DeclareFieldFormat[article]{number}{\bibstring{number}~#1}

\newbibmacro{string+doiurlisbn}[1]{%
  \iffieldundef{doi}{%
    \iffieldundef{url}{%
      \iffieldundef{isbn}{%
        \iffieldundef{issn}{%
          #1%
        }{%
           \href{http://books.google.com/books?vid=ISSN\thefield{issn}}{#1}%
        }%
      }{%
         \href{http://books.google.com/books?vid=ISBN\thefield{isbn}}{#1}%
      }%
    }{%
       \href{\thefield{url}}{#1}%
    }%
  }{%
     \href{http://dx.doi.org/\thefield{doi}}{#1}%
  }%
}
\DeclareFieldFormat{title}{\usebibmacro{string+doiurlisbn}{\mkbibemph{#1}}}
\DeclareFieldFormat[article,incollection,inproceedings,unpublished,misc]{title}%
{\usebibmacro{string+doiurlisbn}{\mkbibquote{#1}}}
\begin{document}

\maketitle

\begin{abstract}
 
 Over the past few years, the federated learning (\FL) community has witnessed a proliferation of new \FL algorithms. However,  our understating of the theory of \FL is still fragmented, and a thorough, formal comparison of these algorithms remains elusive. Motivated by this gap,  we show that many of the existing \FL algorithms can be understood from an operator splitting point of view. This unification allows us to compare different algorithms with ease, to refine previous convergence results and to uncover new algorithmic variants. In particular, our analysis reveals the vital role played by the step size in \FL algorithms. The unification also leads to a streamlined and economic way to accelerate \FL algorithms, without incurring any communication overhead. We perform numerical experiments on both convex and nonconvex models to validate our findings.  
\end{abstract}

\section{Introduction}
The accompanying (moderate) computational power of modern smart devices such as phones, watches, home appliances, cars, \etc, and the enormous data accumulated from their interconnected ecosystem have fostered new opportunities and challenges to train/tailor modern deep models. Federated learning (\FL), as a result, has recently emerged as a massively distributed framework that enables training a shared or personalized model without infringing user privacy. Tremendous progress has been made since the seminal work of \textcite{McMahanMRHA17}, including algorithmic innovations \cite{tian_2018_fedprox,Yurochkin19a,ReddiCZG20,pathak2020fedsplit,HuoYGCH20,WangYSPK20,LiKQR20}, convergence analysis \cite{Khaled2019FirstAO,xiang_convergence_2019,KhaledMR20,MalinovskiyKGCR20,GorbunovHR21}, personalization \cite{Mansour2020,DinhTN20,HeteroFL2021,PersonalizedFL2021}, privacy protection \cite{NasrSH19,Augenstein2020Generative}, model robustness \cite{bhagoji_adversariallens_2019,BagdasaryanVHES20,SunKSM19,ReisizadehFPJ20}, fairness \cite{Fedmgda,MohriSS19,LiSBS20}, standardization \cite{Caldas18,He20}, applications \cite{SmithCST17}, just to name a few. We refer to the excellent surveys \cite{Kairouz2019s,LiSTS19,YangLCT19} and the references therein for the current state of affairs in \FL.

The main goal of this work is to take a closer examination of some of the most popular \FL algorithms, including \FA \cite{McMahanMRHA17}, \FP \cite{tian_2018_fedprox}, and \FS \cite{pathak2020fedsplit}, by connecting them with the well-established theory of operator splitting in optimization. In particular, we show that \FA corresponds to   forward-backward splitting and we demonstrate a trade-off in its step size and number of local epochs, while a similar phenomenon has also been observed in \cite{xiang_convergence_2019,pathak2020fedsplit,MalinovskiyKGCR20,Charles21}. \FP, on the other hand, belongs to backward-backward splitting, or equivalently, forward-backward splitting on a related but regularized problem, which has been somewhat overlooked in the literature. Interestingly, our results reveal that the recent personalized model in \cite{DinhTN20} is exactly the problem that \FP aims to solve. Moreover, we show that when the step size in \FP diminishes (sublinearly fast), then it actually solves the original problem, which is contrary to the observation in \cite{pathak2020fedsplit} where the step size is fixed and confirms again the importance of step size in \FL.

For \FS, corresponding to Peaceman-Rachford splitting \cite{PeacemanRachford55,LionsMercier79}, we show that its convergence (in theory) heavily hinges on the strong convexity of the objective and hence might be less stable for nonconvex problems.

Inspecting \FL through the lens of operator splitting also allows us to immediately uncover new \FL algorithms, by adapting existing splitting algorithms. Indeed, we show that Douglas-Rachford splitting \cite{DouglasRachford56,LionsMercier79} (more precisely the method of partial inverse \cite{Spingarn85a}) yields a (slightly) slower but more stable variant of \FS, while at the same time shares a close connection to \FP: the latter basically freezes the update of a dual variable in the former. We also propose \FR by combining the reflector in \FS and the projection (averaing) in \FA and \FP, and extend an algorithm of \textcite{BauschkeKruk04}. We improve the convergence analysis of \FR and empirically demonstrate its competitiveness with other \FL algorithms. We believe these results are just the tip of the ice-berg and much more progress can now be expected: the marriage between \FL and operator splitting simply opens a range of new possibilities to be explored.

Our holistic view of \FL algorithms suggests a very natural unification, building on which we show that the aforementioned algorithms reduce to different parameter settings in a grand scheme. We believe this is an important step towards standardizing \FL from an algorithmic point of view, in addition to standard datasets,  models and evaluation protocols as already articulated in \cite{Caldas18,He20}. Our unification also allows one to compare and implement different \FL algorithms with ease, and more importantly to accelerate them in a streamlined and economic fashion. Indeed, for the first time we show that  Anderson-acceleration \cite{Anderson65,FuZB20}, originally proposed for nonlinear fixed-point iteration, can be adapted to accelerate existing \FL algorithms and we provide practical implementations that do not incur any communication overhead. We perform experiments on both convex and nonconvex models to validate  our findings, and compare the above-mentioned \FL algorithms on an equal footing.

We proceed in \S\ref{sec:bg} with some background introduction. Our main contributions are presented in \S\ref{sec:sp}, where we connect \FL with operator splitting, shed new insights on existing algorithms, suggest new algorithmic variants and refine convergence analysis, and in \S\ref{sec:acc}, where we present a unification of current algorithms and propose strategies for further acceleration. We conclude in \S\ref{sec:con} with some future directions. Due to space limit, all proofs are deferred to \Cref{sec:app-proof}.

\section{Background}
\label{sec:bg}
In this \namecref{sec:bg} we recall the federated learning (\FL) framework of \textcite{McMahanMRHA17}.
We consider $\numuser$ users (edge devices), where the $\userind$-th user aims at minimizing a function $f_{\userind}: \RR^{\numdim} \to \RR, \userind = 1, \ldots, \numuser$, defined on a shared model parameter  $\wv\in \RR^{\numdim}$. Typically, each user function $f_{\userind}$ depends on the respective  user's local (private) data $\Dc_{\userind}$. 
The main goal in \FL is to \emph{collectively and efficiently} optimize \emph{individual} objectives $\{f_{\userind}\}$ in a decentralized, privacy-preserving and low-communication way.

Following \textcite{McMahanMRHA17}, many existing \FL algorithms fall into the natural formulation that optimizes the (arithmetic) average performance:
\begin{align}
\label{eq:wFL}
\min_{\wv\in \RR^{\numdim}} ~ f(\wv), ~~\where~~ f(\wv) := \sum_{{\userind}=1}^{\numuser} \lambda_{\userind} f_{\userind}(\wv).
\end{align}
The weights $\lambda_{\userind}$ here are nonnegative and w.l.o.g. sum to 1. Below, we assume they are specified \emph{beforehand} and fixed throughout.

To facilitate our discussion, we recall the Moreau envelope and proximal map of a  function\footnote{We allow the function $f$ to take $\infty$ when the input is out of our domain of interest.} $f: \RR^d \to \RR \cup\{\infty\}$, defined respectively as:
\begin{align}
\label{eq:prox}
\env[\eta]{f}(\wv) = \min_{\xv} ~\tfrac{1}{2\eta}\|\xv - \wv\|_2^2 + f(\xv), \qquad \prox[\eta]{f}(\wv) = \argmin_{\xv} ~\tfrac{1}{2\eta}\|\xv - \wv\|_2^2 + f(\xv),
\end{align}
where $\eta > 0$ is a parameter acting similarly as the step size. We also define the reflector 
\begin{align}
\ra[\eta]{f}(\wv) &= 2 \prox[\eta]{f}(\wv) - \wv.
\end{align}
As shown in \cref{fig:proxprojection}, when $f = \iota_C$ is the indicator function of a (closed) set $C \subseteq \RR^d$, $\prox[\eta]{f}(\wv) \equiv \prox[]{C}(\wv)$ is the usual (Euclidean) projection onto the set $C$ while $\ra[\eta]{f}(\wv) \equiv \ra[]{C}(\wv)$ is the reflection of $\wv$ \wrt $C$.

\begin{figure}
\centering
\begin{tikzpicture}[
    roundnode/.style={circle, draw=black!100, fill=black!100, very thick,inner sep=0pt, minimum size=3pt},
    projroundnode/.style={circle, draw=black!100, fill=black!100, very thick,inner sep=0pt, minimum size=3pt},
    projroundnode2/.style={circle, draw=black!100, fill=black!100, very thick,inner sep=0pt, minimum size=3pt},
    squarednode/.style={rectangle, draw=black!60, fill=black!5, very thick, minimum size=5mm},
    ellipseset/.style={ellipse, draw=red!60, fill=red!5, thick, minimum size=30mm},
    ]
    \node[ellipse, draw = red!100, text = black, fill = blue!10, minimum width = 8cm, minimum height = 4cm, line width=0.3mm] (e) at (0,0) (cvxset) {$C$ ~~~~~~~~~~};
    \node[roundnode]     at (5.5, 1.9)      (node)        {};
    \node[projroundnode] at (3.75, 0.65)    (projnode)        {};
    
    \node[roundnode]     at (2.7, 2.6)      (node2)        {};
    \node[projroundnode2] at (2.0, 0.55)    (projnode2)        {};

    \draw[<-, line width=0.3mm] (projnode) -- (node) node[right] {$\mathbf{~w}$};
    \draw[<-, line width=0.3mm] (projnode) node[right] {$~P_{\iota_C}^\eta(\mathbf{w})$} -- (node);
    \draw[-, line width=0.4mm] (3.2,1.35) -- (1.55,1.95) ;
    
    \node[rectangle,
    draw = black,
    minimum width = 0.2cm, 
    minimum height = 0.2cm, rotate around={68:(0,0)}] (r) at (2.3, 1.82) {};
    
    \draw[<-, line width=0.3mm] (projnode2) -- (node2) node[right] {$\mathbf{~v}$};
    \draw[<-, line width=0.3mm] (projnode2) node[left] {$~R_{\iota_C}^\eta(\mathbf{v})~$}-- (node2) node[midway,left] {};
    
    \draw[-, line width=0.4mm] (3.42,1.18) -- (4.15,0.08) ;
    
    \node[rectangle,
    draw = black,
    minimum width = 0.2cm, 
    minimum height = 0.2cm, rotate around={33:(0,0)}] (r) at (3.80, 0.86) {};
\end{tikzpicture}
\caption{When $f = \iota_C$ is the indicator function of a (closed) set $C$, $\prox[\eta]{f}(\wv)$ is the Euclidean projection of $\wv$ onto the set $C$. Similarly, $\ra[\eta]{f}(\vv)$ is the Euclidean reflection of $\vv$ w.r.t. the set $C$.}
\label{fig:proxprojection}
\end{figure}
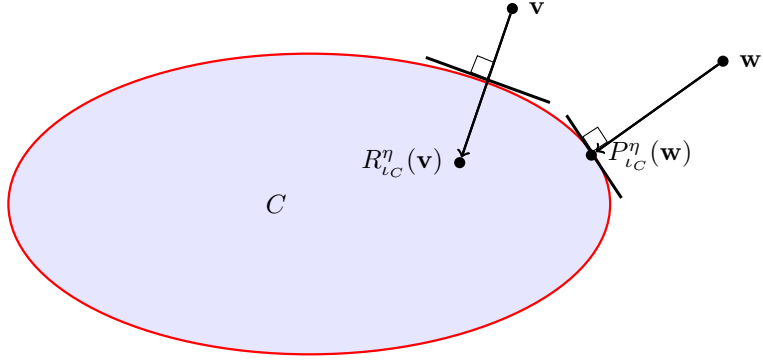

We are now ready to reformulate our main problem of interest \eqref{eq:wFL} in a product space:
\begin{align}
\label{eq:FL-prod}
\min_{\wbs \in H} ~ \fs(\wbs) = \inner{\one}{\fv(\wbs)} , ~\where~ \fv(\wbs) := \big(f_1(\wv_1), \ldots, f_m(\wv_m)\big), 
\end{align}
$\one$ is the vector of all 1s, 
\begin{align}
H := \{ \wbs = (\wv_1, \ldots, \wv_{\numuser}) \in \RR^{dm}: ~ \wv_1 = \cdots = \wv_{\numuser} \}
\end{align} 
and 
we equip with the inner product 
\begin{align}
\inner{\wbs}{\zbs} = \inner{\wbs}{\zbs}_{\lambdav} := \sum_{\userind} \lambda_\userind \wv_{\userind}^\top \zv_{\userind}.
\end{align} 
Note that the complement $H^\perp = \{ \wbs: \sum_{\userind} \lambda_\userind \wv_\userind = \zero\}$, and
\begin{align}
\prox[]{H}(\wbs) = (\bar\wv, \ldots, \bar\wv), ~\where~ \bar\wv := \sum\nolimits_{\userind} \lambda_{\userind} \wv_{\userind}, ~ \mbox{and} ~ \ra[]{H}(\wbs) = ( 2\bar\wv-\wv_1, \ldots, 2\bar\wv - \wv_{\numuser} ).
\end{align}
We define the (forward) gradient update map \wrt to a (sub)differentiable function $f$: 
\begin{align}
\grad[\eta]{f} := \id - \eta \cdot\partial f, ~\qquad~ \wv \mapsto \wv - \eta \cdot \partial f(\wv).
\end{align}
When $f$ is differentiable and convex, we note that $\ra[\eta]{f} = \grad[\eta]{f} \circ \prox[\eta]{f}$ and $\nabla \env[\eta]{f} = \tfrac{\id - \prox[\eta]{f}}{\eta}$.

\section{\FL as operator splitting}
\label{sec:sp}
Following \textcite{pathak2020fedsplit}, in this section we interpret existing \FL algorithms (\FA, \FP, \FS) from the operator splitting point of view. We reveal the importance of the step size, obtain new convergence guarantees, and uncover some new and accelerated algorithmic variants. 

\subsection{\FA as forward-backward splitting}
The \FA algorithm of \textcite{McMahanMRHA17} is essentially a $k$-step version of the forward-backward splitting approach of \textcite{Bruck77}: 
\begin{align}
\label{eq:FB}
\wbs_{t+1} \gets \prox{H} \grad[\eta_t]{\fs, k} \wbs_t, ~~\where~~ \grad[\eta_t]{\fs, k} := \underbrace{\grad[\eta_t]{\fs} \circ \grad[\eta_t]{\fs} \circ \cdots \circ \grad[\eta_t]{\fs}}_{k \mbox{ times }},
\end{align}
where we perform $k$ steps of (forward) gradient updates \wrt $\fs$ followed by 1 (backward) proximal update \wrt $H$. 
To improve efficiency, \textcite{McMahanMRHA17} replace $\grad[\eta_t]{\fs}$ with $\hgrad[\eta_t]{\fs}$ where the (sub)gradient is approximated on a minibatch of training data and $\prox{H}$ with $\hprox{H}$ where we only average over a chosen subset $I_t \subseteq [\numuser]:=\{1, \ldots, \numuser\}$ of users. 
Note that the forward step $\grad[\eta_t]{\fs, k}$ is performed in parallel at the user side while the backward step $\prox{H}$ is performed at the server side.

The number of local steps $k$ turns out to be a key factor: Setting $k=1$ reduces to the usual (stochastic) forward-backward algorithm and enjoys the well-known convergence guarantees. 
On the other hand, setting $k=\infty$ (and assuming convergence on each local function $f_{\userind}$) amounts to (repeatedly) averaging the chosen minimizers of $f_{\userind}$'s, and eventually converges to the average of minimizers of all $\numuser$ user functions. In general, the performance of the fixed point solution of \FA appears to have a dependency on the number of local steps $k$ and the step size $\eta$. Let us illustrate this with quadratic functions $f_i(\wv) = \tfrac12\|A_i\wv - \bv_i\|_2^2$ and non-adaptive step size $\eta_t \equiv \eta$, where we obtain the fixed-point $\wv_{\FA}^{*}(k)$ of \FA in closed-form \parencite{pathak2020fedsplit}:
\begin{align}
\textstyle
\wv_{\FA}^{*}(k) 
= \big(\sum_{\userind=1}^{\numuser} \frac{1}{\numstep}\sum_{\localstep=0}^{\numstep-1} A_{\userind}^\top A_{\userind} (I \!-\! \eta A_{\userind}^\top A_{\userind})^{\localstep}\big)^{\!-1}
\big(\sum_{\userind=1}^{\numuser} \frac{1}{\numstep}\sum_{\localstep=0}^{\numstep-1} (I \!-\! \eta A_{\userind}^\top A_{\userind})^{\localstep} A_{\userind}^\top \bv_{\userind}\big).
\end{align}

For small $\eta$ we may apply Taylor expansion and ignore the higher order terms:
\begin{align}
\label{eq:ls-appx}
\textstyle
\frac{1}{\numstep} \sum_{\localstep=0}^{\numstep-1} (I-\eta A_{\userind}^\top A_{\userind})^{\localstep} 
\approx \frac{1}{\numstep}\sum_{\localstep=0}^{\numstep-1} (I-\localstep \eta A_{\userind}^\top A_{\userind})
= I - \tfrac{\eta (\numstep-1)}{2} A_{\userind}^\top A_{\userind}.
\end{align}

Thus, we observe that the fixed point $\wv_{\FA}^{*}(k)$ of \FA depends on $\eta(k-1)$, up to higher order terms. In particular, when $k=1$, the fixed point does not depend on $\eta$ (as long as it is small enough to guarantee convergence), but for $k > 1$, the solution that \FA converges to does depend on $\eta$. Moreover, the final performance is almost completely determined by $\eta(k-1)$ in this quadratic setting when $\eta$ is small, as we verify in experiments, see \Cref{fig:fixed_pr_stationary} left (details on generating $A_i$ and $\bv_i$ can be found in \cref{sec:lslr_exp_setup}). 

\Cref{fig:trade_off} in the appendix further illustrates this trade-off  where larger local epochs $\numstep$ and learning rate $\eta$ leads to faster convergence (in terms of communication rounds) during the early stages, at the cost of a worse final solution. Similar results are also reported in e.g. \citep{xiang_convergence_2019,pathak2020fedsplit,MalinovskiyKGCR20,Charles21}.
In \Cref{fig:fixed_pr_stationary}, we run \FA in \eqref{eq:FB} on least squares, logistic regression and a nonconvex CNN on the MNIST dataset \citep{mnist}. The experiments are run with pre-determined numbers of  communication rounds and different configurations of local epochs $k$ and learning rates $\eta$. (Complete details of our experimental setup are given in \Cref{sec:app-setup}.\footnote{All nonconvex experimental results in this paper are averaged over 4 runs with different random seeds.}) We examine the dependencies of \FA's final performance regarding $k$ and $\eta$. We can conclude that, in  general, smaller local epochs and learning rates gives better final solutions (assuming sufficient communication rounds to ensure convergence). Moreover, for least squares and logistic regression, the product  $\eta(k-1)$ we derived in \eqref{eq:ls-appx} almost completely determines the final performance, when $k$ and $\eta$ are within a proper range. For the nonconvex CNN (Fig \ref{fig:fixed_pr_stationary}, right), especially with limited communications, the approximation in \eqref{eq:ls-appx} is too crude to be indicative.
\begin{figure}[t]
\centering
\includegraphics[width=0.33\columnwidth]{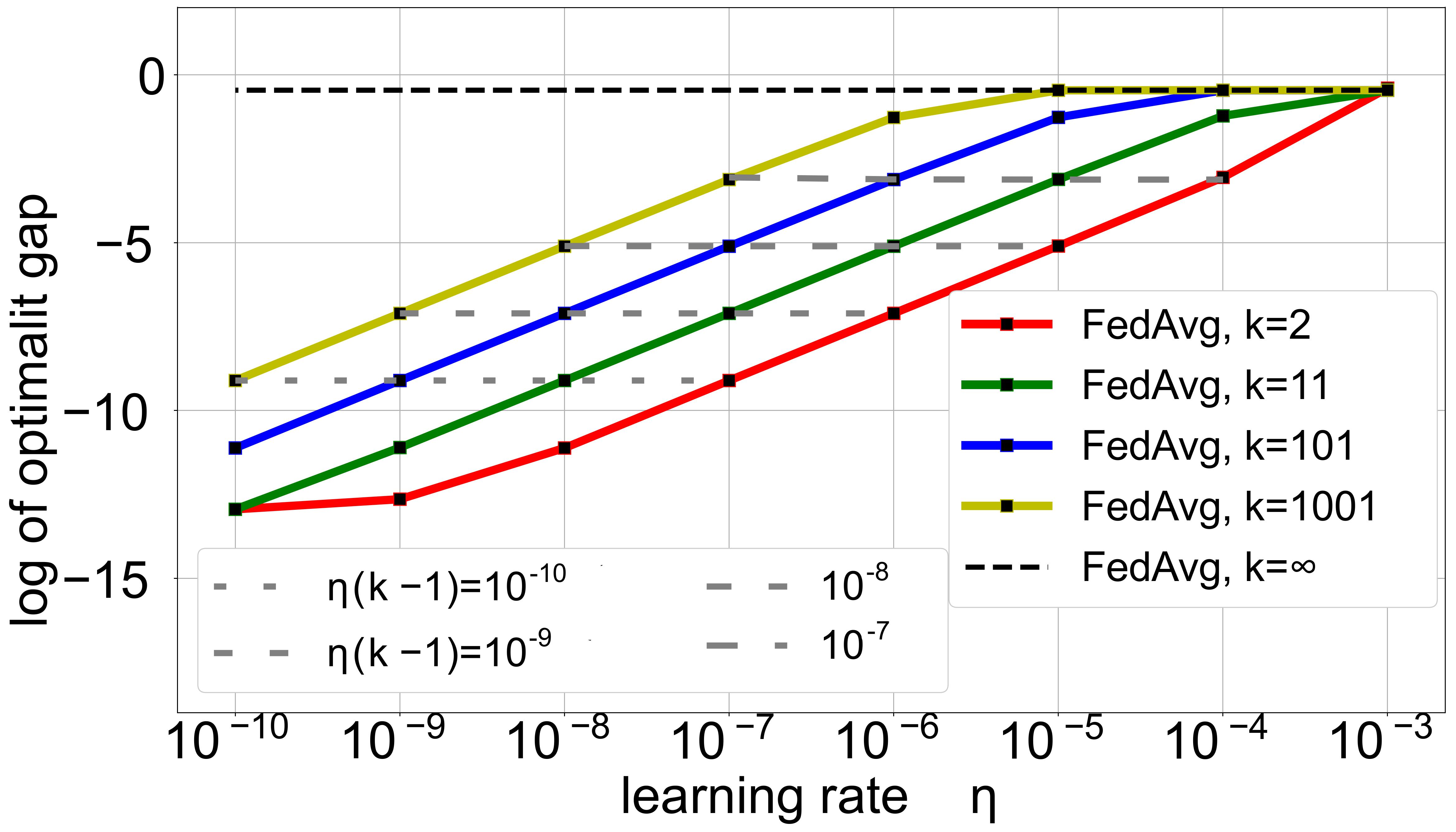}
\includegraphics[width=0.32\columnwidth]{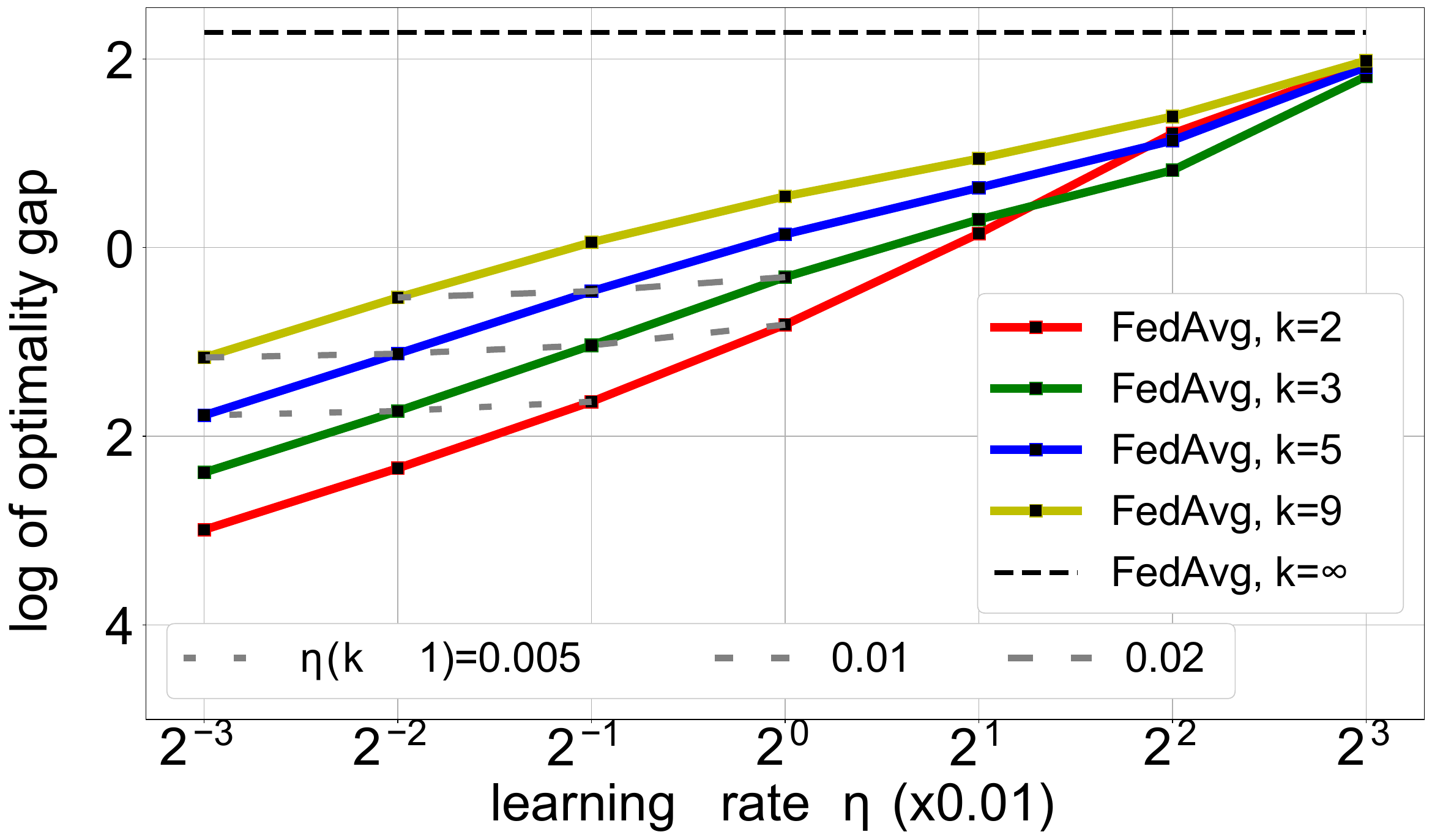}
\includegraphics[width=0.32\columnwidth]{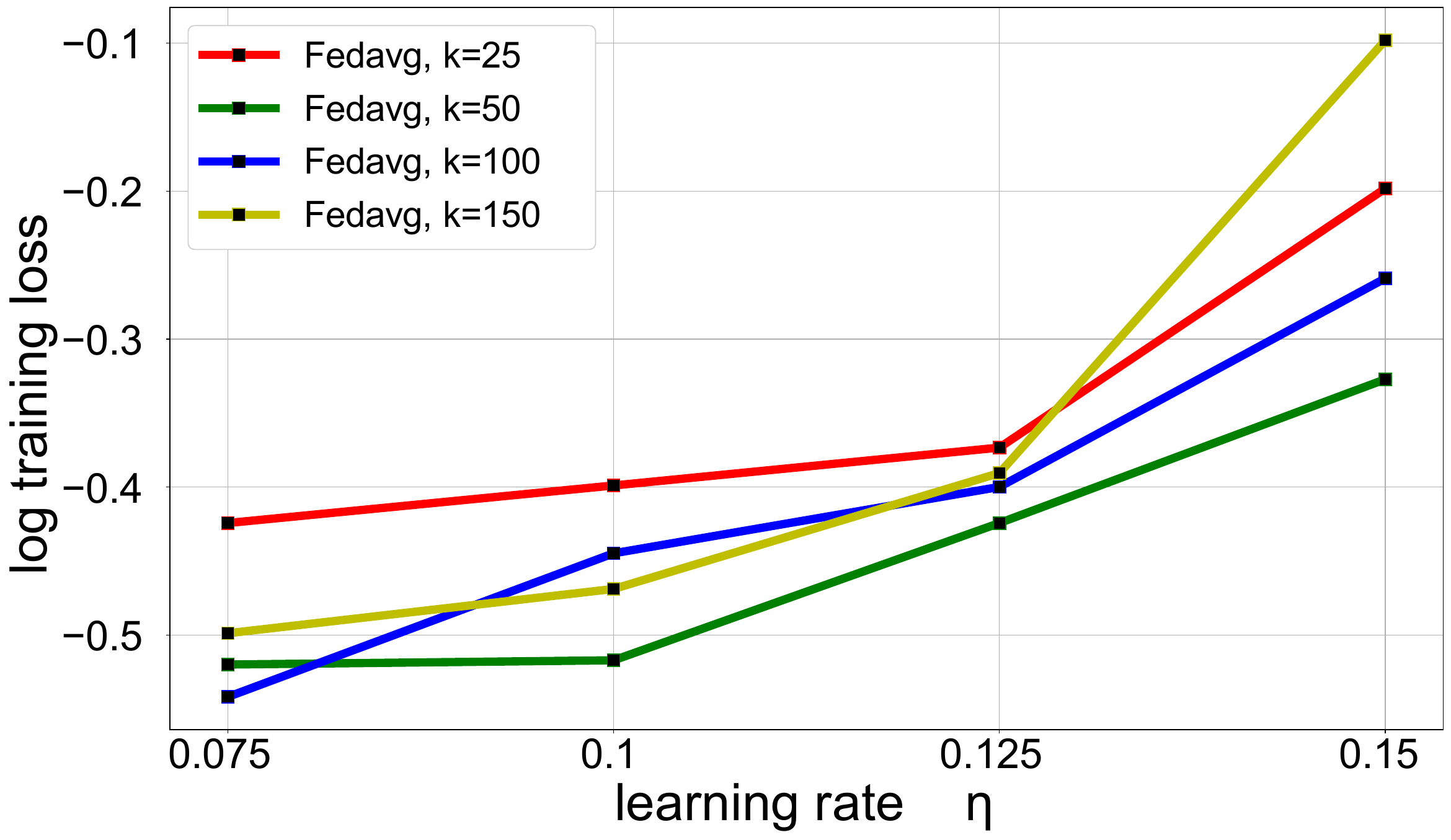}
\caption{Optimality gap $\{f(\wv_{\FA}^{*})-f^{*}\}$ or training loss $\{f(\wv_{\FA}^{*})\}$ of (approximate) fixed-point solutions of \FA for different learning rates $\eta$ and local epochs $k$. Different colored lines are for different numbers of local epochs,
and dashed lines for different product values $\eta(k-1)$. Left: least squares (closed-form solution); Middle: logistic regression ($6000$ communication rounds); Right: nonconvex CNN on the MNIST dataset ($200$ communication rounds). } 

\label{fig:fixed_pr_stationary}
\end{figure}

\subsection{\FP as backward-backward splitting}
\label{sec:fp}
The recent \FP algorithm  \cite{tian_2018_fedprox} replaces the gradient update in \FA with a proximal update:
\begin{align}
\label{eq:BB}
\wbs_{t+1} = \prox[]{H}\prox[\eta_t]{\fs}\wbs_t,
\end{align}
where as before we may use a minibatch to approximate $\prox[\eta]{\fs}$ or select a subset of users to approximate $\prox[]{H}$. Written in this way, it is clear that \FP is an instantiation of the backward-backward splitting algorithm of \textcite{Lions78,Passty79}. In fact, this algorithm traces back to the early works of \eg \textcite{Cimmino38,LionsTemam66,Auslender69}, sometimes under the name of the Barycenter method. It was also rediscovered by \textcite{Yu13a,YuZMX15} in the ML community under a somewhat different motivation. \textcite{pathak2020fedsplit} pointed out that \FP does not solve the original problem \eqref{eq:wFL}. While technically correct, their conclusion did not take many other subtleties into account, which we explain next.

Following \textcite{BauschkeCR05}, we first note that, with a constant step size $\eta_t \equiv \eta$, \FP is actually equivalent as \FA but applied to a ``regularized'' problem:
\begin{align}
\label{eq:regFL}
\min_{\wbs\in H} ~~ \tilde \fs(\wbs), ~~\where~~\tilde \fs(\wbs) := \inner{\one}{\env[\eta]{\fv}(\wbs)}, ~ \mbox{ and } ~ \env[\eta]{\fv}(\wbs) = \big( \env[\eta]{f_1}(\wv_1), \ldots, \env[\eta]{f_m}(\wv_m) \big).
\end{align}
Interestingly, \textcite{DinhTN20} proposed exactly \eqref{eq:regFL} for the purpose of personalization, which we now realize is automatically achieved if we apply \FP to the original formulation \eqref{eq:wFL}. Indeed, $\nabla \env[\eta]{\fv}(\wbs) = [\wbs - \prox[\eta]{\fv}(\wbs) ] / \eta$ hence $\grad[\eta]{\tilde f}(\wbs) = \wbs - \eta \nabla \tilde\fs(\wbs) =  \prox[\eta]{\fv}(\wbs)=\big( \prox[\eta]{f_1}(\wv_1), \ldots, \prox[\eta]{f_m}(\wv_m) \big)$. This simple observation turns out to be crucial in understanding \FP. 

Indeed, a significant challenge in \FL is user heterogeneity (a.k.a. non-iid distribution of data), where the individual user functions $f_i$ may be very different (due to distinct user-specific data). 
But, we note that\footnote{These results are classic and well-known, see \eg \textcite{RockafellarWets98}.} as $\eta \to 0$, $\env[\eta]{f} \to f$ (pointwise or uniformly if $f$ is Lipschitz) while $\env[\eta]{f} \to \min f$ as $\eta \to \infty$. In other words, a larger $\eta$ in \eqref{eq:regFL} ``smoothens'' heterogeneity in the sense that the functions $\env[\eta]{f_i}$ tend to have similar minimizers (while the minimum values may still differ). We are thus motivated to understand the convergence behaviour of \FP, for small and large $\eta$, corresponding to the original problem \eqref{eq:wFL} and the smoothened problem, respectively. In fact, we can even adjust $\eta$ dynamically. Below, for simplicity, we assume full gradient (\ie large batch size) although extensions to stochastic gradient are possible.

\begin{restatable}{theorem}{fphom}
Assuming each user participates indefinitely, the step size $\eta_t$ is bounded from below (\ie $\liminf_t \eta_t > 0$), the user functions $\{f_i\}$ are convex, and homogeneous in the sense that they have a common minimizer, \ie
$\Fm := \bigcap_i \argmin_{\wv_i} ~ f_i(\wv_i) \ne \emptyset$, 
then the iterates of \FP converge to a point in $\Fm$, which is a correct solution of \eqref{eq:wFL}.
\label{thm:fp-hom}
\end{restatable}
The homogeneity assumption
might seem strong since it challenges the necessity of \FL in the first place. However, we point out that (a) as a simplified limiting case it does provide insight on when user functions have close minimizers (\ie homogeneous); (b) by using a large step size $\eta$ \FP effectively homogenizes the users and its convergence thus follows\footnote{The consequence of this must of course be further investigated; see \textcite{DinhTN20} and \textcite{tian_2018_fedprox}.}; (c) modern deep architectures are typically over-parameterized so that achieving null training loss on each user is not uncommon. However, \FL is still relevant even in this case since it selects a model that works for \emph{every} user and hence provides some regularizing effect; (d) even when the homogeneity assumption fails, $\prox[\eta_t]{\fs}$ still converges to a point that is in some sense closest to $H$ \parencite{BauschkeCR05}.

The next result removes the homogeneity assumption by simply letting step size $\eta_t$ diminish:
\begin{restatable}[\cite{Lions78,Passty79}]{theorem}{fp}
Assuming each user participates in every round, the step size $\eta_t$ satisfies $\sum_t \eta_t = \infty$ and $\sum_t \eta_t^2 < \infty$, and the functions $\{f_i\}$ are convex, then the \emph{averaged} iterates 
$\bar\wbs_t := \tfrac{\sum_{s=1}^t \eta_s \wbs_s}{\sum_{s=1}^t \eta_s}$
of \FP converge to a \emph{correct} solution of the \emph{original} problem \eqref{eq:wFL}.
\label{thm:fp}
\end{restatable}
In \Cref{sec:app-proof}, we give examples to show the tightness of the step size condition in \Cref{thm:fp}. We emphasize that $\bar\wbs_t$ converges to a solution of the original problem \eqref{eq:wFL}, not the regularized problem \eqref{eq:regFL}. The subtlety is that we must let the step size $\eta_t$ approach 0 reasonably slowly, a possibility that was not discussed in \textcite{pathak2020fedsplit} where they always fixed the step size $\eta_t$ to a constant $\eta$. \Cref{thm:fp} also makes intuitive sense, as $\eta_t \to 0$ slowly, we are effectively tracking the solution of the regularized problem \eqref{eq:regFL} which itself tends to the original problem \eqref{eq:wFL}: recall that $\env[\eta]{f} \to f$ as $\eta\to 0$.

Even the ergodic averaging step in \Cref{thm:fp} can be omitted in some cases:

\begin{restatable}[\cite{Passty79}]{theorem}{fperg}
Under the same assumptions as in \Cref{thm:fp}, if $\fv$ is strongly convex or the solution set of  \eqref{eq:wFL} has nonempty interior, then the vanilla iterates $\wbs_t$ of \FP also converge.
\end{restatable}
We remark that convergence is in fact linear for the second case. Nevertheless, 
the above two conditions are perhaps not easy to satisfy or verify in most applications. Thus, in practice, we recommend the ergodic averaging in \Cref{thm:fp} since it does not create additional overhead (if implemented incrementally) and in our experiments it does not slow down the algorithm noticeably.

There is in fact a quantitative relationship between the original problem \eqref{eq:FL-prod} and the regularized problem \eqref{eq:regFL}, even in the absence of convexity:
\begin{restatable}[\cite{YuZMX15}]{theorem}{fpnvx}
Suppose each $f_i$ is $M_i$-Lipschitz continuous (and possibly nonconvex), then
\begin{align}
\textstyle
\forall \wbs, ~~~ \fs(\wbs) - \tilde \fs(\wbs) \leq \sum_i \lambda_i \tfrac{\eta M_i^2}{2}.
\end{align}
\end{restatable}
Thus, for a small step size $\eta$, \FP, aiming to minimize the regularized function $\tilde \fs$ is not quantitatively different from \FA that aims at minimizing $\fs$. This again reveals the fundamental importance of step size in \FL.
We remark that, following the ideas in \textcite{YuZMX15}, we may remove the need of ergodic averaging for \FP if the functions $f_i$ are ``definable'' (which most, if not all, functions used in practice are). We omit the technical complication here since the potential gain does not appear to be practically significant.

In \Cref{fig:eta_effect_fedprox}, we show the effect of step size $\eta$ on the convergence of \FP, and compare the  results with that of \FA on both convex and nonconvex models. 
We run \FP with both fixed and diminishing step sizes. In the experiments with diminishing step sizes, we have set the initial value of $\eta$ (i.e. $\eta_0$) to larger values - compared to the constant $\eta$ values - to ensure that $\eta$ does not get very small after the first few rounds. From the convex experiments (Fig~\ref{fig:eta_effect_fedprox}, left and middle),  one can see that \FP with a fixed learning rate converges fast (in a few hundred rounds)  to a suboptimal solution. In contrast, \FP with diminishing $\eta$ converges slower, but to better quality solutions. It is interesting to note that only when the conditions of \Cref{thm:fp} are satisfied (i.e. $\eta$ diminishes neither too fast nor too slow), \FP converges to a correct solution of the original problem \eqref{eq:wFL}, \eg see the results for $\eta_t \propto 1/t$ which satisfies both conditions in \Cref{thm:fp}. Surprisingly, for the nonconvex setting (\Cref{fig:eta_effect_fedprox}, right), the best results are achieved with larger learning rates. A similar observation about \FA in nonconvex settings was reported in \citep[Fig. 5 \& 6]{McMahanMRHA17}. \footnote{Note that the results of \FA and \FP for $\eta=100$ and $100/\log (t)$ overlap with each other.} Moreover, from the results on both convex and nonconvex models, one can see that  ergodic averaging does not affect the convergence rate of \FP noticeably.  

\begin{figure}[t]
 \includegraphics[width=0.33\columnwidth]{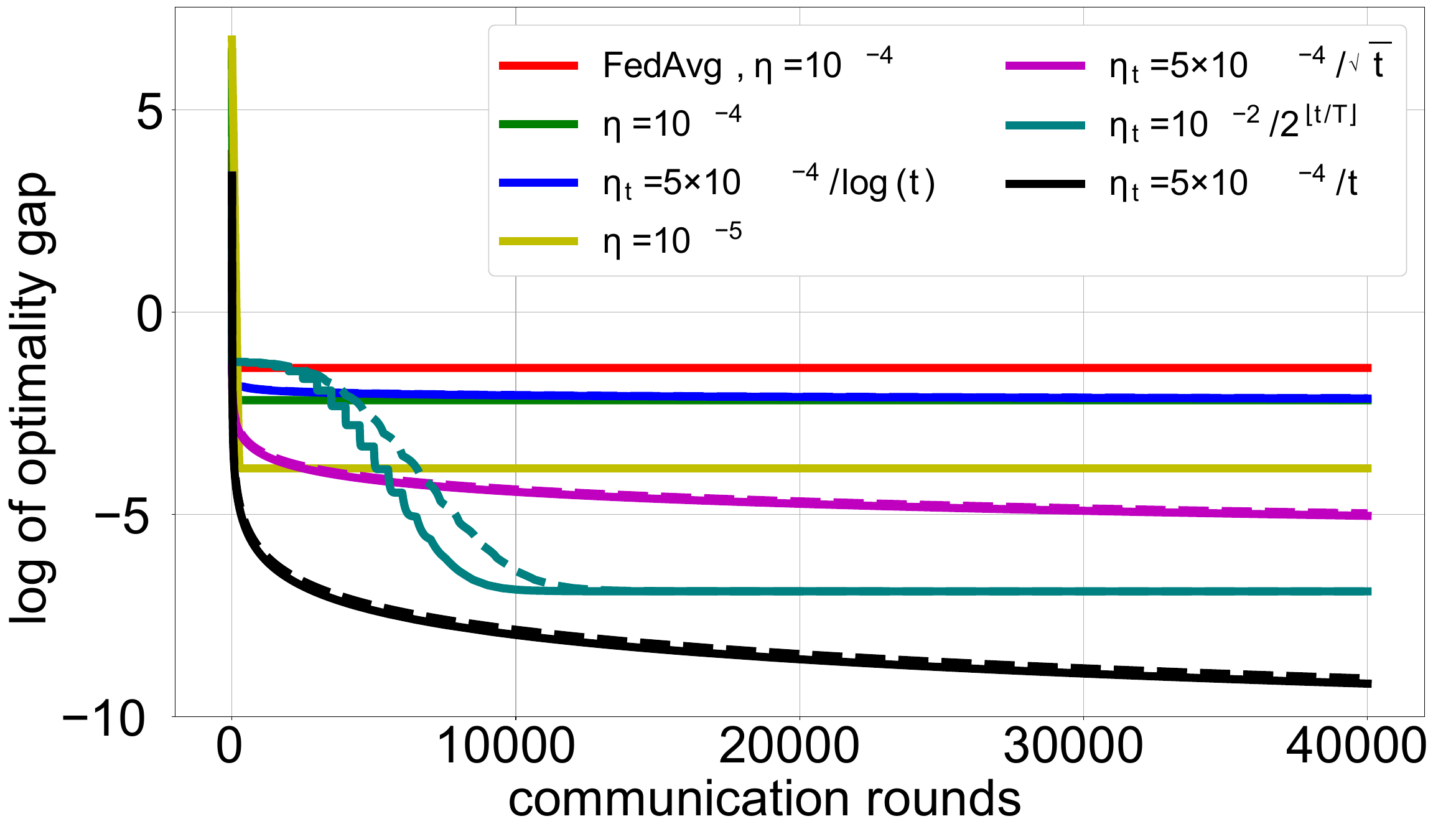}
 \includegraphics[width=0.32\columnwidth]{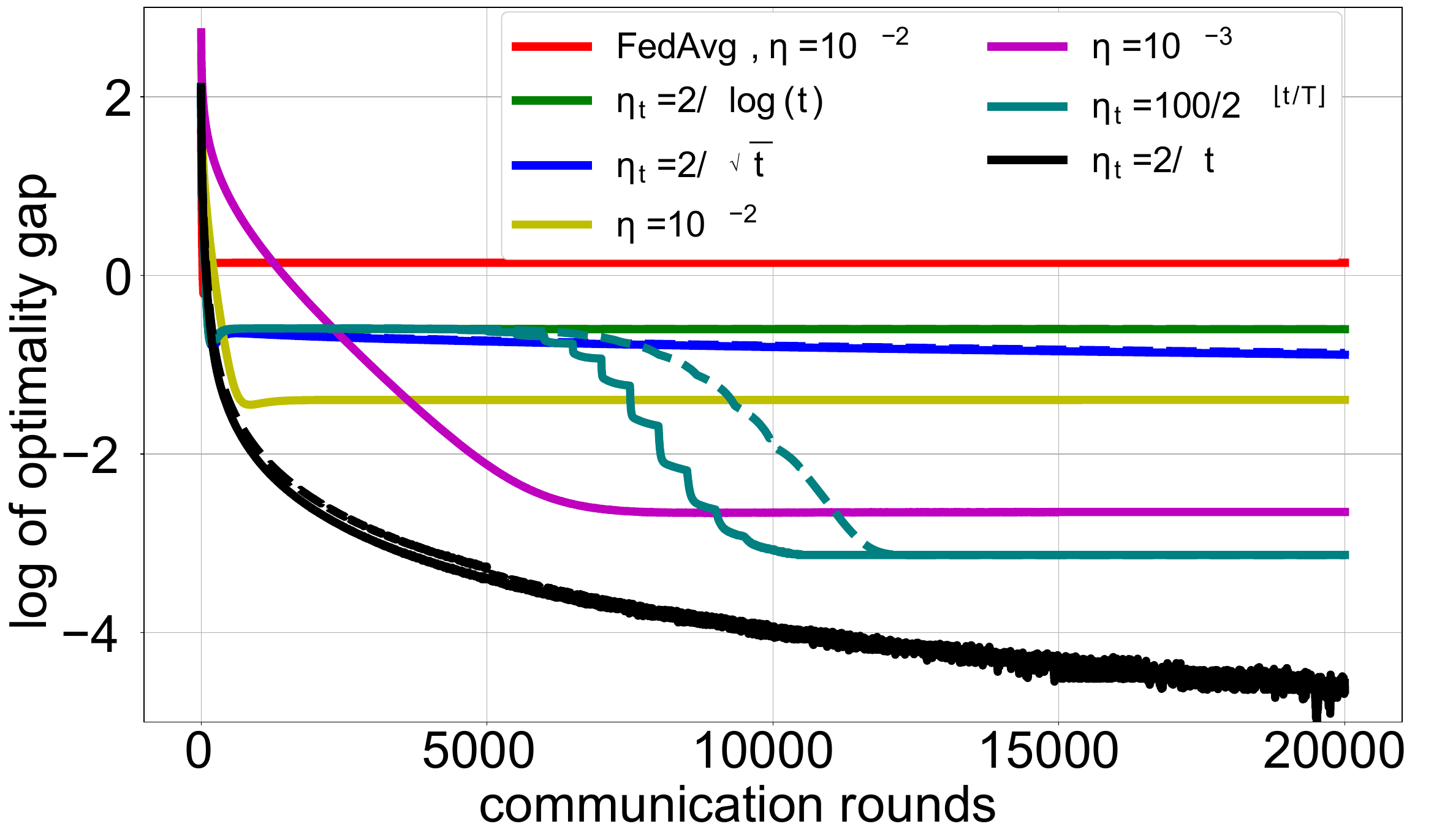}
 \includegraphics[width=0.32\columnwidth]{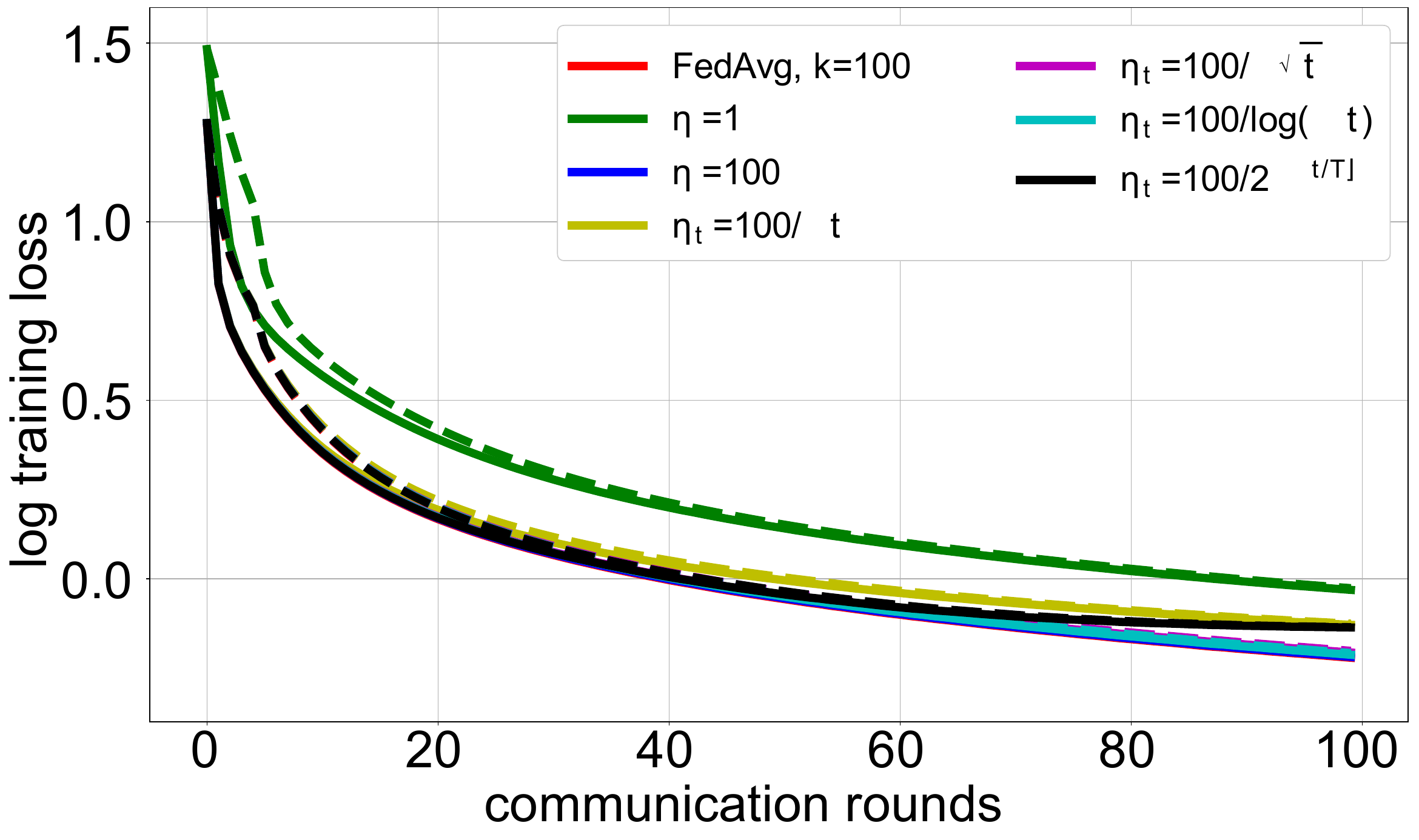}
\caption{Effect of step size $\eta$ and averaging on \FP. Left: least squares; Middle: logistic regression; Right: CNN on MNIST. The dashed and solid lines with the same color show the results obtained with and without the ergodic averaging step in \Cref{thm:fp}, respectively. For exponentially decaying $\eta_t$, we use period $T$ equal to $500$ for both least squares and logistic regression experiments, and $10$ for CNN experiment.
}
\label{fig:eta_effect_fedprox}
\end{figure}

\subsection{\FS as Peaceman-Rachford splitting}
\textcite{pathak2020fedsplit} introduced the \FS algorithm recently:
\begin{align}
\label{eq:PR}
\wbs_{t+1} = \ra[]{H} \ra[\eta_t]{\fs}\wbs_t,
\end{align}
which is essentially an instantiation of the Peaceman-Rachford splitting algorithm \parencite{PeacemanRachford55,LionsMercier79}. As shown by \textcite{LionsMercier79}, \FS converges to a correct solution of \eqref{eq:wFL} if $\fv$ is strictly convex, and the convergence rate is linear if $\fv$ is strongly convex and smooth (and $\eta$ is small). 
\textcite{pathak2020fedsplit} also studied the convergence behaviour of \FS when the reflector $\ra[\eta]{\fs}$ is computed approximately. 
However, we note that  convergence behaviour of \FS is not known or widely studied for nonconvex problems. In particular, we have the following surprising result:
\begin{restatable}{theorem}{rcon}
If the reflector $\ra[\eta]{f}$ is a (strict) contraction, then $f$ must be strongly convex.
\end{restatable}
The converse is true if $f$ is also smooth and $\eta$ is small \parencite{LionsMercier79,Gabay83}. 
Therefore, for non-strongly convex or nonconvex problems, we cannot expect \FS to converge linearly (if it converges at all).

\subsection{\FPI as Douglas-Rachford splitting}
A popular alternative to the Peaceman-Rachford splitting is the Douglas-Rachford splitting \parencite{DouglasRachford56,LionsMercier79}, which, to our best knowledge, has not been applied to the \FL setting. The resulting update, which we call \FPI, can be written succinctly as:
\begin{align}
\label{eq:DR}
\wbs_{t+1} = \tfrac{\wbs_t + \ra{H}\ra[\eta_t]{\fs}\wbs_t }{2},
\end{align}
\ie we simply average the current iterate and that of \FS evenly. Strictly speaking, the above algorithm is a special case of the Douglas-Rachford splitting and was rediscovered by \textcite{Spingarn83} under the name of partial inverse (hence our name \FPI). The moderate averaging step in \eqref{eq:DR} makes \FPI much more stable:
\begin{theorem}[\cite{Spingarn83,LionsMercier79}]
Assuming each user participates in every round, the step size $\eta_t \equiv \eta$ is constant, and the functions $\{f_i\}$ are convex, then the vanilla iterates $\wbs_t$ of \FPI converge to a \emph{correct} solution of the \emph{original} problem \eqref{eq:wFL}.
\end{theorem}
Compared to \FS, \FPI imposes less stringent condition on $f_i$. However, when $f_i$ is indeed strongly convex and smooth, as already noted by \textcite{LionsMercier79}, \FPI will be slower than \FS by a factor close to $\sqrt{2}$ (assuming the constant step size is set appropriately for both). More importantly, it may be easier to analyze \FPI on nonconvex functions, as recently demonstrated in \cite{Rockafellar19b,ThemelisPatrinos20}. 

Interestingly, \FPI also has close ties to \FP. Indeed, this is best seen by  expanding the concise formula in \eqref{eq:DR} and introducing a ``dual variable'' $\ubs$ on the server side\footnote{The acute readers may have recognized here the alternating direction method of multipliers (ADMM). Indeed, the equivalence of ADMM, Douglas-Rachford and partial inverse (under our \FL setting) has long been known \parencite[\eg][]{Gabay83,EcksteinBertsekas92}.}: 
\begin{align}
\label{eq:pi-1}
\zbs_{t+1} &\gets \prox[\eta]{\fs}(\wbs_t + \ubs_{t}) \\
\label{eq:pi-2}
\wbs_{t+1} & \gets \prox[]{H}(\zbs_{t+1} \gray{- \ubs_{t}}) \\
\label{eq:pi-3}
\ubs_{t+1} &\gets \ubs_t + \wbs_{t+1} - \zbs_{t+1}. 
\end{align}
From the last two updates \eqref{eq:pi-2} and \eqref{eq:pi-3} it is clear that $\ubs_{t+1}$ is always in $H^\perp$. Thus, after performing a change of variable $\vbs_t := \wbs_t + \ubs_t$ and exploiting the linearity of $\prox[]{H}$, we obtain exactly \FPI:
\begin{align}
\vbs_{t+1} = \ubs_t + 2\wbs_{t+1} - \zbs_{t+1} 
= \vbs_t - \prox[]{H}\vbs_t + 2\prox[]{H}\prox[\eta]{\fs}\vbs_t - \prox[\eta]{\fs}\vbs_t = \tfrac{\vbs_t + \ra[]{H}\ra[\eta]{\fs} \vbs_t}{2}.
\end{align}
Comparing \eqref{eq:BB} and \eqref{eq:pi-1}-\eqref{eq:pi-2} it is  clear that \FP corresponds to fixing the dual variable $\ubs$ to the constant $\zero$ in \FPI. We remark that step \eqref{eq:pi-1} is done at the users' side while steps \eqref{eq:pi-2} and \eqref{eq:pi-3} are implemented at the server side. There is no communication overhead either, as the server need only communicate the sum $\wbs_t + \ubs_t$ to the respective users while the users need only communicate their $\zv_{t, i}$ to the server. The dual variable $\ubs$ is kept entirely at the server's expense.

Let us point out another subtle difference that may prove useful in \FL: \FA and \FP are inherently ``synchronous'' algorithms, in the sense that all participating users start from a common, averaged model at the beginning of each communication round. In contrast, the local models $\zbs_t$ in \FPI may be different from each other, where we ``correct'' the common, average model $\wbs_t$ with user-specific dual variables $\ubs_t$. This opens the possibility to personalization by designating the dual variable $\ubs$ in user-specific ways.

Lastly, we remark that in \FP we need the step size $\eta_t$ to diminish in order to converge to a solution of the original problem \eqref{eq:wFL} whereas \FPI achieves the same with a constant step size $\eta$, although at the cost of doubling the memory cost at the server side. 

\subsection{\FR as Reflection-Projection splitting}
Examining the updates in \eqref{eq:FB}, \eqref{eq:BB} and \eqref{eq:PR}, we are naturally led to the following further variants (that have not been tried in \FL to the best of our knowledge):
\begin{align}
\label{eq:rg}
\wbs_{t+1} &= \ra{H}\grad[\eta_t]{\fs}\wbs_t \\
\label{eq:rpp}
\wbs_{t+1} &= \ra{H}\prox[\eta_t]{\fs}\wbs_t \\
\label{eq:rp}
\mbox{\FR}: \qquad 
\wbs_{t+1} &= \prox{H}\ra[\eta_t]{\fs}\wbs_t. 
\end{align}

Interestingly, the last variant in \eqref{eq:rp}, which we call \FR, has been studied by \textcite{BauschkeKruk04} under the assumption that $\fs = \iota_{\Ksf}$ is an indicator function of an \emph{obtuse}\footnote{Recall that a convex cone $\Ksf$ is obtuse iff its dual cone $\Ksf^* := \{\wv^*: \inner{\wv}{\wv^*} \geq 0\}$ is contained in itself.} convex cone $\Ksf$. We prove the following result for \FR: 
\begin{restatable}{theorem}{fr}
Let each user participate in every round and the functions $\{f_i\}$ be convex. 
If the step size $\eta_t \equiv \eta$ is constant, then any fixed point of \FR is a solution of the \emph{regularized} problem \eqref{eq:regFL}. 
If the reflector $\ra[\eta]{\fs}$ is idempotent (\ie $\ra[\eta]{\fs} \ra[\eta]{\fs} = \ra[\eta]{\fs}$) and the users are homogeneous, 
then the vanilla iterates $\wbs_t$ of \FR converge. 

\label{thm:rp}
\end{restatable}
It follows easily from a result in \textcite{BauschkeKruk04} that a convex cone is obtuse iff its reflector is idempotent, and hence \Cref{thm:rp} strictly extends their result. We remark that \Cref{thm:rp} does not apply to the variants 
\eqref{eq:rg} and \eqref{eq:rpp} 
since the reflector $\ra{H}$ is not idempotent (recall $H$ is defined in \eqref{eq:FL-prod}). Of course, we can prove (linear) convergence of both variants 
\eqref{eq:rg} and \eqref{eq:rpp} 
if $\fs$ is strongly convex. We omit the formal statement and proof since strong convexity does not appear to be easily satisfiable in practice.

In \Cref{fig:comp}, we  compare the performance of different splitting algorithms and how they respond to different degrees of user heterogeneity. We use least squares and a CNN model on MNIST for the convex and noconvex experiments, respectively. The details of the experimental setup  are described in \Cref{sec:lslr_exp_setup}. For the convex setting,  as expected, \FS, \FPI, and \FA with $k=1$ achieve the smallest optimality gaps. The performance of all the algorithms deteriorates as users' data become more heterogeneous (see \Cref{sec:lslr_exp_setup}). For the nonconvex setting, the best results can be achieved by \FR and \FA with a big $k$. It is noteworthy that in the nonconvex setting, the performance of \FA with $k=1$ is significantly worse than that with $k=100$.

\begin{figure}[t]
    \centering
    
    \includegraphics[width=0.32\columnwidth]{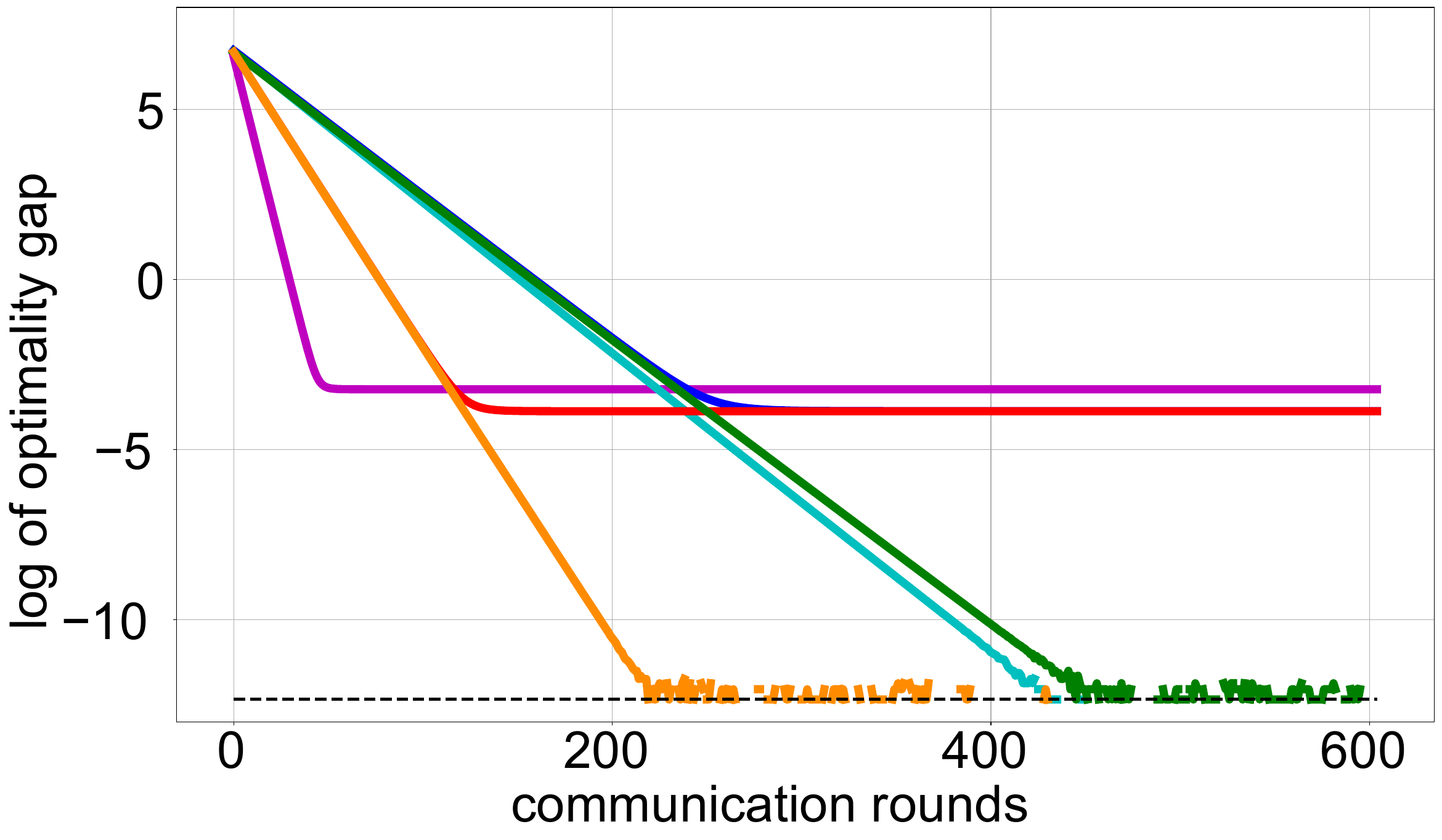}
    \includegraphics[width=0.32\columnwidth]{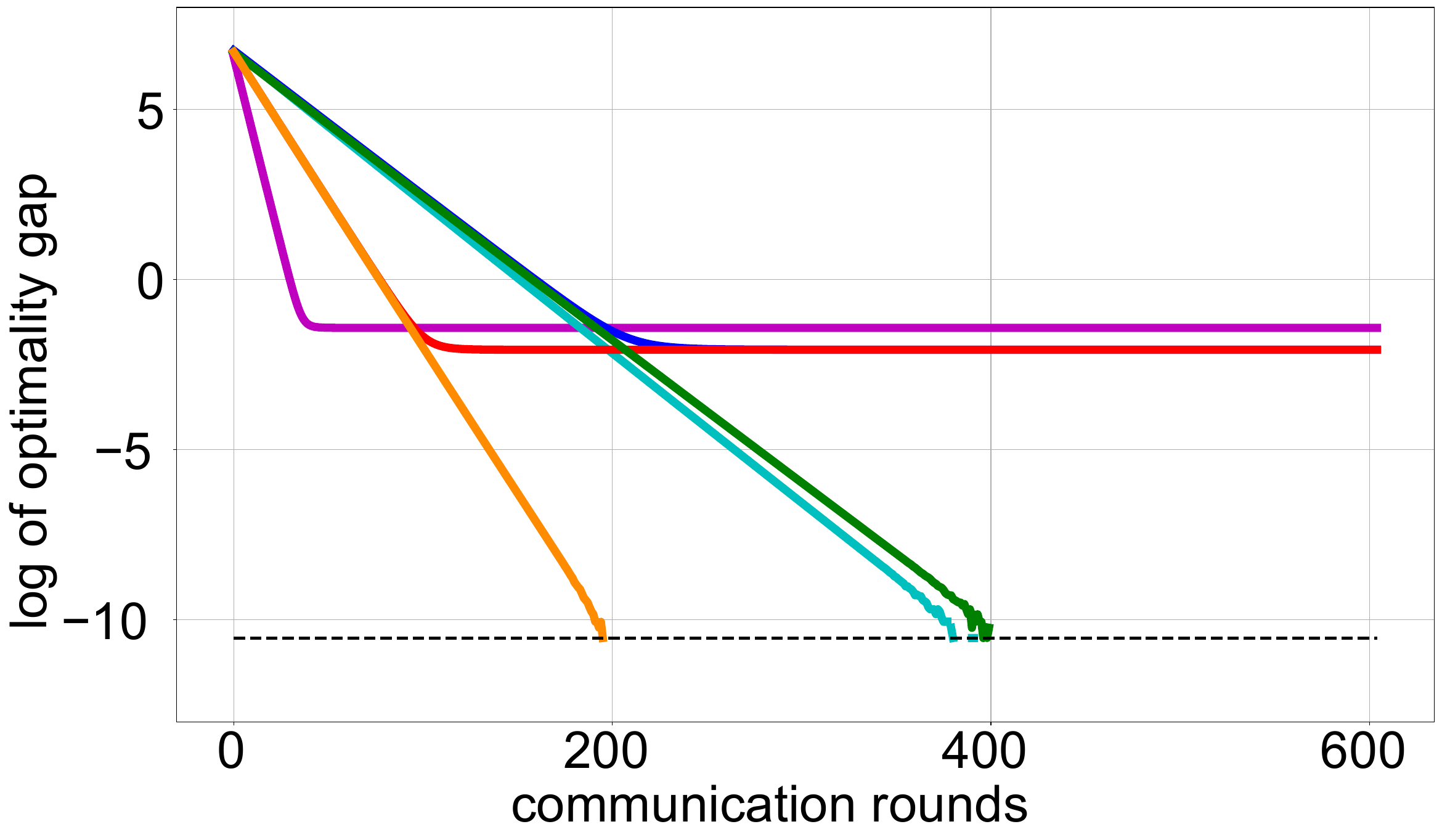}
    \includegraphics[width=0.32\columnwidth]{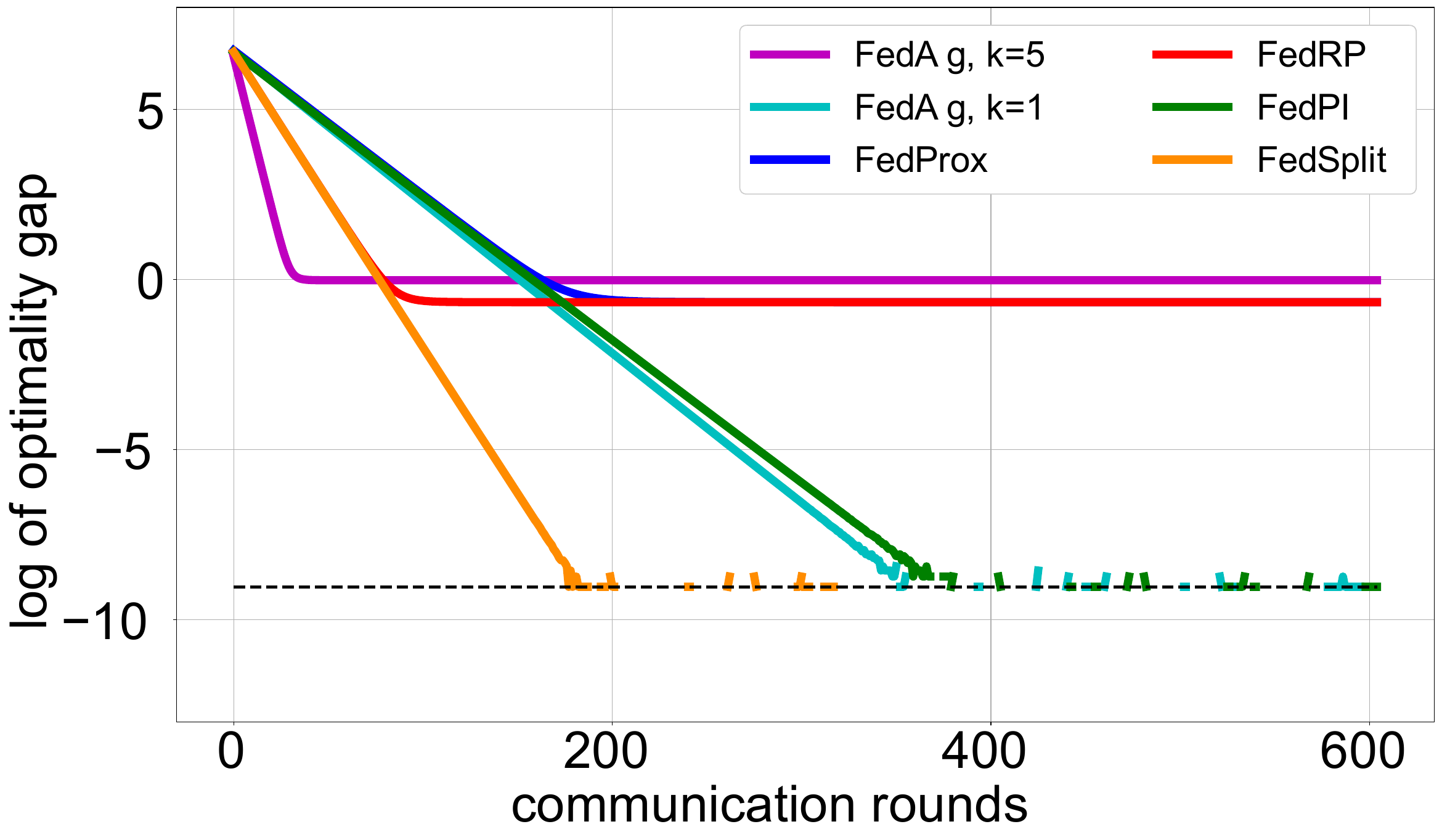}
    \includegraphics[width=0.32\columnwidth]{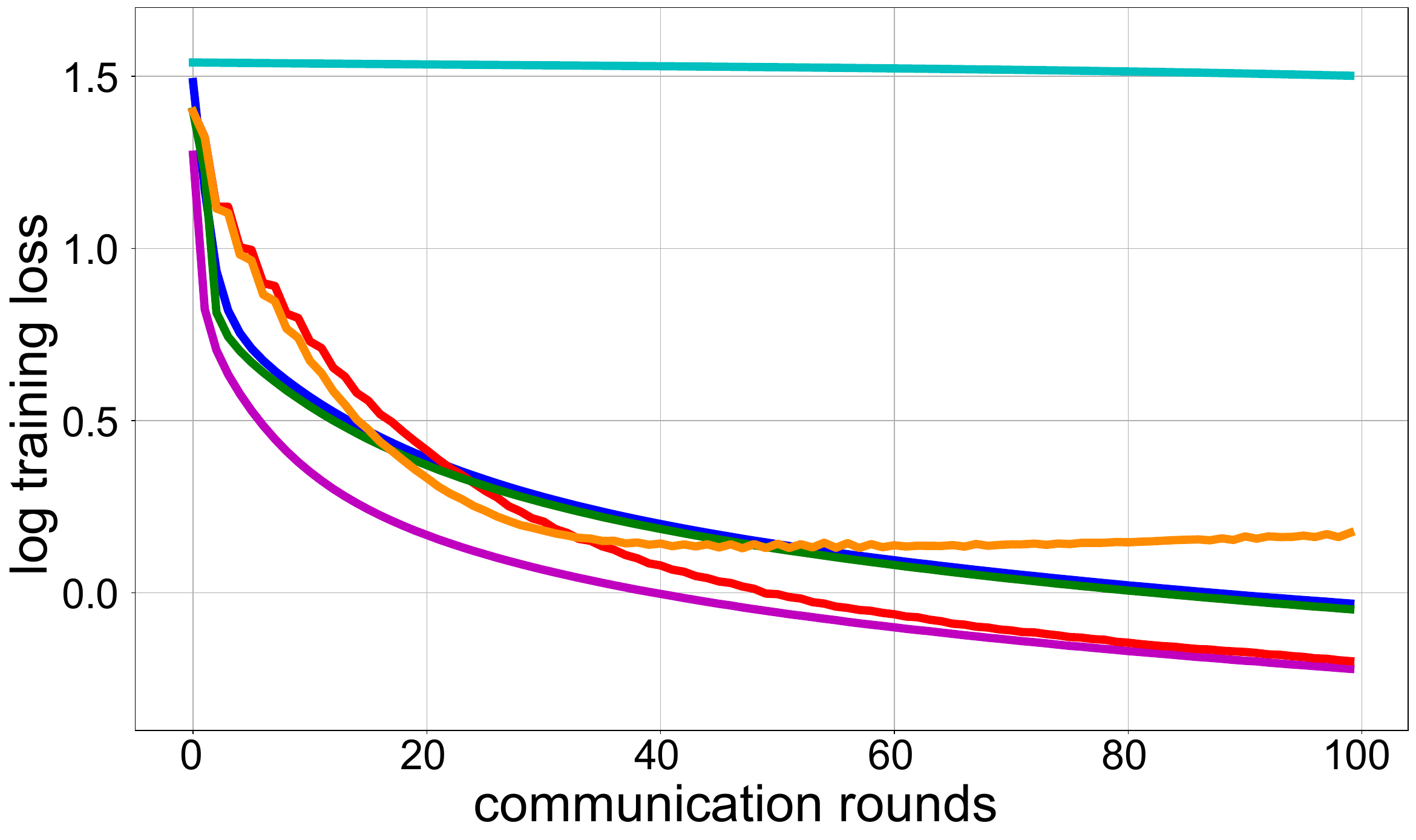}
    \includegraphics[width=0.32\columnwidth]{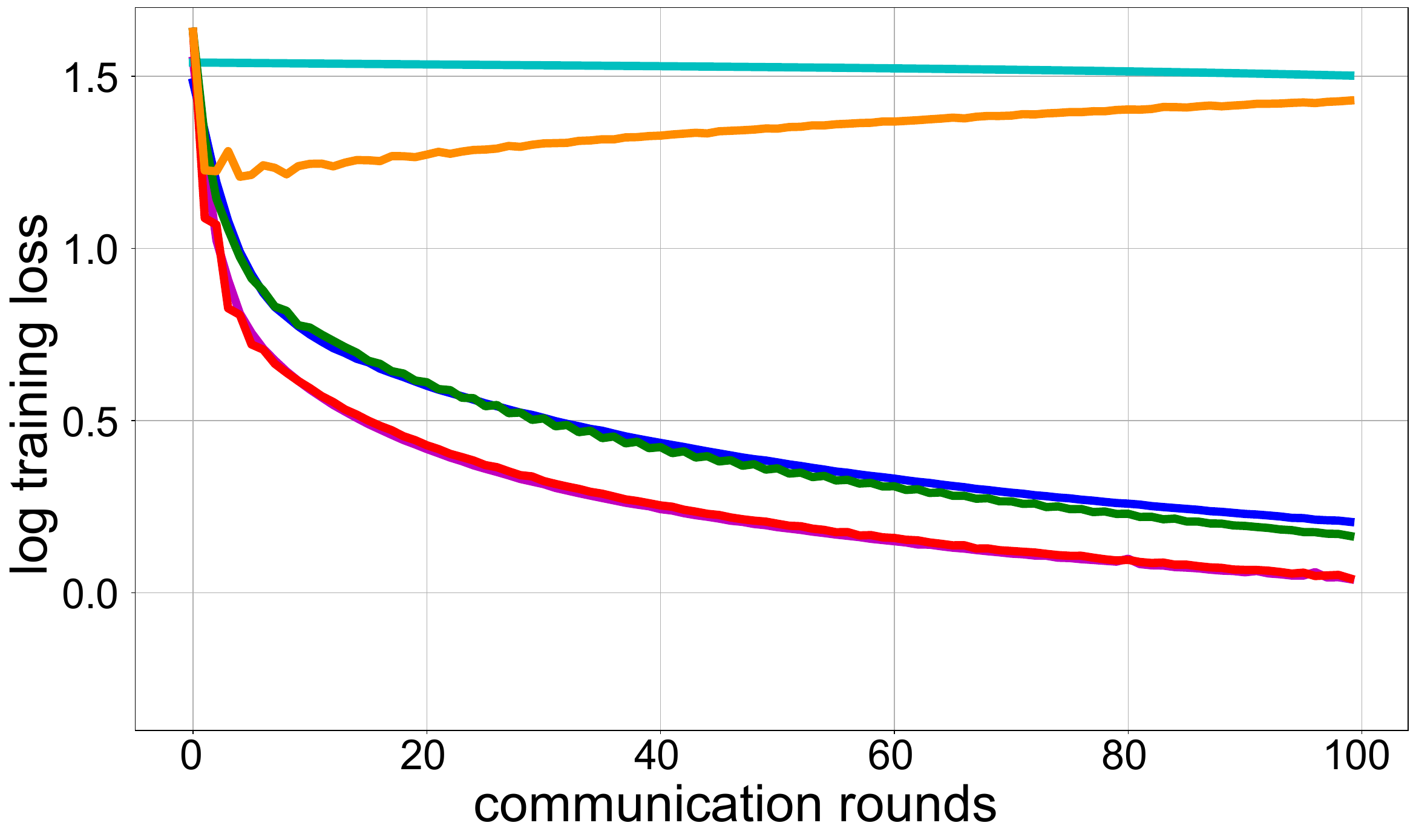}
    \includegraphics[width=0.32\columnwidth]{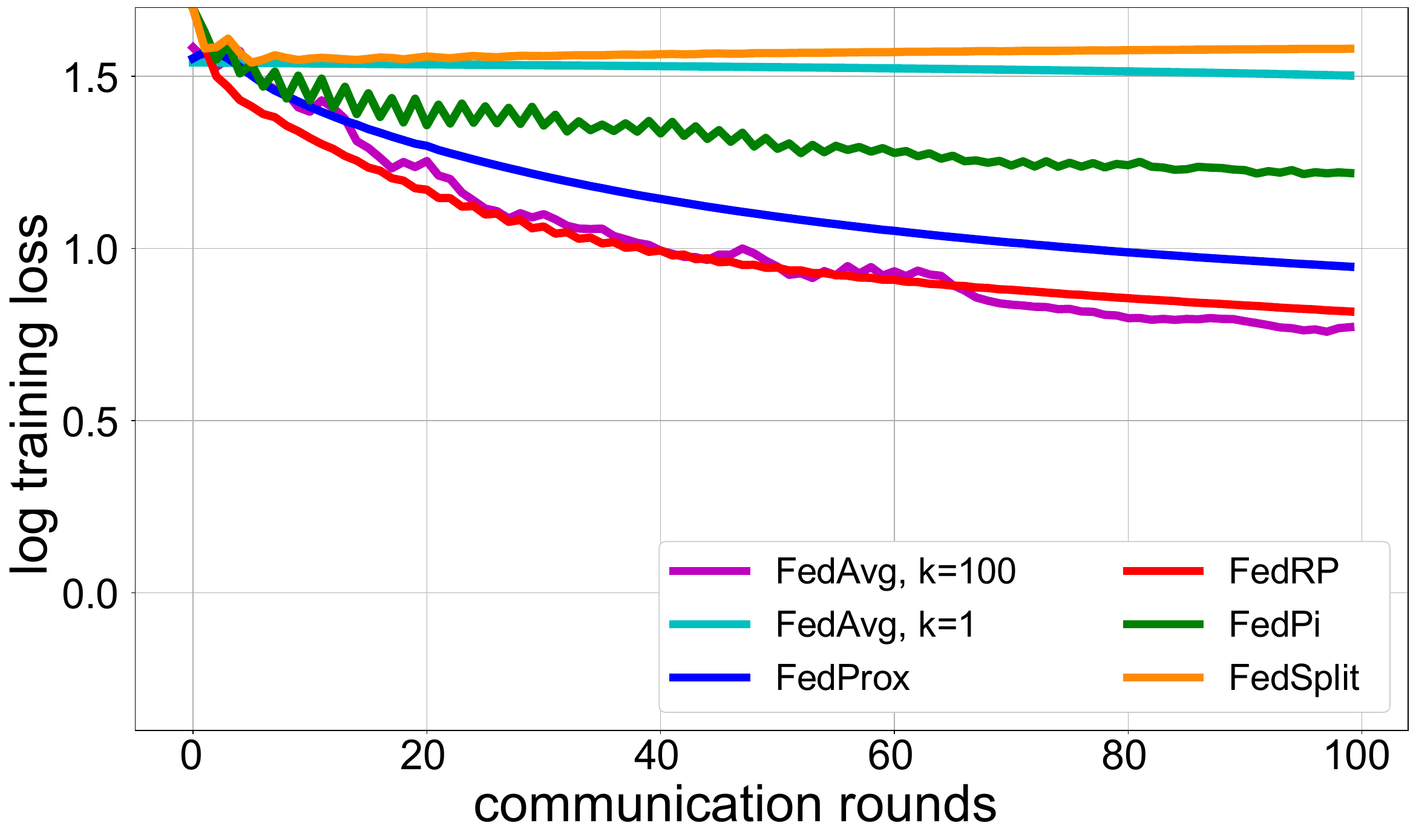}

    \caption{The effect of data heterogeneity on the performance of different splitting methods. The top row shows the results for the least squares, and the bottom row shows the results for nonconvex CNN model. Top-Left: small data heterogeneity  with $H\approx 119\times10^3$; Top-Middle: moderately data heterogeneity with $H\approx7.61\times10^{6}$; Top-Right: large data heterogeneity with $H\approx190.3\times10^{6}$. Bottom-Left: i.i.d. data distribution; Bottom-Middle: non-i.i.d. data distribution with maximum $6$ classes per user; Bottom-Right: non-i.i.d. data distribution with maximum $2$ classes per user.}
    \label{fig:comp}
\end{figure}

\section{Unification, implementation and acceleration}
\label{sec:acc}
The operator splitting view of \FL not only allows us to compare and understand the many existing \FL algorithms but also opens the door for unifying and accelerating them. 
In this section we first introduce a grand scheme that unifies all aforementioned \FL algorithms. Then, we provide two implementation variants that can accommodate different hardware and network constraints. Lastly, we explain how to adapt Anderson acceleration  \parencite{Anderson65} to the unified \FL algorithm, with almost no computation or communication overhead.

\paragraph{Unification.} We introduce the following grand scheme:
\begin{align}
\label{eq:gFL-1}
\zbs_{t+1} &= (1-\alpha_t)\ubs_t + \alpha_t \prox[\eta_t]{\fs}(\ubs_t) \\
\label{eq:gFL-2}
\wbs_{t+1} &= (1-\beta_t)\zbs_{t+1} + \beta_t \prox[]{H}(\zbs_{t+1}) \\
\label{eq:gFL-3}
\ubs_{t+1} &= (1-\gamma_t) \ubs_{t} + \gamma_t \wbs_{t+1}.
\end{align}
\Cref{tab:FL} confirms that the \FL algorithms discussed in \Cref{sec:sp} are all special cases of this unifying scheme, which not only provides new (adaptive) variants but also clearly reveals the similarities and differences between seemingly different algorithms. It may thus be possible to transfer progress on one algorithm to the others and vice versa. We  studied the effect of $\alpha$, $\beta$ and $\gamma$ and found that $\gamma$ mostly affects the convergence speed: the closer $\gamma$ is to 1, the faster the convergence is, while $\alpha$ and $\beta$ mostly determine the final optimality gap: the closer they are \emph{both} to 2 (as in \FS and \FPI), the considerably smaller the final optimality gap is. However, setting only one of them close to 2 only has a minor effect on optimality gap or convergence speed. 

\paragraph{Implementation.} Interestingly, we also have two implementation possibilities: 
\begin{itemize}[leftmargin=*]
\item decentralized: In this variant, all updates in \eqref{eq:gFL-1}-\eqref{eq:gFL-3} are computed and stored locally at each user while the server only acts as a bridge for synchronization: it receives $\zv_{t+1, i}$ from the $i$-th user and returns the (same) averaged model $\prox[]{H}(\zbs_{t+1})$ to all users. In a network where users  communicate directly with each other, we may then dispense the server and become fully decentralized.

\item centralized: In this variant, the $i$-th user only stores $\uv_{t,i}$ and is responsible for computing and communicating $\prox[\eta_t]{f_i}(\uv_{t, i})$ to the server. In return, the server performs the rest of \eqref{eq:gFL-1}-\eqref{eq:gFL-3} and then sends (different) $\uv_{t+1, i}$ back to the $i$-th user.
\end{itemize}

The two variants are of course mathematically equivalent in their basic form. However, as we see below, they lead to different acceleration techniques and we may prefer one over the other depending on the hardware and network constraints.

\begin{table}
\caption{A unifying framework \eqref{eq:gFL-1}-\eqref{eq:gFL-3} for \FL. Note that (a) \FA replaces the proximal update $\prox[\eta]{\fs}$ with a gradient update $\grad[\eta]{\fs,k}$; (b) ? indicates properties that remain to be studied; (c) ``sampling'' refers to selecting a subset of users while ``stochastic'' refers to updating with stochastic gradient.
}
\label{tab:FL}
\centering
  \begin{tabular}{llll|lcccc}
    \toprule
    Algorithm & $\alpha$ & $\beta$ & $\gamma$ & $\eta_t \!\equiv\! \eta$ & $\eta_t \!\to\! 0, \sum_t\eta_t\!=\!\infty$ & nonconvex & sampling & stochastic \\
    \midrule
    \FA       &    1     &    1    &    1     & \cref{eq:wFL} &  \cref{eq:wFL} & \checkmark & \checkmark & \checkmark\\
    \FP       &    1     &    1    &    1     & \cref{eq:regFL} &  \cref{eq:wFL} & \checkmark & \checkmark & \checkmark \\
    \FS       &    2     &    2    &    1     & \cref{eq:wFL}    &  --  & ? & ? & ?\\
    \midrule\midrule
    \FPI      &    2     &    2    &$\tfrac12$& \cref{eq:wFL}    &  --  & \checkmark & ? & ?\\
    \FR       &    2     &    1    &    1     & \cref{eq:regFL} &  \cref{eq:wFL} & ? & ? & ?\\
    \bottomrule
  \end{tabular}
\end{table}

\paragraph{Acceleration.} Let us abstract the grand scheme \eqref{eq:gFL-1}-\eqref{eq:gFL-3} as the map $\ubs_{t+1} = \Tm \ubs_t$, where $\Tm$ is nonexpansive if $\fs$ is convex and $\alpha_t, \beta_t \in [0,2], \gamma_t \in [0,1]$. Following \textcite{FuZB20}, we may then apply the Anderson type-II acceleration to furthur improve convergence. Let $U = [\ubs_{t-\tau}, \ldots, \ubs_t]$ be given along with $T = [\Tm\ubs_{t-\tau}, \ldots, \Tm\ubs_t]$. We solve the following simple least-squares problem: 
\begin{align}
\piv^* = \argmin_{\piv^\top \one = 1} ~ \{\| (U-T) \piv\|_2^2\} = \tfrac{G^\dag \one}{\one^\top G^\dag \one}, ~~\where~~ G = (U-T)^\top (U-T).
\end{align}
(For simplicity we do not require $\piv$ to be nonnegative.) Then, we update $\ubs_{t+1} = T\piv^*$. Clearly, when $\tau=0$, (trivially) $\pi^* = 1$ and we reduce to $\ubs_{t+1} = \Tm\ubs_t$. With a larger memory size $\tau$, we may significantly improve convergence. Importantly, all heavy lifting (computation and storage) is done at the server side and we do not increase communication at all. We note that the same acceleration can also be applied on each user for computing $\prox[\eta_t]{\fs}$, if the users can afford the memory cost.

We have already seen how different variants (\eg \FA, \FP, \FS, \FPI, \FR) compare to each other in \S\ref{sec:sp}. We performed further experiments to illustrate their behaviour under Anderson acceleration.
(See \Cref{sec:app-setup} for details on our experimental setup.) As can be observed in \Cref{fig:extrap},  Anderson acceleration helps \FL algorithms converge considerably faster, all without incurring any overhead. For the convex models (least squares and logistic regression), our implementation of Anderson-acceleration speeds up all of the algorithms, especially \FA, \FP and \FR. However, for the nonconvex CNN model, it is beneficial only for \FA, \FP and \FPI, while applying it to \FR and \FS makes them unstable. It is noteworthy that we already know \FP and \FPI are more stable than \FR and \FS, respectively, and hence it makes intuitive sense that acceleration improves the two more stable algorithms. Another important point we wish to point out is that our acceleration method does not affect the quality of the algorithms' final solutions, but rather it just accelerates their convergence.

\begin{figure}[t]
    \includegraphics[width=0.33\columnwidth]{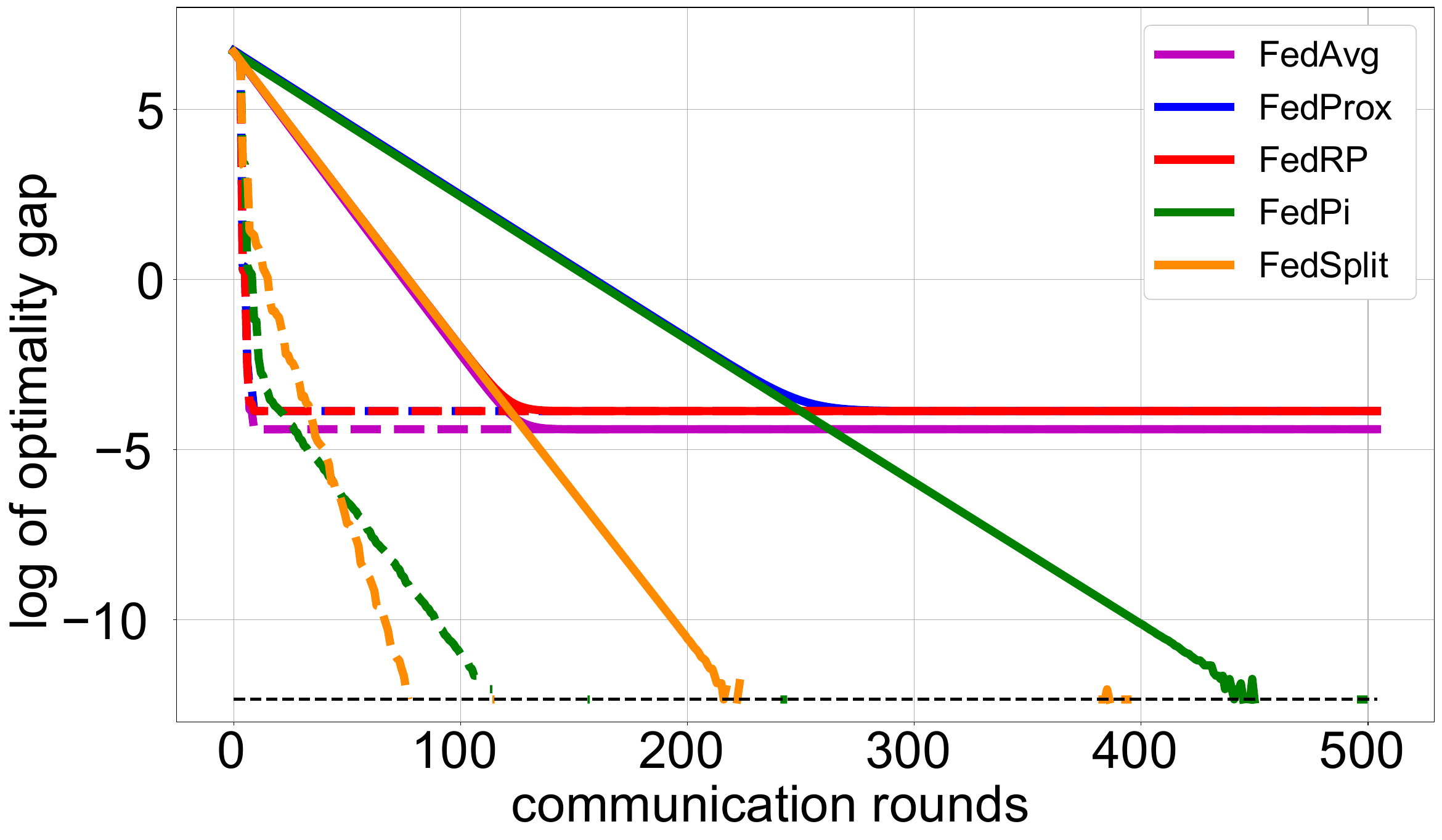}
    \includegraphics[width=0.32\columnwidth]{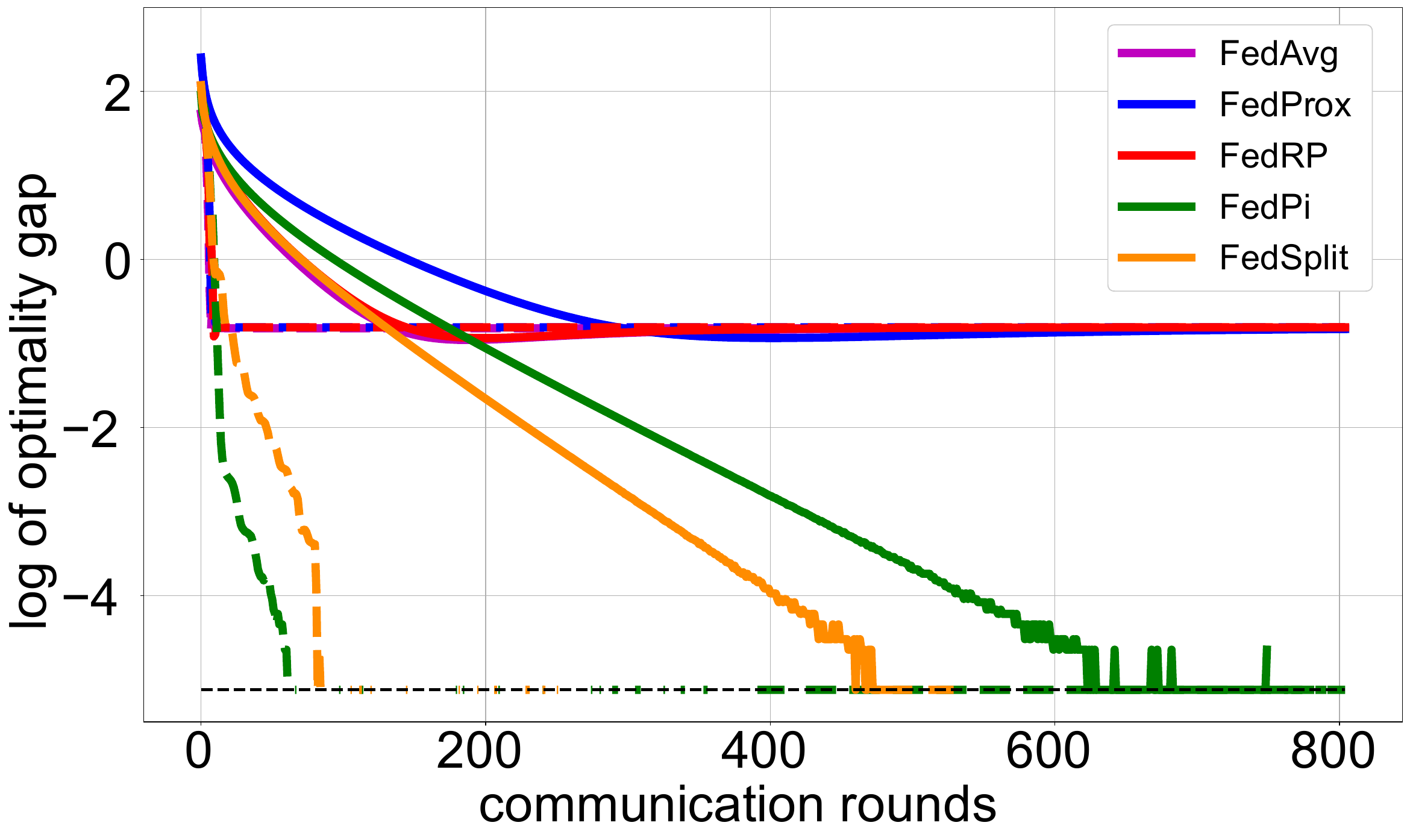}
    \includegraphics[width=0.32\columnwidth]{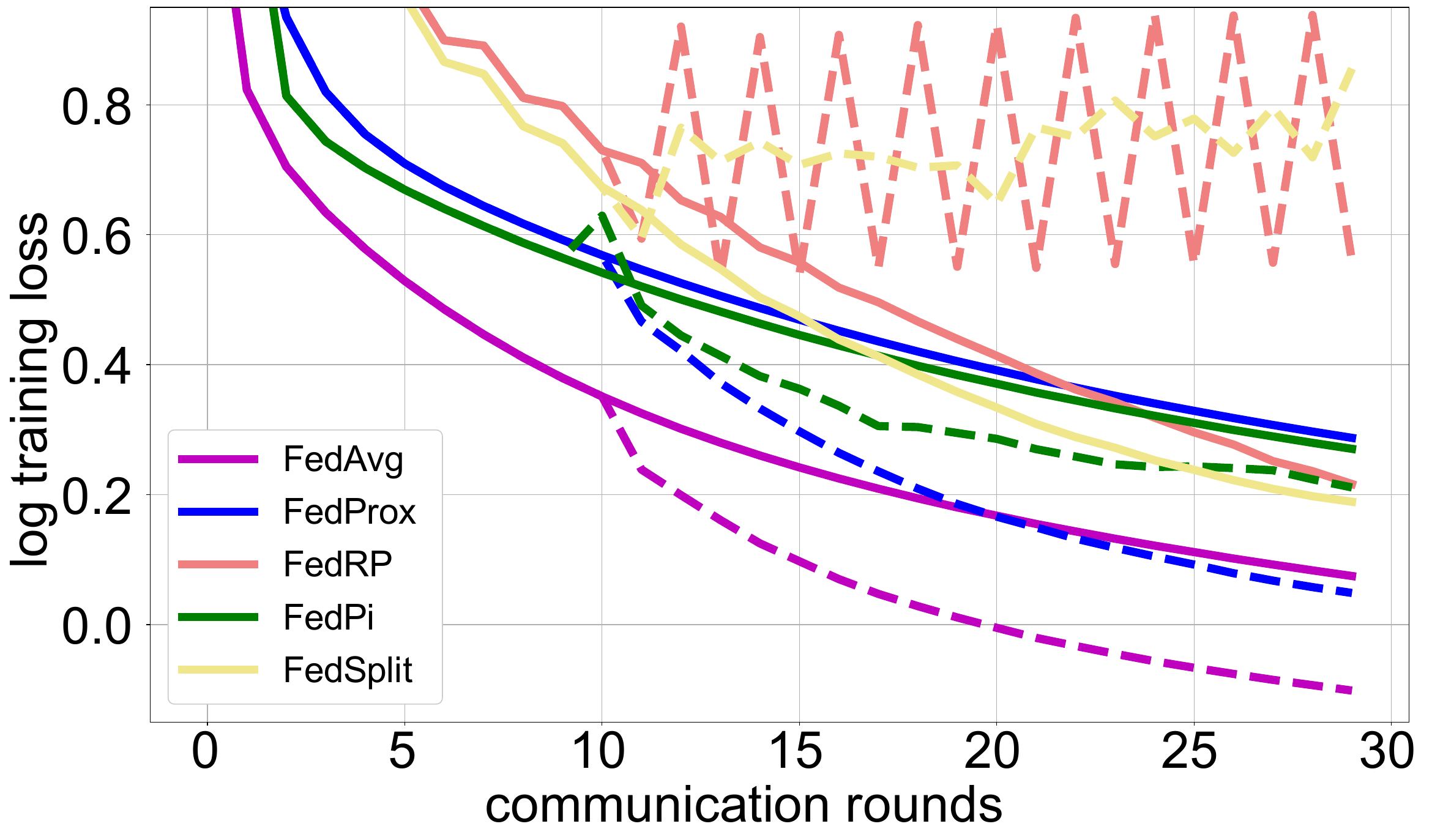}
    \caption{Effect of Anderson acceleration. Left: least squares with $\tau=2$; Middle: logistic regression $\tau=2$; Right: nonconvex CNN with $\tau=10$.  Dashed lines are the accelerated results. 
    }
    \label{fig:extrap}
\end{figure}

\section{Conclusions}
\label{sec:con}
We have connected \FL with the established theory of operator splitting, revealed new insights on existing algorithms and suggested new algorithmic variants and analysis. Our unified view makes it easy to  understand, compare, implement and accelerate different \FL algorithms in a streamlined and standardized fashion. Our experiments demonstrate some interesting differences in the convex and nonconvex settings, and in the early and late communication rounds. In the future we plan to study the effect of stochasticity and extend our analysis to nonconvex functions. While our initial experiments confirmed the potential of Anderson-acceleration for \FL, further work is required to formalize its effect in theory (such as \cite{ZhangOB20}) and to understand its relation with other momentum methods in \FL.

\section*{Acknowledgement}
We thank NSERC and WHJIL for funding support. 


\printbibliography[title={References}]

\clearpage

\appendix

\begin{center}
\Large
\bf
Appendix for \emph{An Operator Splitting View of Federated Learning}
\end{center}
\section{Experimental Setup}
\label{sec:app-setup}
In this section we provide more experimental details that are deferred from the main paper.

\subsection{Experimental setup: Least Squares and Logistic Regression}
\label{sec:lslr_exp_setup}
For simulating an instance of the aforementioned least squares and logistic regression problems, we follow the experimental setup of \textcite{pathak2020fedsplit}.

\paragraph{Least squares regression:} We consider a set of $\numuser$ users with local loss functions  $f_{\userind}(\wv):=\frac{1}{2}\|A_{\userind}\wv- \bv_{\userind}\|_2^2$, and the main goal is to solve the following minimization problem:
\begin{align*}
   \min_{\wv\in \RR^{\numdim}} ~ F(\wv) := \sum_{i=1}^{\numuser}f_{\userind}(\wv) = \frac{1}{2} \sum_{\userind=1}^m \|A_{\userind}\wv - \bv_{\userind}\|_2^2,
\end{align*}
where $\wv\in\RR^d$ is the optimization variable. For each user $\userind$, the response vector $\bv_\userind\in\RR^{n_\userind}$ is related to the  design matrix $A_\userind \in \RR^{n_\userind \times d}$ via the linear model
\begin{align*}
   \centering
    \bv_i= A_i \wv_\star + \epsilonv_i,
\end{align*}
where $\epsilonv_\userind \sim N(0,\sigma^2 I_{n_\userind})$ for some $\sigma > 0$ is the noise vector. The design matrix $A_\userind \in \RR^{n_\userind \times d}$ is generated by sampling its elements from a standard normal,  $A_\userind^{k,l} \sim N(0,1)$. We instantiated the problem with the following set of parameters:
\begin{align*}
   \centering
   m=25, ~~ d=100, ~~ n_\userind=5000, ~~ \sigma^2=0.25.
\end{align*}

\paragraph{Binary logistic regression:} There are $m$ users and the design matrices $A_\userind \in \RR^{n_\userind \times d}$ for $\userind = 1, \ldots, m$ are generated as described before. Each user $\userind$ has a label vector $\bv_\userind \in \{-1,1\}^{n_\userind}$. The conditional probability of observing $\bv_{\userind j}=1$ (the $j$-th label in $\bv_\userind$) is
\begin{align*}
   \centering
   \mathbf{P}\{\bv_{\userind j}=1\}=\frac{e^{\av_{\userind j}^\top \wv_0}}{1+e^{\av_{\userind j}^\top \wv_0}}, ~~~ j=1, \ldots, n_\userind,
\end{align*}
where $\av_{\userind j}$ is the $j$-th row of $A_\userind$. Also, $\wv_0\in\RR^d$ is fixed and sampled from $N(0,1)$. Having generated the design matrices $A_\userind$ and sampled the labels $\bv_\userind$ for all users, we find the maximum likelihood estimate of $\wv_0$ by solving the following convex program, which has a solution $\wv_*$:
\begin{align*}
   \min_{\wv\in \RR^{\numdim}} ~ F(\wv) := \sum_{\userind=1}^{m} f_\userind(\wv)=\sum_{\userind=1}^{m}\sum_{j=1}^{n_\userind}\log(1+e^{-b_{ij} \av_{ij}^\top \wv})+\frac{\| \wv \|_2^2}{2mn_\userind}.
\end{align*}
Following \textcite{pathak2020fedsplit}, we set
\begin{align*}
  \centering
  m=10, ~~ d=100, ~~ n_\userind=1000.
\end{align*}

\paragraph{Data heterogeneity measure:} We adopt the data heterogeneity measure  of  \textcite{KhaledMR20}, and quantify the amount of heterogeneity in users' data for  the least squares and logistic regression problems by
\begin{align*}
   H := \frac{1}{m} \sum_{\userind=1}^{\numuser}\parallel\nabla f_\userind(\wv_{*})\parallel_2^2,
\end{align*}
where $\wv_*$ is a minimizer of the original problem \eqref{eq:wFL}. When users' data is homogeneous, all the local functions $f_\userind$ have the same minimizer of $\wv_*$ and $H=0$. In general, the more heterogeneous the users' data is, the larger $H$ becomes.

\paragraph{Other parameters:} Throughout the experiments, we use local learning rate $\eta=10^{-5}$ for least squares, and $\eta = 10^{-2}$ for logistic regression, and run \FA  with $k=5$ for both least squares and logistic regression, unless otherwise specified. In \Cref{fig:extrap}, \FA is run with $k=2$.

\subsection{Experimental setup: MNIST datasets}
\label{sec:mnist_exp_setup}
 We consider a distributed setting with $20$ users. In order to create a non-i.i.d. dataset, we follow a similar procedure as in \textcite{McMahanMRHA17}: first we split the data from each class into several shards. Then, each user is randomly assigned a number of shards of data.  For example, in \Cref{fig:comp} to guarantees that no user receives data from more than $6$ classes, we split each class of MNIST into $12$ shards (i.e., a total of $120$ shards for the whole dataset), and each user is randomly assigned $6$ shards of data. By considering $20$ users, this procedure guarantees that no user receives data from more than $6$ classes and the data distribution of each user is different from each other. The local datasets are balanced--all users have the same amount of training samples.  The local data is split into train, validation, and test sets with percentage of $80$\%, $10$\%, and $10$\%, respectively. In this way, each user has $2400$ data points for training,  $300$ for test,  and $300$ for  validation. We use a simple 2-layer CNN model with ReLU activation, the detail of which can be found in \Cref{table:mnist_model}. To update the local models at each user using its local data, unless otherwise is stated, we apply gradient descent with $\eta=0.01$ for \FA, and gradient descent with $k=100$ and $\eta=0.01$ for proximal update in the splitting algorithms. 
\begin{table}[th]
\footnotesize	
\centering
\caption{MNIST model \label{table:mnist_model}}
\begin{tabular}{lcccc} \toprule
          Layer &  Output Shape &  $\#$ of Trainable Parameters & Activation & Hyper-parameters  \\\midrule
           Input & $(1, 28, 28)$ & $0$ &  &  \\
           Conv2d & $(10, 24, 24)$ & $260$ & ReLU & kernel size =$5$; strides=$(1, 1)$ \\
           MaxPool2d & $(10, 12, 12)$ & $0$ &  & pool size=$(2, 2)$ \\
           Conv2d & $(20, 8, 8)$ & $5,\!020$ & ReLU & kernel size =$5$; strides=$(1, 1)$ \\
           MaxPool2d & $(20, 4, 4)$ & $0$ &  & pool size=$(2, 2)$ \\
           Flatten & $320$ & $0$ & & \\
            Dense &  $20$ & $6,\!420$ & ReLU & \\
            Dense &  $10$ & $210$ & softmax & \\ \midrule
          Total & & $11,\!910$  & & \\ \bottomrule
\end{tabular}
\end{table}
\subsection{Experimental setup: CIFAR-10 dataset}
\label{sec:cifar_exp_setup}
We consider a distributed setting with $10$ users, and use a sub-sampled CIFAR-10 dataset, which is $20\%$ of the training set. In order to create a non-i.i.d.  dataset, we follow a similar procedure as in \textcite{McMahanMRHA17}: first we sort all data points according to their classes. Then, they are split into $100$ shards, and each user is randomly assigned $10$ shards of data. 
The local data is split into train, validation, and test sets with percentage of $80$\%, $10$\%, and $10$\%, respectively. In this way, each user has $800$ data points for training,  $100$ for test,  and $100$ for  validation. We use a  2-layer CNN model on the dataset, , the detail of which can be found in \Cref{table:cifar-10_model}. To update the local models at each user using its local data, we apply stochastic gradient descent (SGD) with local batch size $B=20$ and local learning rate $\eta=0.1$. All users participate in each communication round.

\begin{table}[th]
\footnotesize	
\centering
\caption{CIFAR-10 model \label{table:cifar-10_model}}
\begin{tabular}{lcccc} \toprule
          Layer &  Output Shape &  $\#$ of Trainable Params & Activation & Hyper-parameters  \\\midrule
           Input & $(3, 32, 32)$ & $0$ &  &  \\
           Conv2d & $(64, 28, 28)$ & $4,\!864$ & ReLU & kernel =$5$; strides=$(1, 1)$ \\
           MaxPool2d & $(64, 14, 14)$ & $0$ &  & pool size=$(2, 2)$ \\
           LocalResponseNorm & $(64, 14, 14)$ & $0$ &  & size=$2$ \\
           Conv2d & $(64, 10, 10)$ & $102,\!464$ & ReLU & kernel =$5$; strides=$(1, 1)$ \\
           LocalResponseNorm & $(64, 10, 10)$ & $0$ &  & size=$2$ \\
           MaxPool2d & $(64, 5, 5)$ & $0$ &  & pool size=$(2, 2)$ \\
           Flatten & $1,\!600$ & $0$ & & \\
           Dense &  $384$ & $614,\!784$ & ReLU & \\
           Dense &  $192$ & $73,\!920$ & ReLU & \\
           Dense &  $10$ & $1,\!930$ & softmax & \\ \midrule
          Total & & $797,\!962$  & & \\ \bottomrule
\end{tabular}
\end{table}

\newpage
\section{Additional experimental results}

In this section we provide more experimental results that are deferred from the main paper.

\subsection{FedRP experiments to supplement  \Cref{fig:eta_effect_fedprox}}
In \Cref{fig:fedrp_eta}, we show the effect of step size $\eta$ on the convergence of \FR on least squares and logistic regression. 
We run \FR with both fixed and diminishing step sizes. In the experiments with diminishing step sizes, we  set the initial value of $\eta$ (i.e. $\eta_0$) to larger values--compared to the constant $\eta$ values--to ensure that $\eta$ does not get very small after the first few rounds. From the figure, one can see that \FR with a fixed learning rate converges faster in early stage but to a sub-optimal solution. In contrast, \FR with diminishing $\eta$ converges slower early, but to better final solutions. It is interesting to note that the convergence behaviour of \FR is similar to that of \FP and when the conditions of \Cref{thm:fp} are satisfied, \FR converges to a correct solution of the original problem \eqref{eq:wFL}, \eg see the results for $\eta_t \propto 1/t$ which satisfies both conditions in \Cref{thm:fp}.
Moreover, one can see that ergodic averaging hardly affects the convergence of \FR in the following results.  
\begin{figure}[H]
    \centering
    \includegraphics[width=0.49\columnwidth]{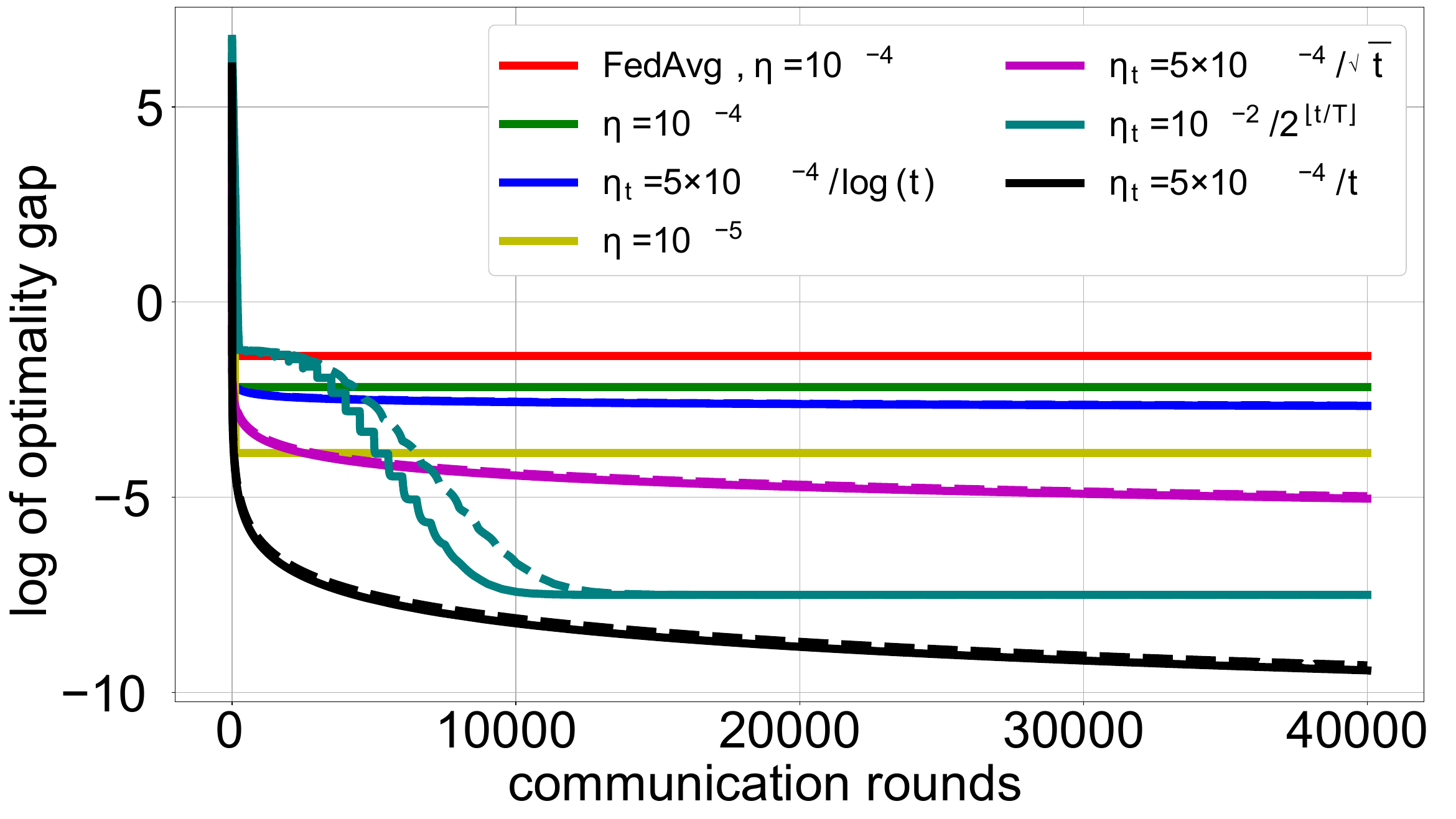}
    \includegraphics[width=0.49\columnwidth]{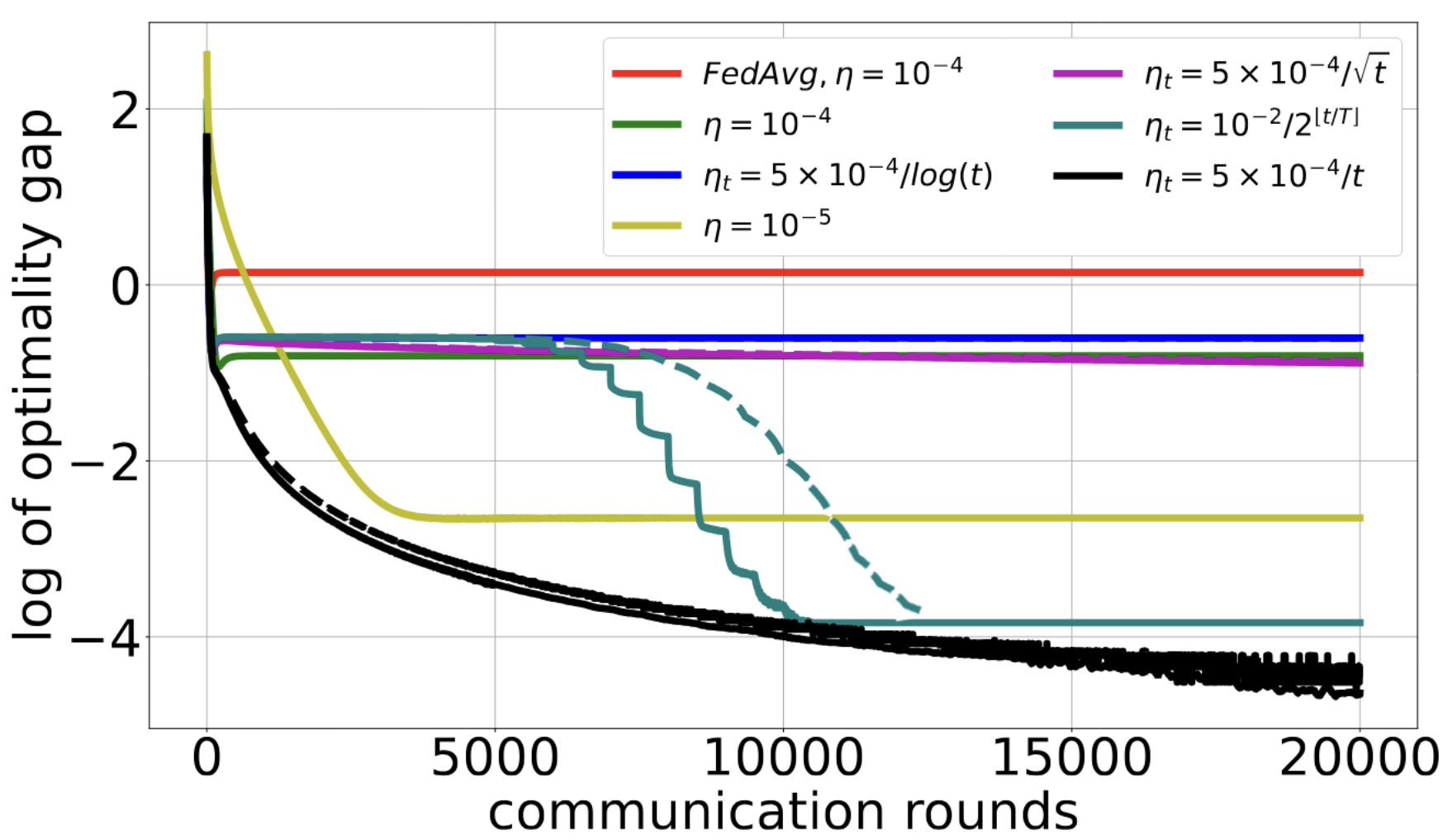}
    \caption{Effect of step size $\eta$ and averaging on \FR. Left: least squares; Right: logistic regression. The dashed and solid lines with the same color show the results obtained with and without the ergodic averaging step in \Cref{thm:fp}, respectively. For exponentially decaying $\eta_t$, we use period $T$ equal to $500$ for both least squares and logistic regression. }
    \label{fig:fedrp_eta}
\end{figure}

\clearpage
\newpage

\subsection{Effect of changing \texorpdfstring{$k$}{k} or \texorpdfstring{$\eta$}{eta} for \FA}
In \Cref{fig:trade_off}, we observe a trade-off between faster (early) convergence and better final solution  when varying local epochs $k$ or learning rate $\eta$, in least squares and logistic regression experiments. When $\eta$ is fixed, larger local epochs $\numstep$ leads to faster convergence (in terms of communication rounds) during the early stages, at the cost of a worse final solution. When local epochs $\numstep$ is fixed, same trade-off holds true for increasing $\eta$.

\begin{figure}[H]
    \centering
    \includegraphics[height=4.2cm,width=0.45\columnwidth]{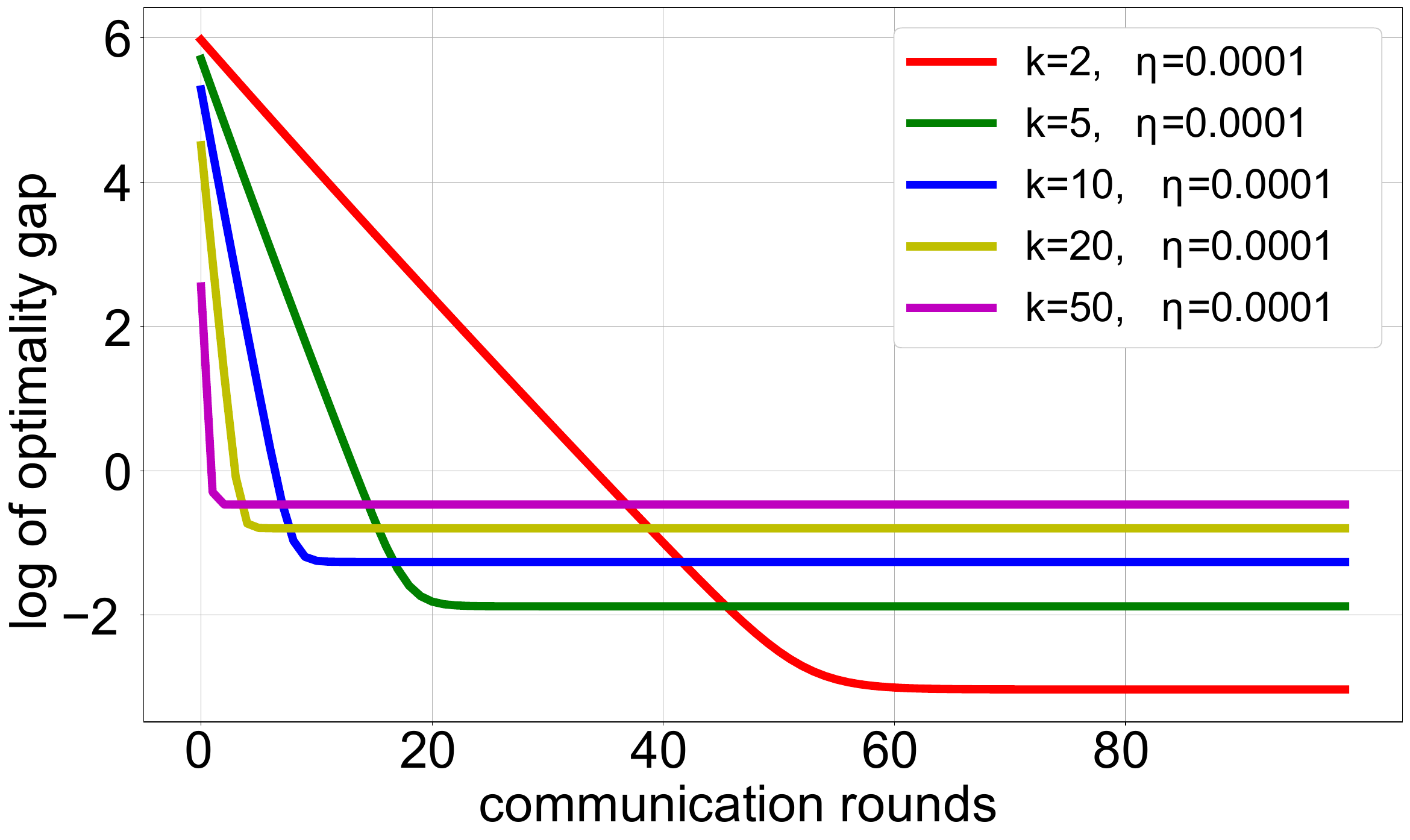}
    \includegraphics[height=4.2cm,width=0.45\columnwidth]{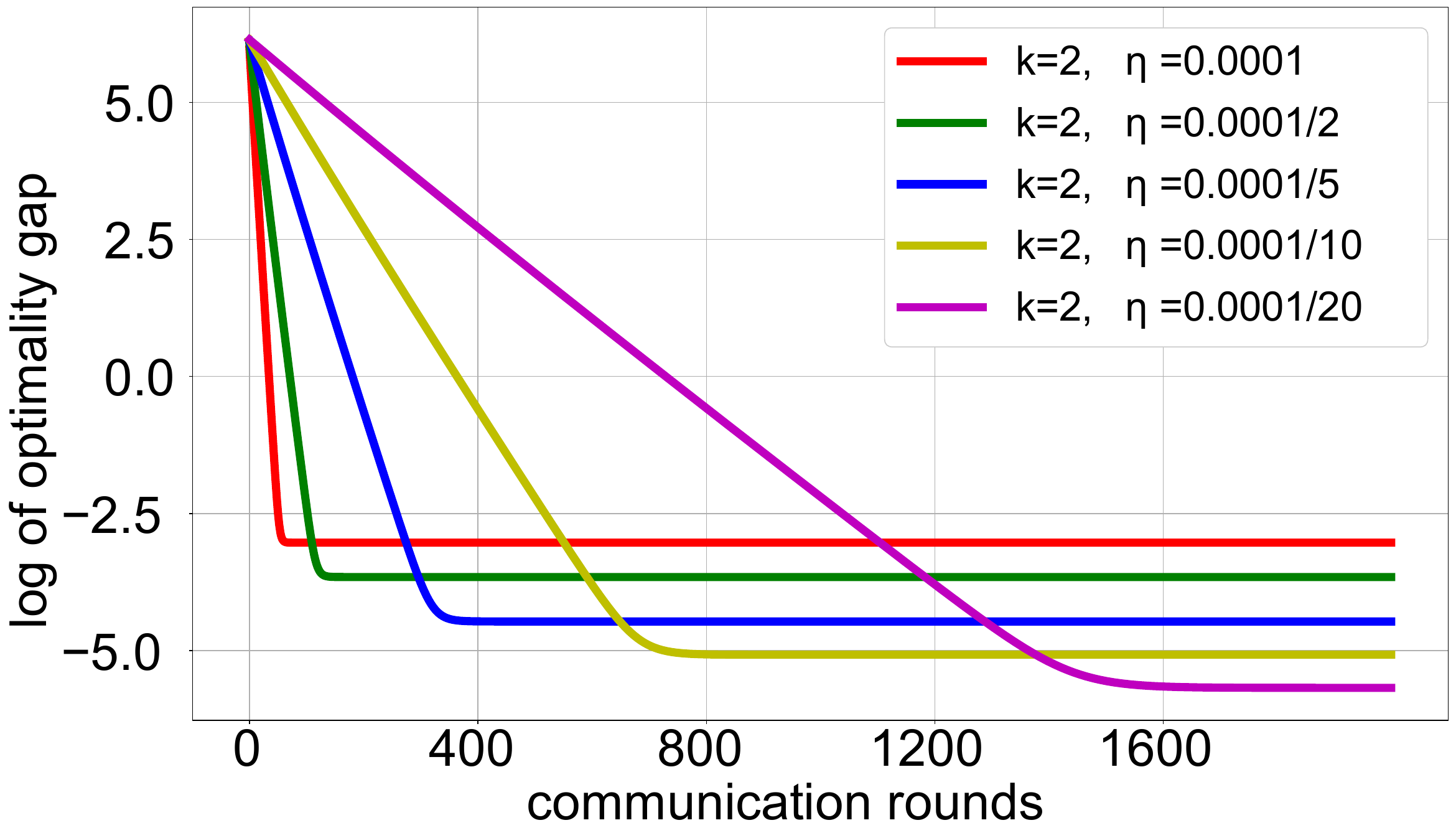}
    \includegraphics[height=4.2cm,width=0.45\columnwidth]{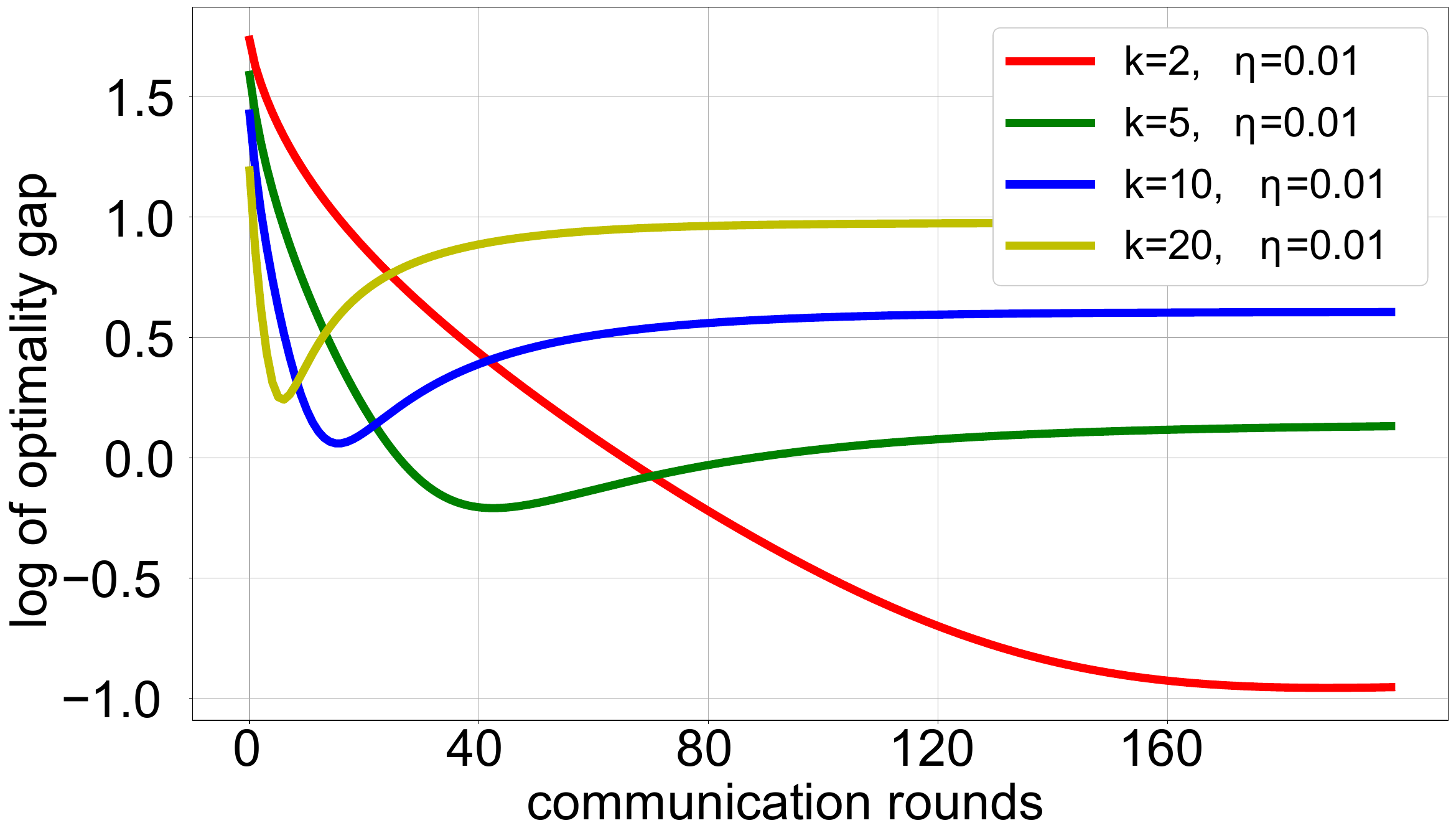}
    \includegraphics[height=4.2cm,width=0.45\columnwidth]{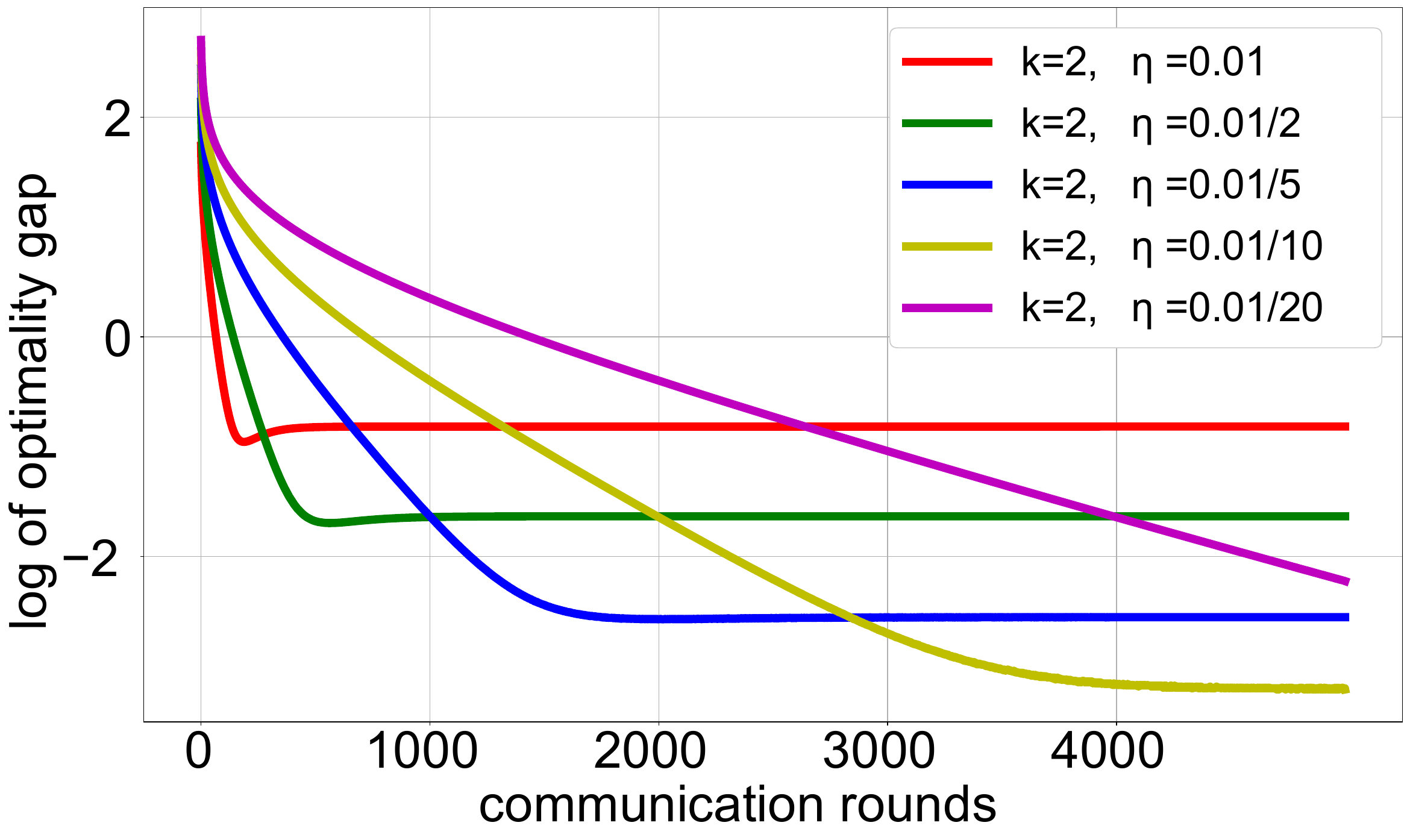}
    \caption{Convergence of \FA changing $k$ or $\eta$. Top: least squares; Bottom: logistic regression.}
    \label{fig:trade_off}
\end{figure}

In \Cref{fig:cifar10}, contrary to what we observe in the previous convex experiments where fixed learning rate with less local epochs leads to a better final solution at the expense of slower convergence, we do not see similar trade-offs in this nonconvex setting, (at least) given the computation we can afford. This is probably because the solution we get is still in early stage due to the combined complexity of CNN and CIFAR-10 dataset, compared to the relatively simple least squares and logistic regression we had above. 
\begin{figure}[H]
\centering
\includegraphics[height=5.6cm,width=0.5\columnwidth]{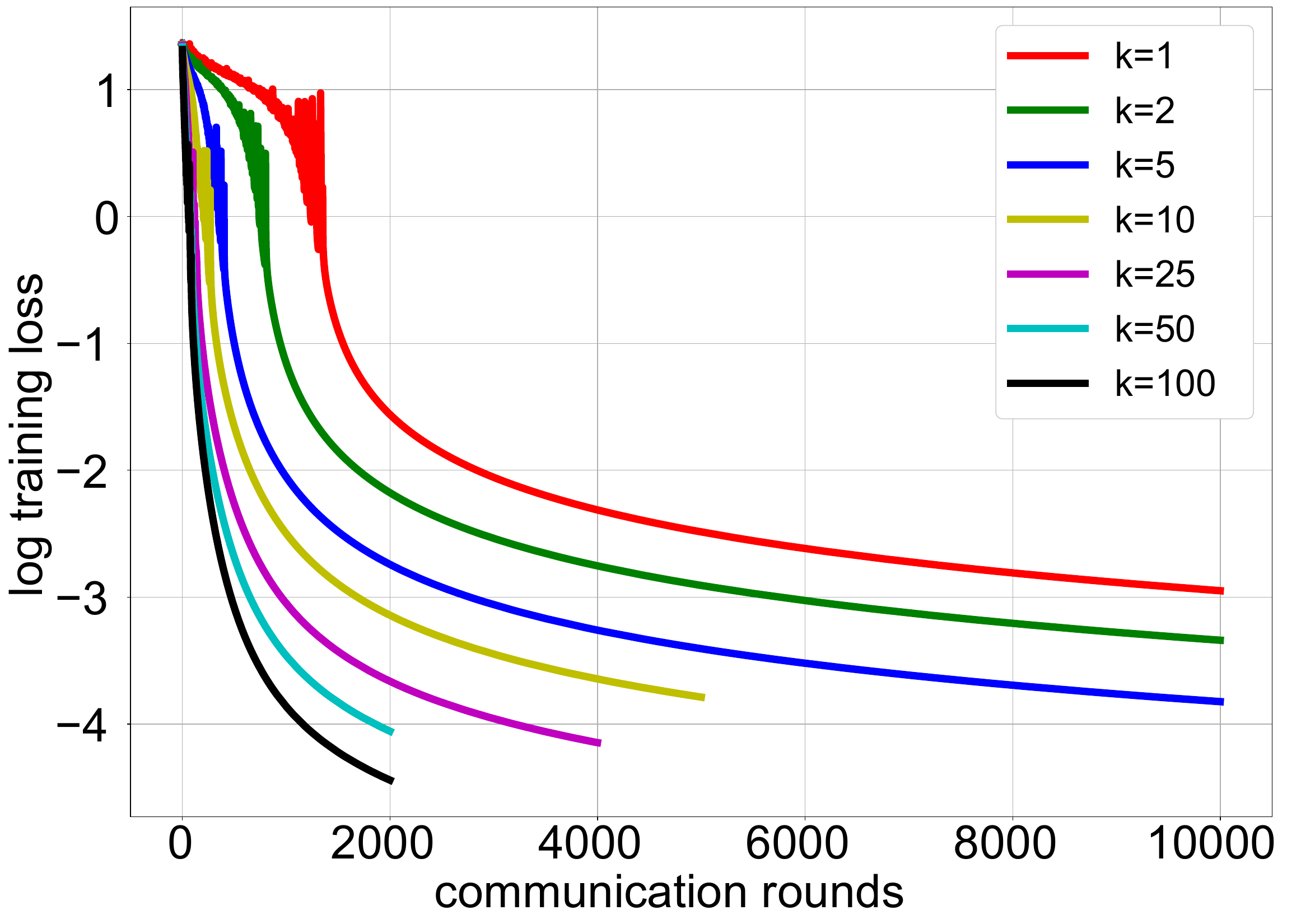}
\caption{Running \FA on CIFAR-10 with different number of local epochs $k$. The number of users is $10$, local learning rate $\eta=0.1$. Data points are sampled and distributed to users non-uniformly.} 
\label{fig:cifar10}
\end{figure}


\subsection{Effect of user sampling}
The grand scheme introduced in \eqref{eq:gFL-1}-\eqref{eq:gFL-3} assumes that all users participate in each communication round--full user participation; however, this might not be always the case.  So, we modify the grand scheme  to investigate the effect of user participation. The modified version, which is based on the centralized implementation of \FL algorithms (explained in \cref{sec:acc}), is given in Algorithm \ref{alg:splitting_usersampling}. The simulation results for least squares and logistic regression are shown in \Cref{fig:splitting_usersampling}.  From the results it can be observed that, with partial user participation, none of the aforementioned algorithms converge to the correct solution, especially \FA with $k=1$. Also, \FS becomes unstable when user participation is low ($p=0.5$), while other algorithms \eg \FPI remain stable.

\begin{algorithm}[H]
\DontPrintSemicolon
\footnotesize
\SetAlgoLined
\KwIn{$p$ (user participation probability), $\alpha_t, \beta_t, \gamma_t, \eta_t$}
\KwOut{$\bar\wbs$ (final global model)}
\SetKwProg{Server}{Server}{:}{}
\SetKwProg{Client}{ClientUpdate}{:}{}

 initialize $\ubs_0$ and $\zbs_0$ for all users

 \Server{(p)}{
 \For{each round $t = 1, 2, \ldots$}{
  $S_t \gets$ sample the set of present users, where each user participates with probability $p$
  
  \For{each client $\userind$ \textbf{in parallel}}{
  \eIf{$\userind \in S_t$}{
   \textbf{ClientUpdate}($\ubs_{i,t}$, $i$)
   }{
   $\zbs_{i, t+1} \gets \zbs_{i, t}$
  }
  }
  
 \vspace{1mm}
 $\bar\wbs \gets \prox[]{H}(\zbs_{t+1})$ \tcp*{over $\zbs_{t+1}$ of present clients}

 \vspace{1mm}
 $\wbs_{t+1} = (1-\beta_t)\zbs_{t+1} + \beta_t \bar\wbs$ 
 
 $\ubs_{t+1} = (1-\gamma_t) \ubs_{t} + \gamma_t \wbs_{t+1}$

  }
 }

\Client{($\ubs_{i,t}$, $i$)}{
compute $\prox[\eta_t]{f_i}(\ubs_{i,t})$ \tcp*{using  standard gradient descent}

$\zbs_{i, t+1} \gets (1-\alpha_t)\ubs_{i, t} + \alpha_t \prox[\eta_t]{f_i}(\ubs_{i,t})
$
}  
 \caption{Federated learning with user sampling.}
 \label{alg:splitting_usersampling}
\end{algorithm}

\begin{figure}[H]
    \centering
    \includegraphics[height=4.6cm,width=0.49\columnwidth]{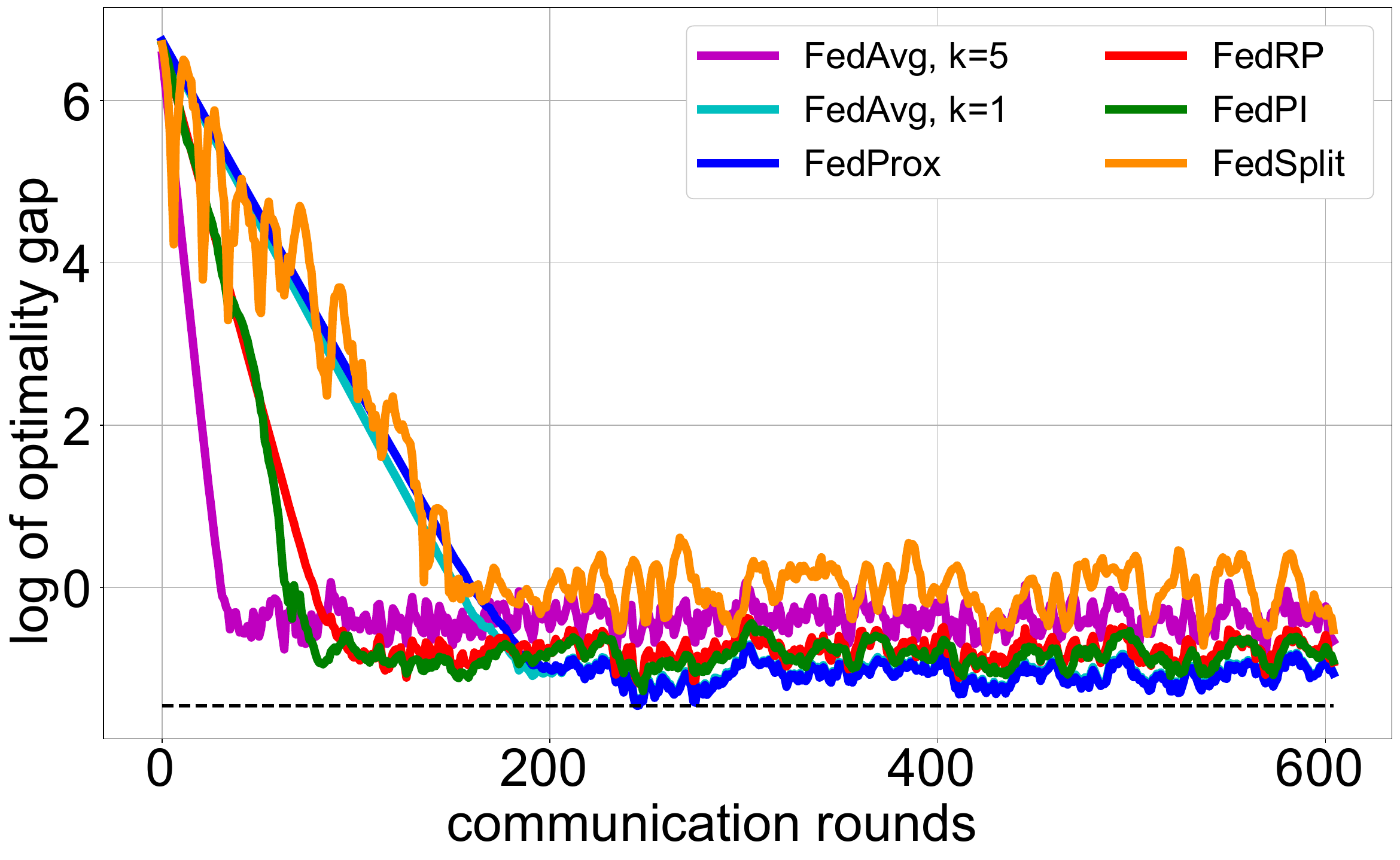}
    \includegraphics[height=4.6cm,width=0.49\columnwidth]{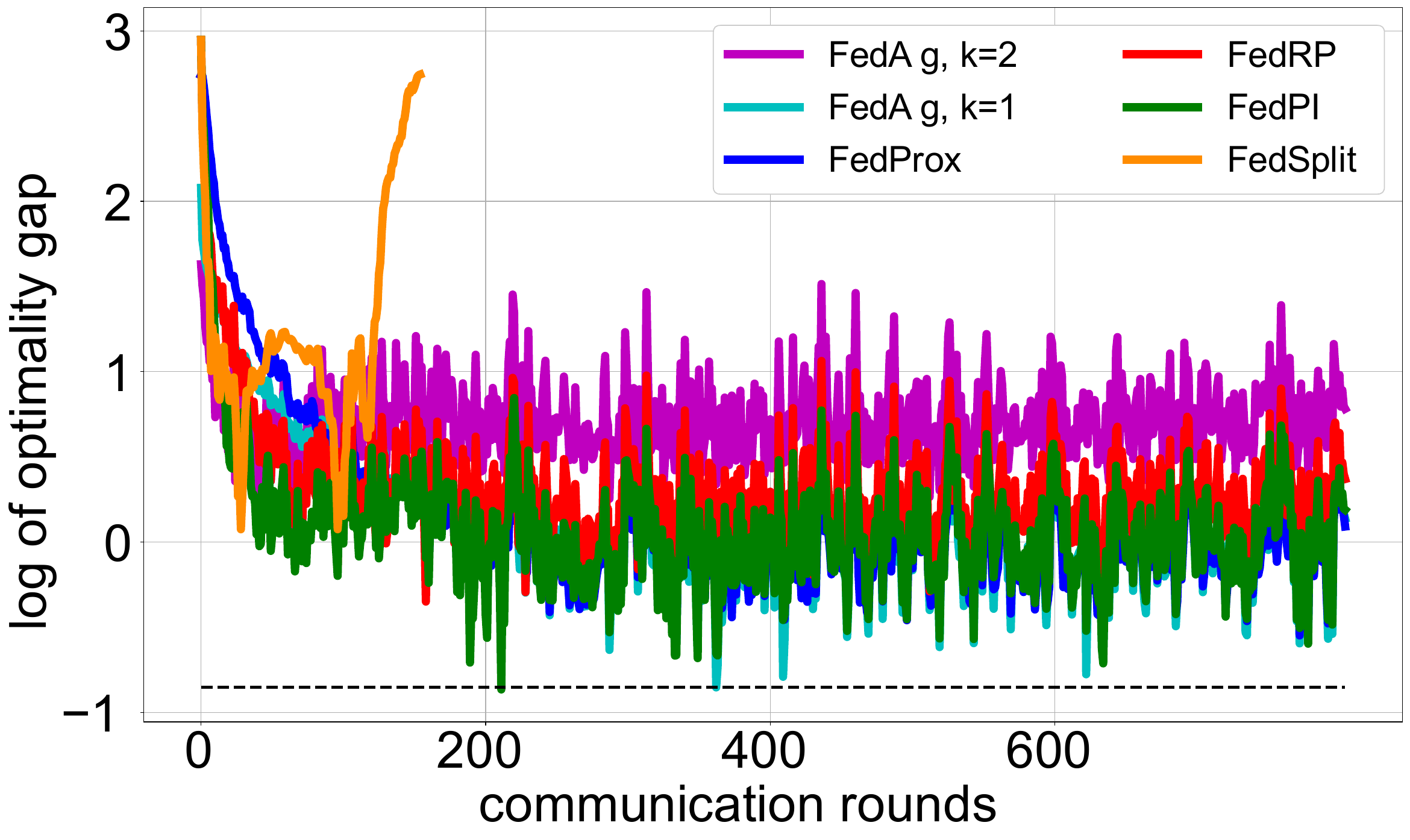}
    \includegraphics[height=4.6cm,width=0.49\columnwidth]{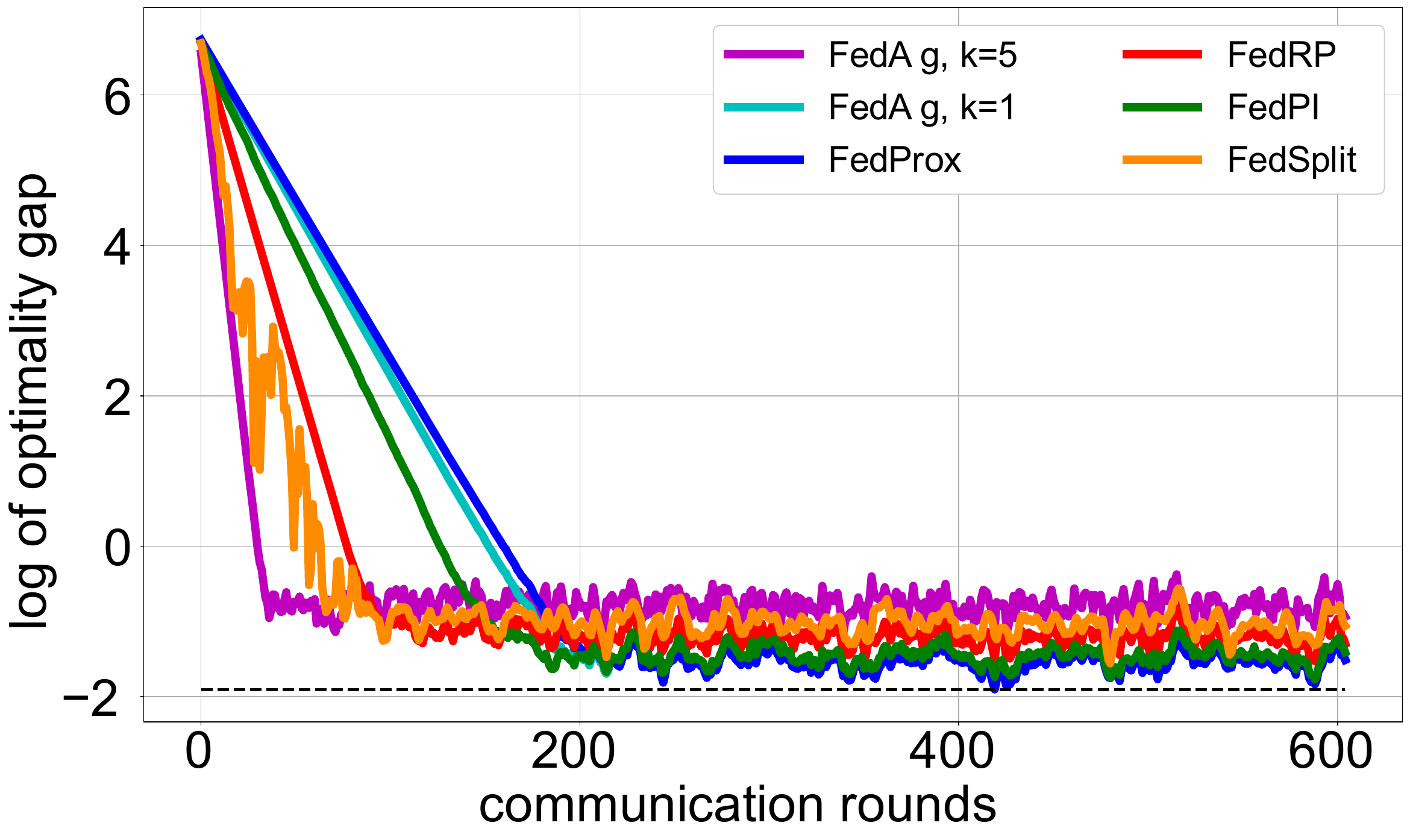}
    \includegraphics[height=4.6cm,width=0.49\columnwidth]{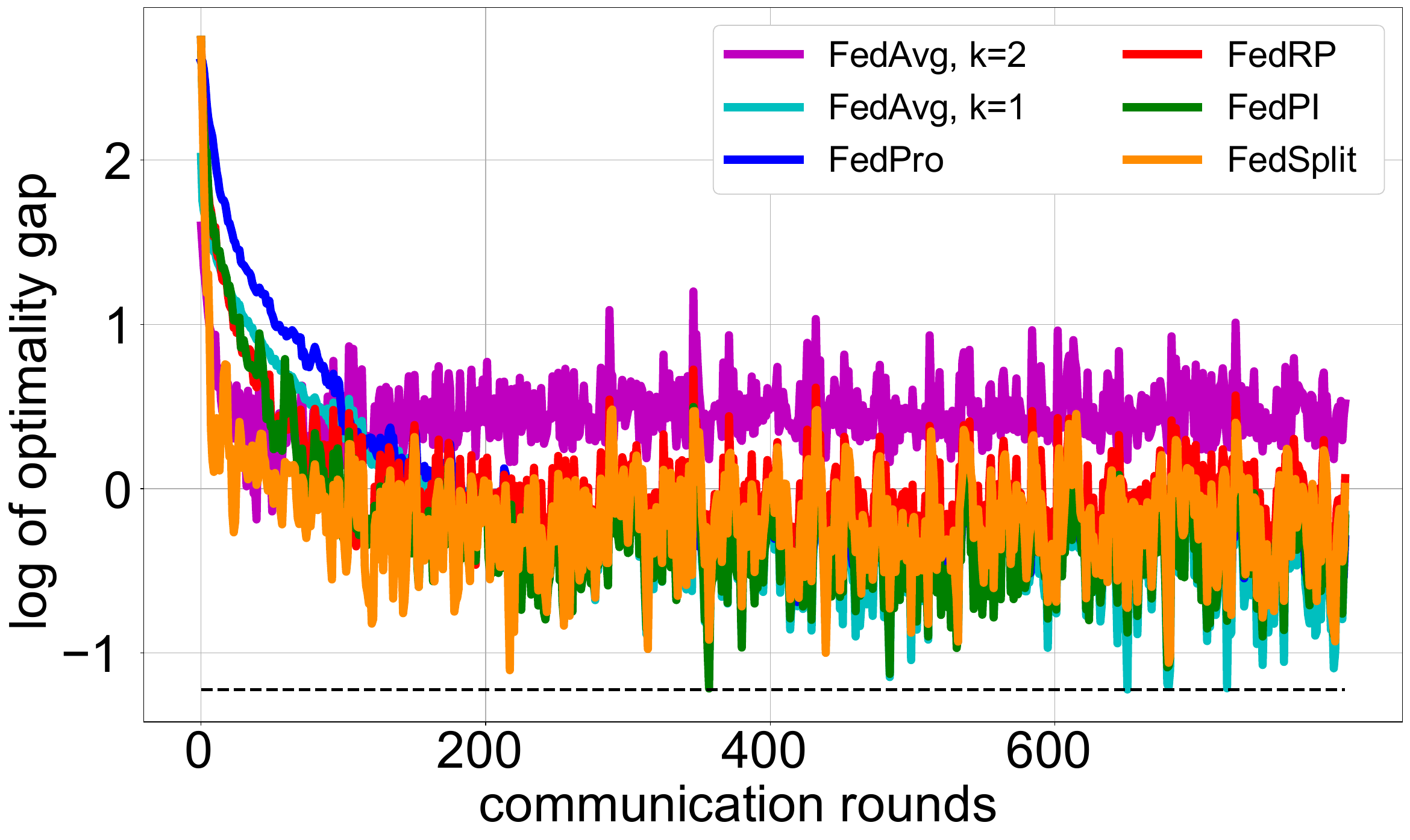}
    \caption{Effect of user sampling. Top-Left: least squares with $p=0.5$ (participation probability); Top-Right: logistic regression with $p=0.5$; Bottom-Left: least squares with $p=0.7$; Bottom-Right: logistic regression with $p=0.7$.}
    
    \label{fig:splitting_usersampling}
\end{figure}

\newpage
\subsection{Effect of using mini-batch gradient descent locally}
\Cref{fig:splitting_minibatch} shows the results when each user uses mini-batch gradient descent to update its model for different splitting algorithms. As in the grand scheme \eqref{eq:gFL-1}-\eqref{eq:gFL-3}, we assume full user participation. Algorithm \ref{alg:splitting_SGD} summarizes the implementation details. As can be observed in the convex setting experiments, with mini-batch gradient descent used locally, none of the aforementioned algorithms converge to the correct solution. Furthermore, convergence speed of the algorithms decreases compared to when users use standard gradient descent. Also, \FA, even with $k=1$, does not outperform the other splitting algorithms, and \FS and \FPI converge to better solutions. On the other hand, in the nonconvex setting, \FA with $k=100$ achieves better solutions compared to the other algorithms.

\begin{algorithm}[H]

\footnotesize
\SetAlgoLined
\DontPrintSemicolon
\footnotesize
\SetAlgoLined
\KwIn{$B$ (local minibatch size), $\alpha_t, \beta_t, \gamma_t, \eta_t$}
\KwOut{$\bar\wbs$ (final global model)}
\SetKwProg{Server}{Server}{:}{}
\SetKwProg{Client}{ClientUpdate}{:}{}

initialize $\ubs_0$ for each user $i$

\Server{}{ 
 \For{each round $t = 1, 2, \ldots$}{
  
  \For{each client $\userind$ \textbf{in parallel}}{
  \textbf{ClientUpdate}($\ubs_{i,t}, i$)
   
  }
  
 \vspace{1mm}
 $\bar\wbs \gets \prox[]{H}(\zbs_{t+1})$

 \vspace{1mm}
 $\wbs_{t+1} = (1-\beta_t)\zbs_{t+1} + \beta_t \bar\wbs$ 
 
 $\ubs_{t+1} = (1-\gamma_t) \ubs_{t} + \gamma_t \wbs_{t+1}$

 }
 }

\Client{($\ubs_{i,t}, i$)}{
  compute $\prox[\eta_t]{f_i}(\ubs_{i,t})$ \tcp*{using minibatch gradient descent with batch size $B$} 
  
  $\zbs_{i,t+1} \gets (1-\alpha_t)\ubs_{i,t} + \alpha_t \prox[\eta_t]{f_i}(\ubs_{i,t})
  $  
 }
 \caption{Mini-batch Gradient Descent for local computations.}
 \label{alg:splitting_SGD}
\end{algorithm}

\begin{figure}[H]
\centering
    \includegraphics[width=0.49\columnwidth]{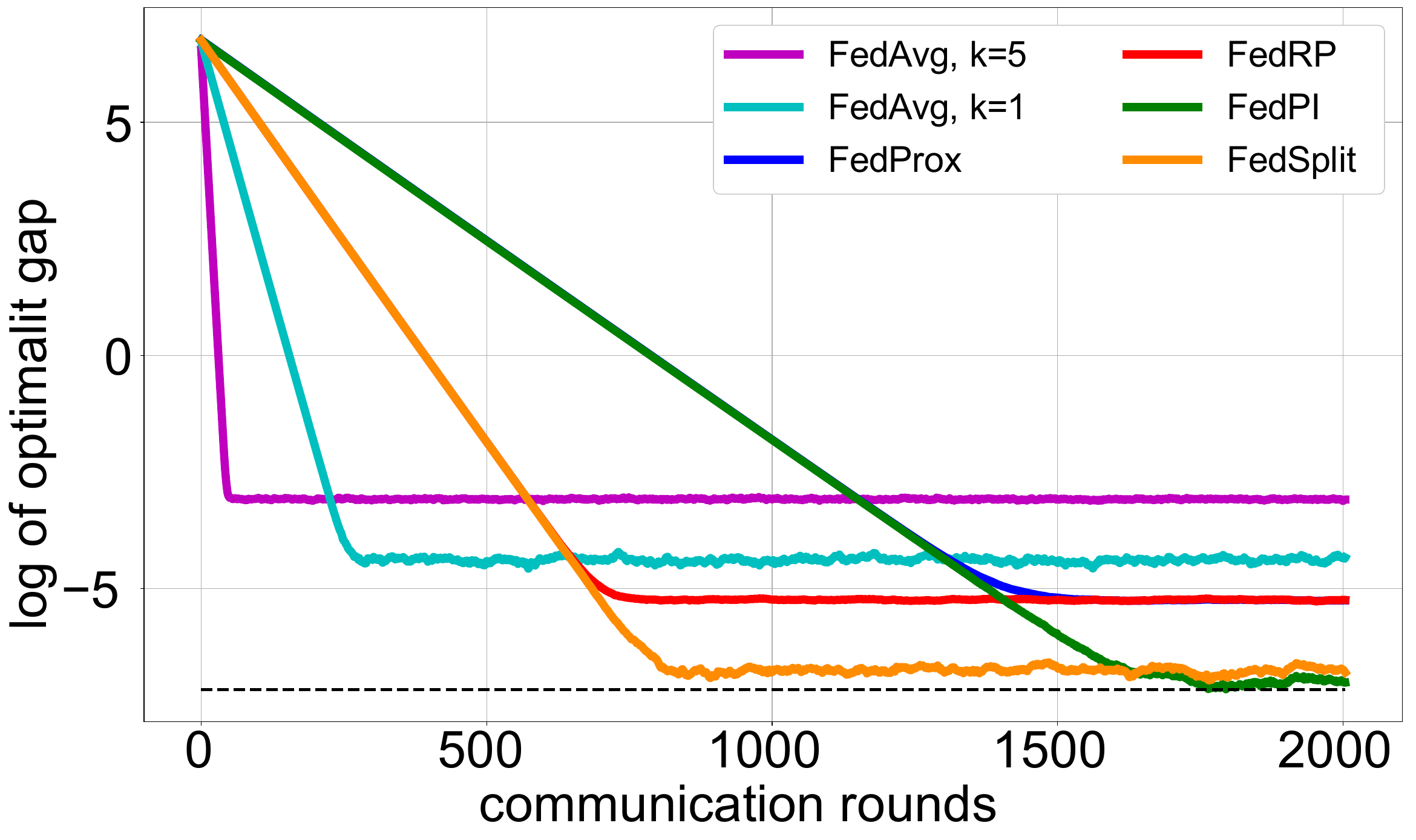}
    \includegraphics[width=0.49\columnwidth]{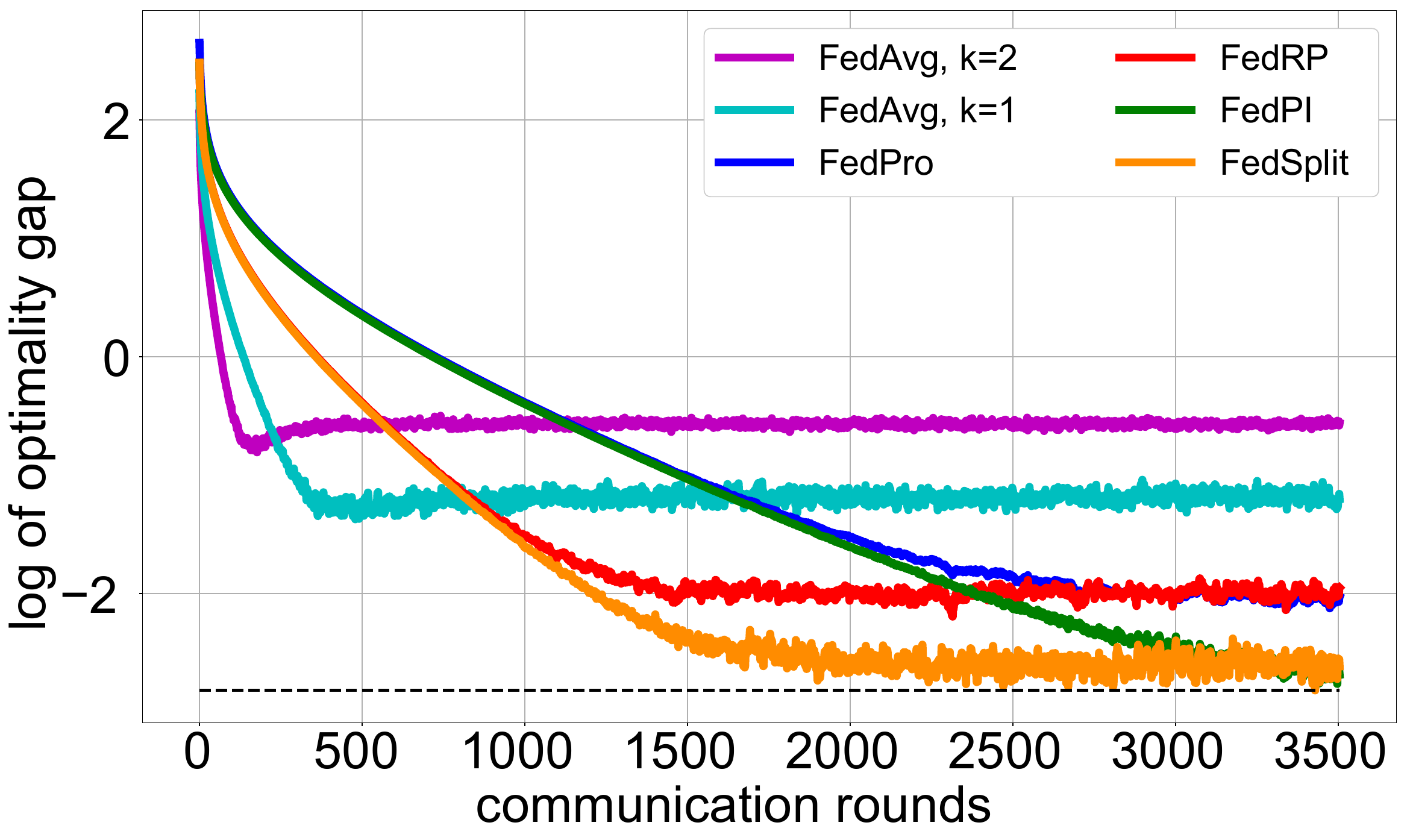}
    \includegraphics[width=0.49\columnwidth]{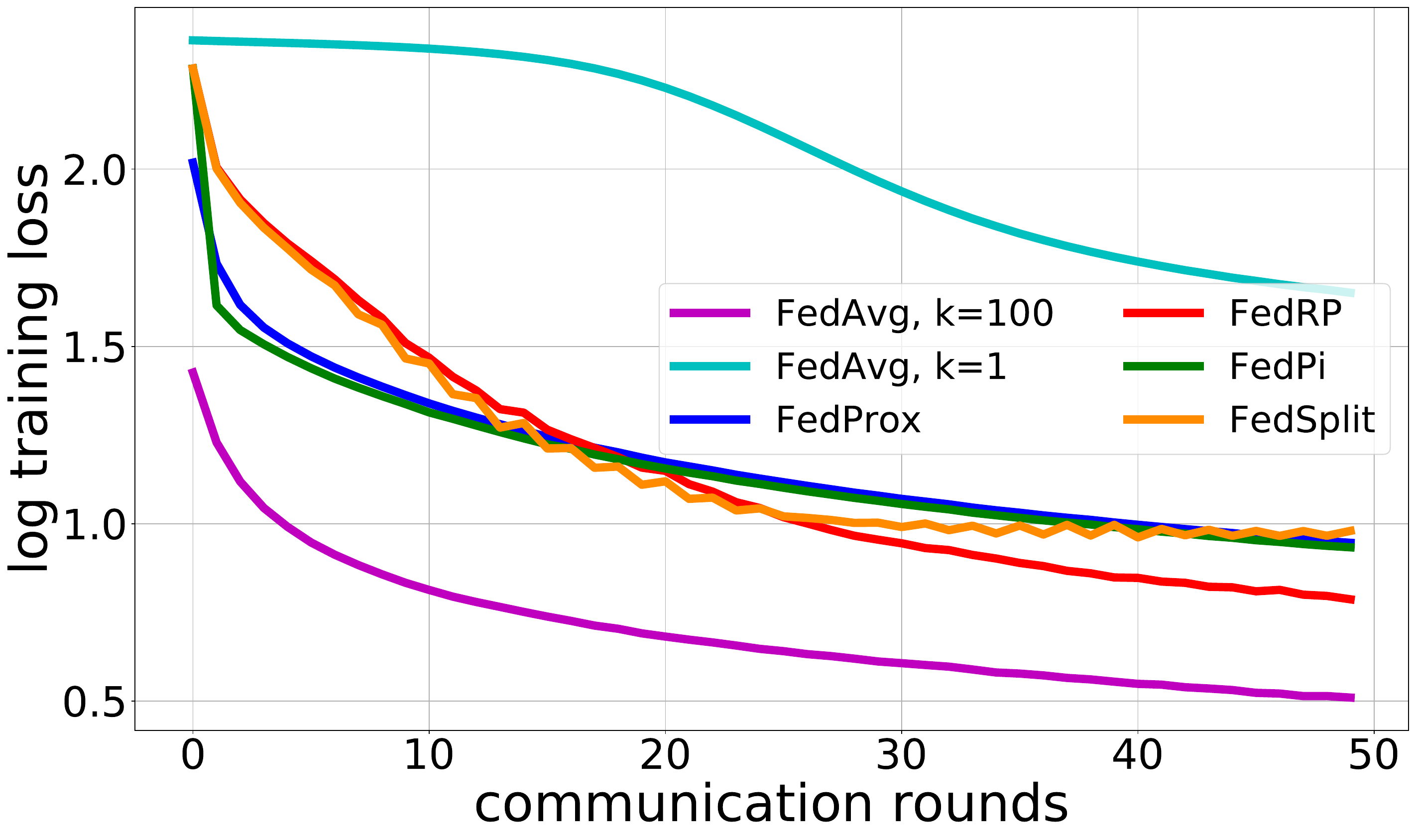}
    \caption{The effect of using minibatch gradient descent on the splitting algorithms. Local datasets are divided into $5$ minibatches.  Top-Left: least squares with  $B = 1000$. Top-Right: logistic regression with $B = 200$. Bottom: nonconvex CNN on MNIST with $B = 500$.}
    
    \label{fig:splitting_minibatch}
\end{figure}

\subsection{Results with error bars for \Cref{fig:extrap}, Right}
\paragraph{Implementation details of Anderson acceleration for nonconvex CNN:} In practice, we found that applying the acceleration algorithm to  $\prox[]{H}(\zbs_{t+1})$ in \eqref{eq:gFL-2} rather than $\ubs_{t+1}$ in \eqref{eq:gFL-3} leads to a more stable algorithm. So, we  applied acceleration to $\prox[]{H}(\zbs_{t+1})$, and used the result to find $\ubs_{t+1}$ and $\wbs_{t+1}$ for the nonconvex model.   

\Cref{fig:extrap_error} shows the effect of Anderson acceleration on the nonconvex CNN with memory size $\tau = 10$. \FR and \FS are less stable and Anderson acceleration actually made them worse while \FA, \FP and \FPI converge much faster with our implementation of Anderson acceleration. The deviation among 4 different runs is shown as the shaded area.

\begin{figure}[H]
\centering
    \includegraphics[width=0.32\columnwidth]{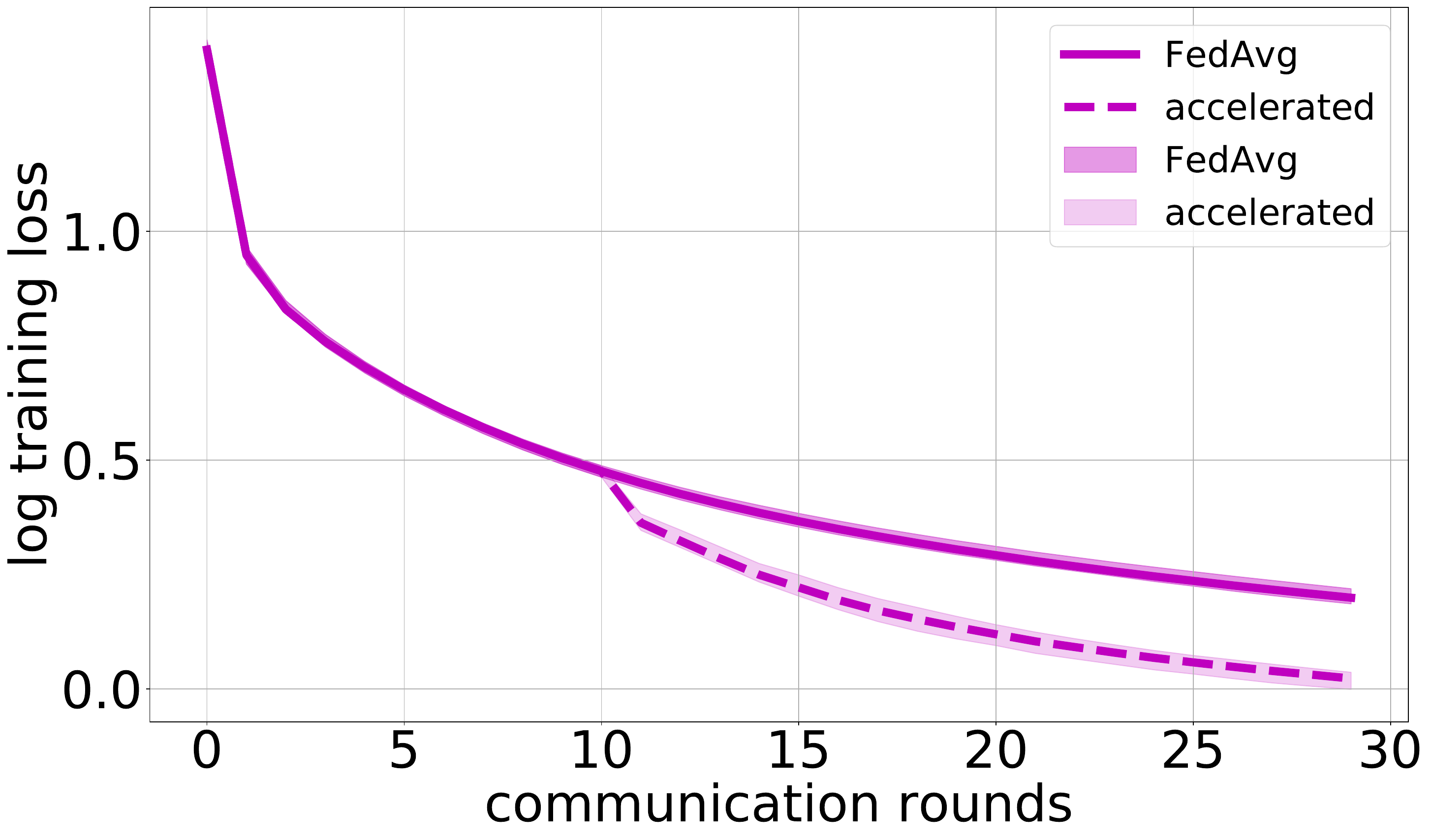}
    \includegraphics[width=0.32\columnwidth]{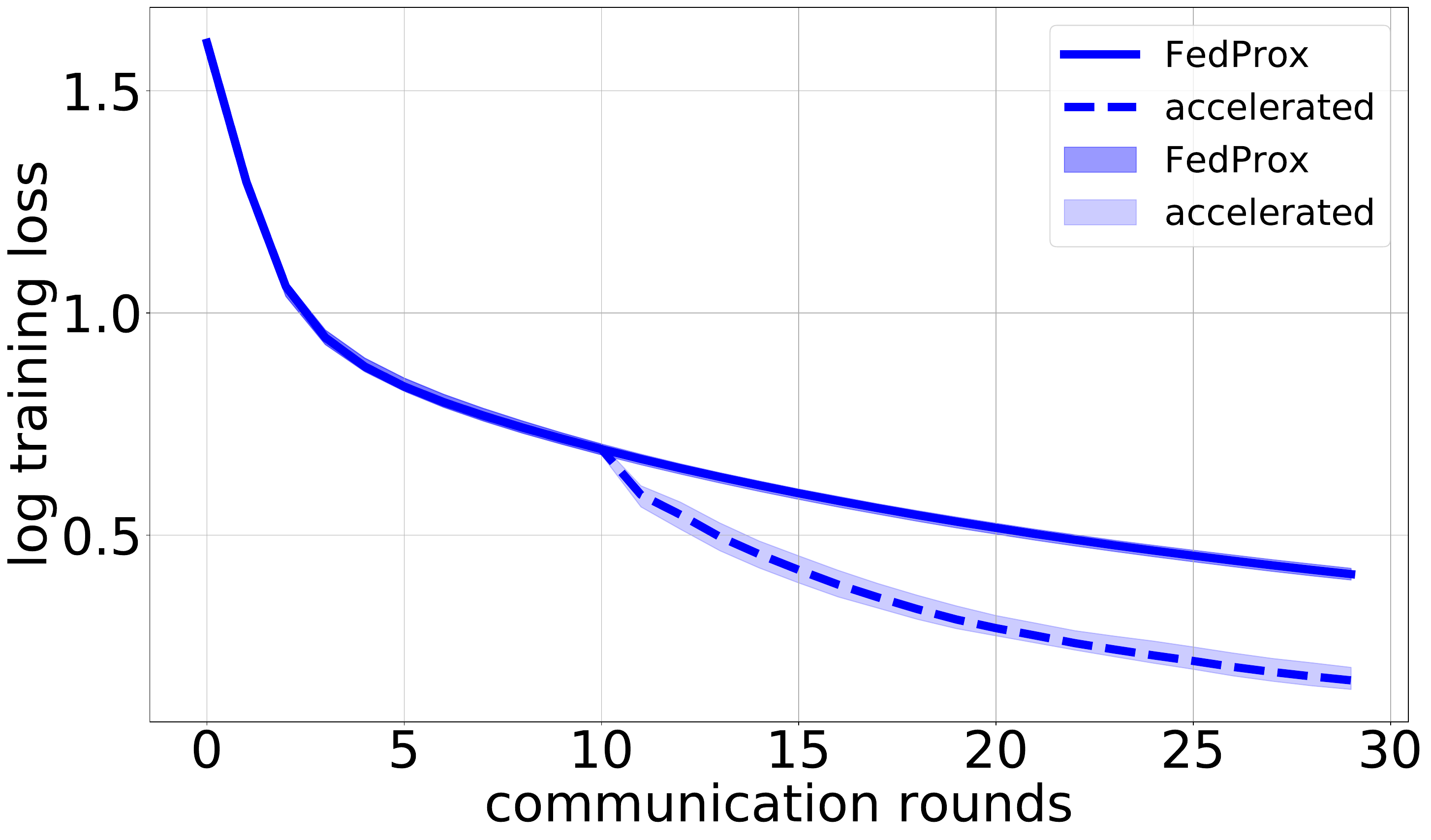}
    \includegraphics[width=0.32\columnwidth]{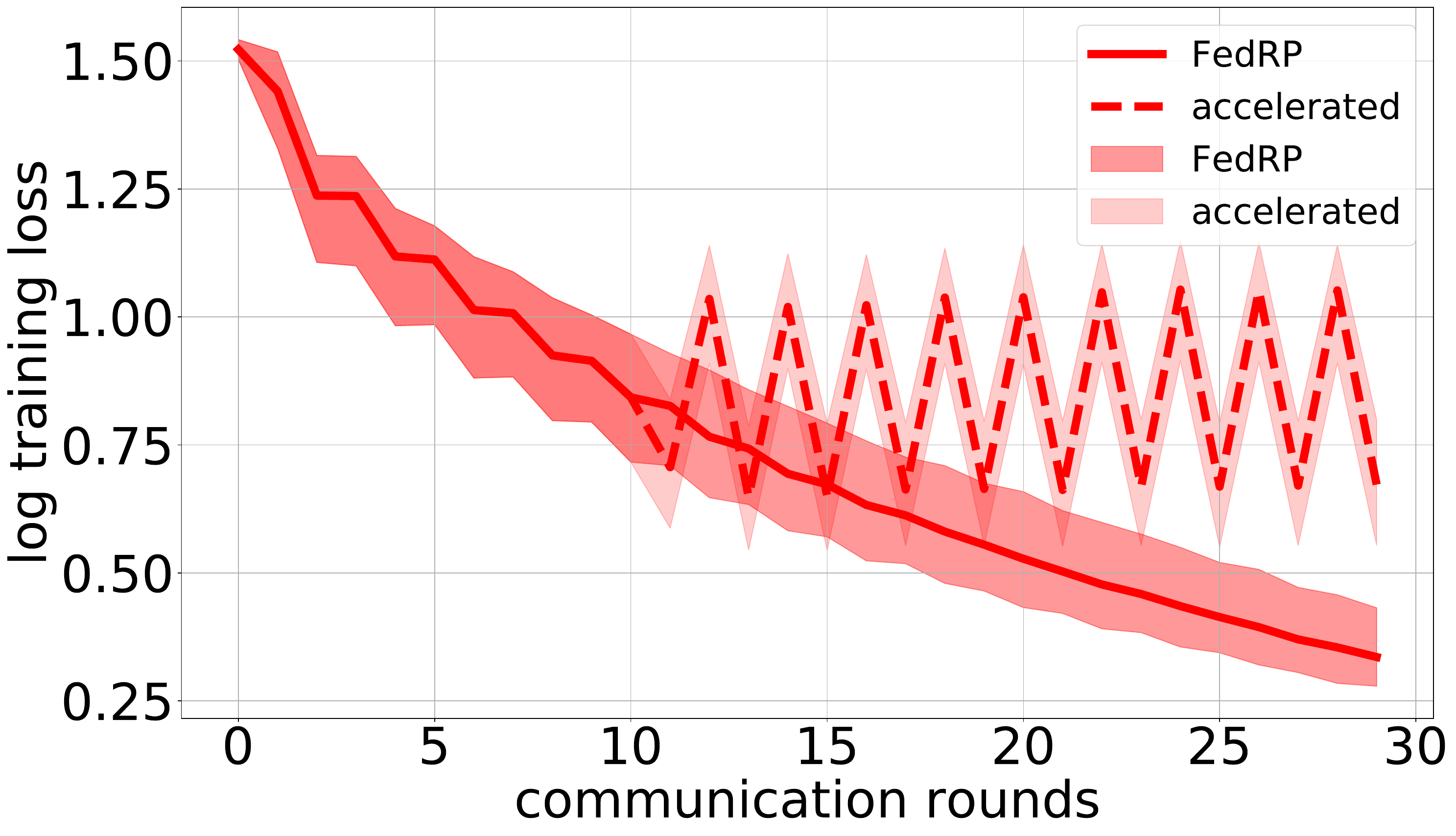}
    \includegraphics[width=0.32\columnwidth]{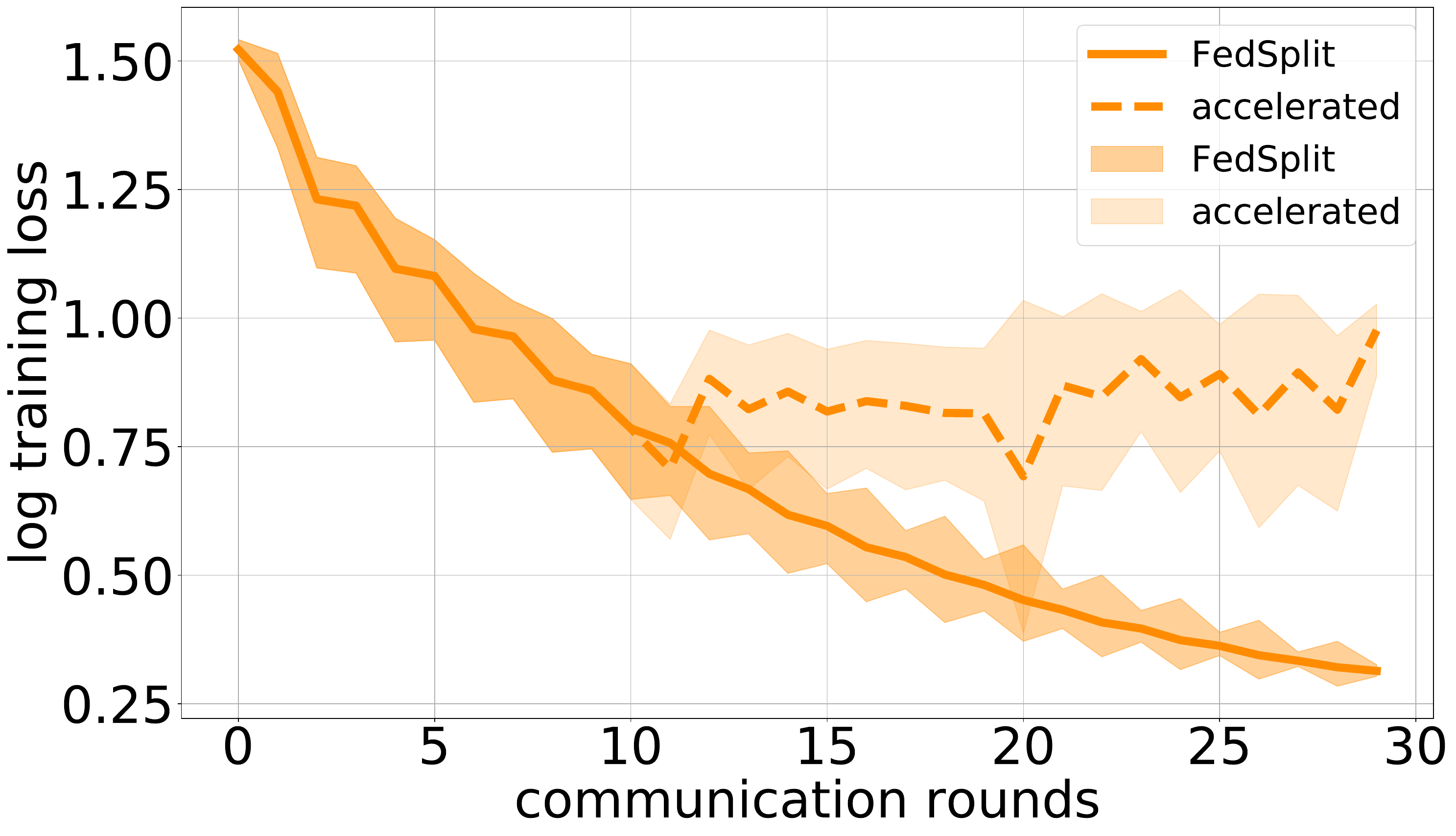}
    \includegraphics[width=0.32\columnwidth]{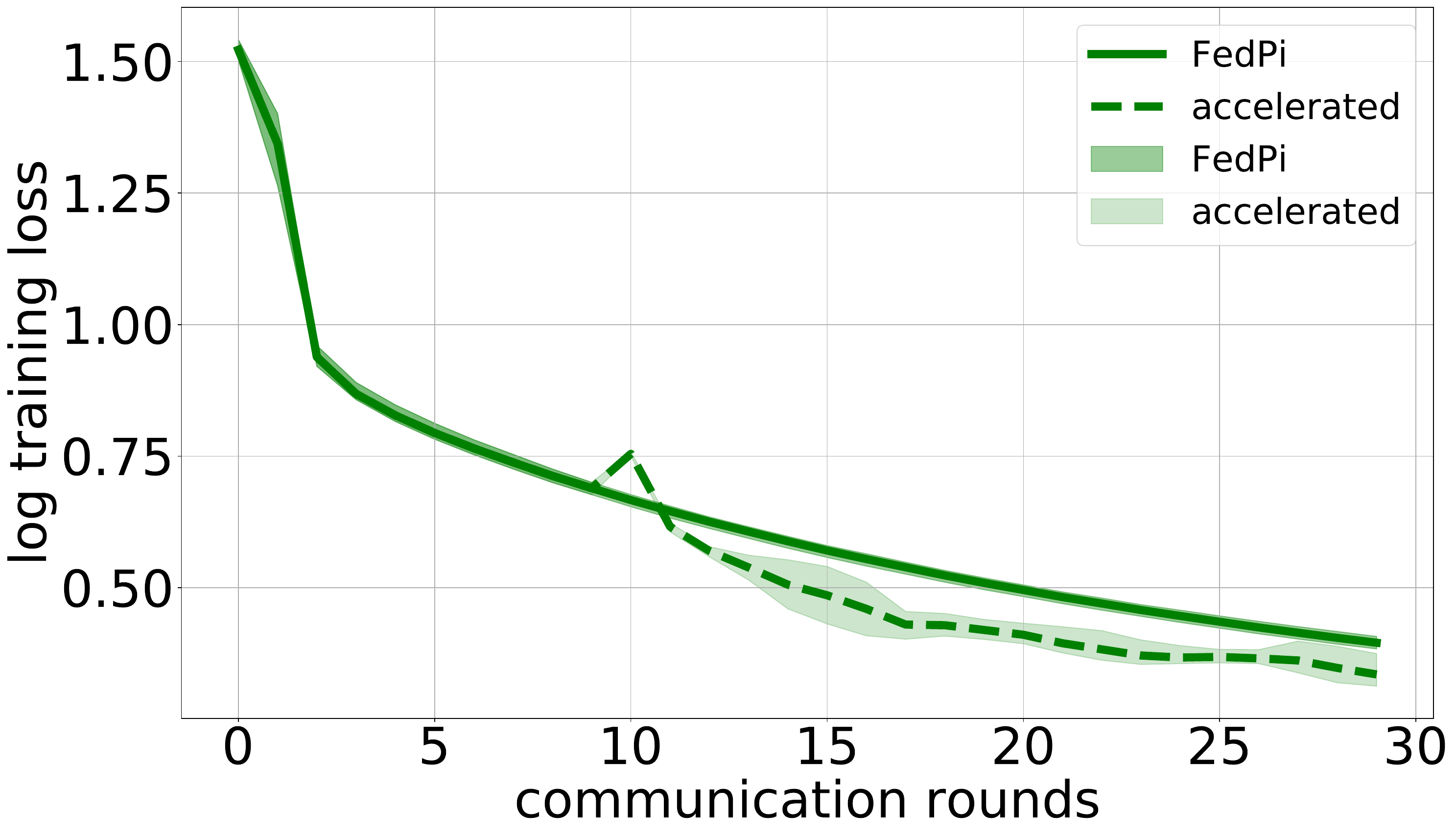}
    \caption{Detailed results for Anderson acceleration for nonconvex optimization problem of CNN training with $\tau=10$. Dashed lines are the accelerated results.}
    \label{fig:extrap_error}
\end{figure}

\section{Proofs}
\label{sec:app-proof}

We first recall the following convenient result: 
\begin{lemma}[\cite{BrezisBrowder76}, {\cite[Theorem 5.36]{Bauschke2011}}]
\label{thm:GFP}
Let $\{\wv_t\}$ and $\{\zv_t\}$ be two sequences in $\RR^d$, $\emptyset\ne \Fm\subseteq\RR^d$, and $W_k := \cl\conv{\wv_t : t\geq k}$. Suppose that 
\begin{enumerate}
\item for all $\wv \in \Fm$, $\|\wv_t - \wv\|_2^2 \to p(\wv) < \infty$;
\item for all $k$, $\dist(\zv_t, W_k) \to 0$ as $t \to \infty$.
\end{enumerate}
Then, the sequence $\{\zv_t\}$ has at most one limit point in $\Fm$. Therefore, if additionally 
\begin{enumerate}[resume]
\item all limit points of $\{\zv_t\}$ lie in $\Fm$,
\end{enumerate}
then the whole sequence $\{\zv_t\}$ converges to a point in $\Fm$.
\end{lemma}
Below, we use well-known properties about firm nonexpansions and Fej\'er monotone sequences, see the excellent book \cite{Bauschke2011} for background.

\fphom*
\begin{proof}
The proof is a simplification and application of \cite{Tseng92}. 

Consider all possible combinations of user participation, \ie let $\Ic := 2^{[m]} \setminus \emptyset$ be the power set of users $[m]:=\{1, \ldots, m\}$, with the empty set removed. For each $I \in \Ic$, define 
\begin{align}
\Tm_I = \Tm_{I, \eta} := \frac{1}{|I|}\sum_{i \in I} \prox[\eta]{f_i}.
\end{align}
For any $\wv\in\Fm$, we have $\wv = \prox[\eta]{f_i}\wv$ for all $i$, and hence $\Fm \subseteq \cap_{I\in\Ic} \mathrm{Fix}\Tm_I$ (equality in fact when we consider singleton $I$). 
Moreover, since we assume each $f_i$ to be convex, the mapping $\Tm_{I, \eta}$ is firmly nonexpansive. In fact, $\Tm_{I, \eta} = \prox[\eta]{f_I}$ for a so-called proximal average function $f_I$ \cite{Bauschke2011}, and hence it follows from \cite{Tseng00} that
\begin{align}
(1 \wedge \eta) \|\wv - \Tm_{I, 1}\wv \|_2 \leq \|\wv - \Tm_{I, \eta} \wv \|_2.
\end{align}
Since we assume $\liminf_t \eta_t := \eta > 0$, the aforementioned monotonicity allows us to assume w.l.o.g. that $\eta_t \equiv \eta > 0$. 

With this established notation we rewrite \FP simply as 
\begin{align}
\wv_{t+1} = \Tm_{I_t} \wv_t,
\end{align}
where $I_t \in \Ic$ can be arbitrary except for all $i$, $|\{t: i \in I_t\}| = \infty$. 

Applying the firm nonexpansiveness of $\Tm_{I_t}$, we have for any $\wv \in \Fm$:
\begin{align}
\| \wv_{t+1} - \wv \|_2^2 + \|\wv_t - \wv_{t+1}\|_2^2 \leq \|\wv_t - \wv\|_2^2,
\end{align}
whence 
\begin{align}
\| \wv_{t+1} - \wv \|_2^2 \leq \|\wv_t - \wv\|_2^2, 
\end{align}
\ie $\{\wv_t\}$ is Fej\'er monotone \wrt $\Fm$. Moreover, 
\begin{align}
\label{eq:asymreg}
\wv_t - \Tm_{I_t} \wv_t \to \zero.
\end{align}

Since $\{\wv_t\}$ is bounded and there are only finitely many choices for $I_t \in \Ic$, from \eqref{eq:asymreg} there exists
$\zv \in \cap_{J\in \Jc} \mathrm{Fix}\Tm_J$ that is a limit point of $\{\wv_t\}$ for some $\emptyset\ne \Jc \subseteq\Ic$. Take a subsequence $\wv_{t_k} \to \zv$ and let $s_k = \min\{ t \geq t_k: I_{t} \not\in \Jc \}$. Pass to a subsequence we may assume $I_{s_k} \equiv I$ and $\wv_{s_k} \to \wv$. Since $\wv_{s_k} - \Tm_{I} \wv_{s_k} \to \zero$ (see \eqref{eq:asymreg}) we have $\wv \in \mathrm{Fix}\Tm_I$. Since $\zv \in \cap_{J\in \Jc} \mathrm{Fix}\Tm_J$ and $I_t \in \Jc$ for $t\in[t_k, s_k)$ we have
\begin{align}
\|\wv_{s_k} - \zv \|_2 \leq \|\wv_{t_k} - \zv \|_2\to 0,
\end{align}
and hence $\zv = \wv \in \mathrm{Fix}\Tm_I = \cap_{i\in I} \mathrm{Fix} \prox[]{f_i}$. Since each user $f_i$ participates infinitely often, we may continue the argument to conclude that any limit point $\zv \in \Fm$. Applying \Cref{thm:GFP} we know the whole sequence $\wv_t \to \wv_\infty\in \Fm$.
\end{proof}
We remark that the convexity condition on $f_i$ may be further relaxed using arguments in \cite{AleynerReich09}.

\fp*
\begin{proof}
We simply verify \Cref{thm:GFP}.
Let $\ws \in \dom\fs \cap H$, $\av^* \in \partial\fs(\ws)$ and $\bv^*\in H^\perp$. 
Applying the firm nonexpansiveness of $\prox[\eta_t]{\fs}$:
\begin{align}
\|\prox[\eta_t]{\fs} \ws_t -  \ws \|_2^2 &= \|\prox[\eta_t]{\fs} \ws_t -  \prox[\eta_t]{\fs}(\ws+\eta_t \av^*) \|_2^2
\\
&\leq \|\ws_t - \ws - \eta_t \av^*\|_2^2 - \|\ws_t - \prox[\eta_t]{\fs} \ws_t - \eta_t \av^*\|_2^2
\\
\label{eq:bb-tmp}
&= \|\ws_t - \ws\|_2^2 - \|\ws_t - \prox[\eta_t]{\fs} \ws_t\|_2^2 + 2\eta_t \dual{\ws \!-\! \prox[\eta_t]{\fs} \ws_t}{\av^*},
\\
\|\prox[]{H}\prox[\eta_t]{\fs} \ws_t -  \ws \|_2^2
&\leq 
\|\prox[\eta_t]{\fs} \ws_t - \ws\|_2^2 - \|\prox[\eta_t]{\fs} \ws_t - \prox[]{H}\prox[\eta_t]{\fs} \ws_t\|_2^2 + 2\eta_t \dual{\ws-\prox[]{H}\prox[\eta_t]{\fs} \ws_t}{\bv^*}.
\end{align}
Summing the above two inequalities and applying the inequality $-\|\xv\|_2^2 + 2 \dual{\xv}{\yv} \leq \|\yv\|_2^2$ repeatedly:
\begin{align}
\label{eq:BB-qfm}
\|\prox[]{H}\prox[\eta_t]{\fs} \ws_t - \ws \|_2^2 \leq 
\|\ws_t - \ws\|_2^2 + 2 \eta_t \dual{\ws-\ws_t}{\av^*+\bv^*} + \eta_t^2 [\|\av^*+\bv^*\|_2^2 + \|\av^*\|^2_2].
\end{align}
Summing over $t$ and rearranging we obtain for any $\ws\in\dom \fs \cap H, \ws^* = \av^*+\bv^*$:
\begin{align}
2 \dual{\ws- \bar\ws_t}{\ws^*} +[\|\av^*\|_2^2 + \|\ws^*\|_2^2] \sum_{k=0}^t \eta_k^2 / \Lambda_t \geq ( \|\ws_{t+1} - \ws\|_2^2 - \|\ws_1-\ws\|_2^2 ) / \Lambda_t,
\end{align}
where $\Lambda_t := \sum_{k=1}^t \eta_k$. 
Using the assumptions on $\eta_t$ we thus know
\begin{align}
\liminf_{t\to\infty} ~~\dual{\ws-\bar\ws_t}{\ws^*} \geq 0.
\end{align}
Since $\ws$ is arbitrary and $\ws^*$ is chosen to be its subdifferential, 
it follows that any limit point of $\{\bar\ws_t\}$ is a solution of the original problem \eqref{eq:wFL}, \ie condition 3 of \Cref{thm:GFP} holds. If $\{\bar\ws_t\}$ is bounded, then $\Fm\ne\emptyset$, which we assume now. Let $\ws\in \Fm$ and set $\ws^* = \zero$ we know from \eqref{eq:BB-qfm} that $\{\ws_t\}$ is quasi-Fej\'er monotone \wrt $\Fm$ (\ie condition 1 in \Cref{thm:GFP} holds). Lastly, let $\bar\eta_{t,k} := \eta_k / \Lambda_t$ and we verify condition (II) in \Cref{thm:GFP}:
\begin{align}
\dist(\bar\ws_t, W_k) &\leq \left\| \bar\ws_t - \sum_{s = k}^t \bar\eta_{t,s} \ws_s /\sum_{\kappa= k}^t \bar\eta_{t,\kappa} \right\|_2 
\\
&\leq \sum_{\kappa= 0}^{k-1} \bar\eta_{t, \kappa} \left[\|\ws_\kappa\|_2 + \left\|\sum_{s=k}^t \bar\eta_{t,s} \ws_s\right\|_2 /\sum_{\kappa= k}^t \bar\eta_{t,\kappa} \right] 
\\
&\stackrel{t\to\infty}{\longrightarrow} 0,
\end{align}
since $\ws_t$ is bounded and for any $k$, $\bar\eta_{t,k} \to 0$ as $t\to\infty$. All three conditions in \Cref{thm:GFP} are now verified.
\end{proof}

We remark that the step size condition is tight, as shown by the following simple example:
\begin{example}
Let $f_{\pm}(w) = \tfrac{1}{2}(w\pm1)^2$. Simple calculation verifies that 
\begin{align}
\prox[\eta]{f_{\pm}}(w) = \tfrac{w\mp \eta}{1+\eta}.
\end{align}
Therefore, the iterates of \FP  for the two functions $f_+$ and $f_-$ are:
\begin{align}
w_{t+1} = \tfrac{w_t}{1+\eta_t} = \prod_{\tau=0}^t \tfrac{1}{1+\eta_\tau} w_0,
\end{align}
which tends to the true minimizer $w_\star = 0$ for any $w_0$ iff 
\begin{align}
\prod_{\tau=0}^t \tfrac{1}{1+\eta_\tau} \to 0 \iff \sum_t \eta_t \to \infty.
\end{align}
If we consider instead $f_+$ and $2f_-$, then
\begin{align}
2w_{t+1} = \tfrac{w_t-\eta_t}{1+\eta_t} + \tfrac{w_t + 2\eta_t}{1+2\eta_t}.
\end{align}
Passing to a subsequence if necessary, let $\eta_t \to \eta \ne 0$ and suppose $w_t \to w_\star = \tfrac13$, then passing to the limit we obtain
\begin{align}
2 = \tfrac{1-3\eta}{1+\eta}  + \tfrac{1+6\eta}{1+2\eta} 
&\iff 
2(1+\eta)(1+2\eta) = (1-3\eta)(1+2\eta)  + (1+6\eta)(1+\eta) 
\\
&\iff \eta = 0,
\end{align}
contradiction. Therefore, it is necessary to have $\eta_t \to 0$ in order for \FP to converge to the true solution $w_\star = \tfrac13$ on this example.
\end{example}

\rcon*
\begin{proof}
Since $\ra[\eta]{f} = \ra{\eta f}$ and $f$ is strongly convex iff $\eta f$ is so, 
w.l.o.g. we may take $\eta = 1$. 
Suppose $\ra{f}$ is $\gamma$-contractive for some $\gamma \in (0,1)$, \ie for all $\wv$ and $\zv$:
\begin{align}
\|\ra{f}\wv - \ra{f}\zv\|_2 \leq \gamma \cdot \|\wv - \zv\|_2. 
\end{align}
It then follows that the proximal map 
$\prox[]{f} = \tfrac{\id + \ra{f}}{2}$ is $\tfrac{1+\gamma}{2}$-contractive. Moreover, $\tfrac{2}{1+\gamma} \prox[]{f}$, being nonexpansive, is the gradient of the convex function $\tfrac{2}{1+\gamma} \env[]{f^*}$. Thus, $\tfrac{2}{1+\gamma}\prox[]{f}$ is actually firmly nonexpansive \cite[Corollary 18.17]{Bauschke2011}. But, 
\begin{align}
\tfrac{2}{1+\gamma}\prox[]{f} = \tfrac{2}{1+\gamma} (\id + \partial f)^{-1} = [\id + (\partial f - \tfrac{1-\gamma}{1+\gamma}\id)\circ \tfrac{1+\gamma}{2}\id]^{-1},
\end{align}
and hence $(\partial f - \tfrac{1-\gamma}{1+\gamma}\id)\circ \tfrac{1+\gamma}{2}\id$ is maximal monotone \cite[Proposition 23.8]{Bauschke2011}, \ie $f$ is $\tfrac{1-\gamma}{1+\gamma}$-strongly convex.
\end{proof}

\fr*
\begin{proof}
To see the first claim, let $\ws$ be a fixed point of \FR, \ie 
\begin{align}
\ws = \prox[]{H} \ra[\eta]{\fs}\ws = \prox[]{H}(2\prox[\eta]{\fs}\ws -\ws) = 2 \prox[]{H}\prox[\eta]{\fs}\ws - \prox[]{H}\ws = 2 \prox[]{H}\prox[\eta]{\fs}\ws - \ws, 
\end{align}
since $\prox[]{H}$ is linear and $\ws \in H$ due to the projection. Thus, $\ws = \prox[]{H}\prox[\eta]{\fs}\ws$. In other words, the fixed points of \FR are exactly those of \FP, and the first claim then follows from \cite{BauschkeCR05}, see the discussions in \Cref{sec:fp} and also \cite[Lemma 7.1]{BauschkeKruk04}.

To prove the second claim, we first observe that $\{\ws_t\}$ is Fej\'er monotone \wrt $\Fm$, the solution set of the regularized problem \eqref{eq:regFL}. Indeed, for any $\ws \in \Fm$, using the firm nonexpansiveness of $\prox[]{H}$ we have

\begin{align}
\|\ws_{t+1} - \ws\|_2^2 + \|\ra[\eta]{\fs}\ws_t - \ws_{t+1} - \ra[\eta]{\fs}\ws + \ws\|_2^2 \leq \|\ra[\eta]{\fs}\ws_t - \ra[\eta]{\fs} \ws\|_2^2 \leq \|\ws_t - \ws\|_2^2.
\end{align}
Summing and telescoping the above we know
\begin{align}
\ra[\eta]{\fs}\ws_t - \ws_{t+1} \to \ra[\eta]{\fs}\ws - \ws.
\end{align}
Let $\ws_\infty$ be a limit point of $\{\ws_t\}$ (which exists since $\ws_t$ is bounded). 
Since $\ra[\eta]{\fs}$ is idempotent, we have the range $\ran\ra[\eta]{\fs}$ equal to its fixed point $\mathrm{Fix}\ra[\eta]{\fs}$. 
When the users are homogeneous, we can choose $\ws = \ra[\eta]{\fs}\ws$.
Therefore, $\ws_\infty \in \ran \ra{\fs} = \mathrm{Fix}\ra{\fs}$. 
Since $\wv_\infty \in H$ by definition, $\wv_\infty \in \Fm$. Applying \Cref{thm:GFP} we conclude that the entire Fej\'er sequence $\{\ws_t\}$ converges to $\wv_\infty\in\Fm$.
\end{proof}
As shown in \cite{BauschkeKruk04}, a (closed) convex cone is obtuse iff its reflector is idempotent. To give another example: the univariate, idempotent reflector $\ra{f}(w) = (w)_+ :=  \max\{w, 0\}$ leads to
\begin{align}
\prox{f}(w) = \tfrac{w + \ra{f}w}{2} = \tfrac{w + (w)_+}{2} \implies f(w) = \tfrac12 (-w)_+^2.
\end{align}
Interestingly, the epigraph of the above $f$ is a convex set but not a cone and yet its reflector is idempotent. Therefore, the idempotent condition introduced here is a strict generalization of \cite{BauschkeKruk04}.

\end{document}